\documentclass[journal, romanappendices]{IEEEtran} 
\usepackage{etex}
\usepackage{amsthm,array,  multirow, bm, amsmath, mathtools, amssymb,booktabs, subcaption,comment}  
\usepackage[bookmarks=false, colorlinks=true, allcolors=blue]{hyperref}  
\usepackage{algorithm}
\usepackage{cite}
\usepackage{caption}
\usepackage{algorithm, algpseudocode}
\usepackage{chngcntr}

\newcommand{\norm}[1]{\left\|#1\right\|}

\newtheorem{theorem}{Theorem}
\newtheorem{lem}[theorem]{Lemma}
\newtheorem{claim}[theorem]{Claim}

\newtheorem{corollary}[theorem]{Corollary}
\newtheorem{definition}[theorem]{Definition}
\newtheorem{remark}[theorem]{Remark}

\newcommand{\bi}{\begin{itemize}}
\newcommand{\ei}{\end{itemize}}
\newcommand{\ben}{\begin{enumerate}}
\newcommand{\een}{\end{enumerate}}

\newcommand{\bean}{\begin{eqnarray*} }
\newcommand{\eean}{\end{eqnarray*} }
\newcommand{\bea}{\begin{eqnarray} }
\newcommand{\eea}{\end{eqnarray} }
\newcommand{\ba}{\begin{align*} }
\newcommand{\ea}{\end{align*} }

\newcommand{\nn}{\nonumber}

\newcommand{\nois}{g}
\newcommand{\epsbnd}{d(\alpha)}
\newcommand{\epsden}{d_{\mathrm{denom}}(\alpha)}

\newcommand{\outfracrow}{\small{\text{max-outlier-frac-row}}}
\newcommand{\outfraccol}{\small{\text{max-outlier-frac-col}}}

\newcommand{\dif}{{\text{dif}}}
\newcommand{\xhat}{\bm{\hat{x}}}

\newcommand{\bl}{\begin{frame}}
\newcommand{\el} {\end{frame}}

\newcommand{\cred}{\color{red}} 
\newcommand{\cb}{\color{blue}}
\newcommand{\cbl}{\color{black}}

\newcommand{\evdeq}{\overset{\mathrm{EVD}}=}

\renewcommand\thetheorem{\arabic{section}.\arabic{theorem}}

\newcommand{\tmax}{d} 

\newcommand{\wt}{\bm{w}_t}

\newcommand{\xt}{\bm{x}_t}
\newcommand{\x}{\bm{x}}
\newcommand{\xhatt}{\hat{\bm{x}}_t}
\renewcommand{\l}{\bm{\ell}}

\newcommand{\lt}{\bm{\ell}_t}
\newcommand{\lhat}{\hat{\bm{\ell}}}

\newcommand{\lhatt}{\hat{\l}_t}   
\newcommand{\yt}{\bm{y}_t}
\newcommand{\y}{\bm{y}}
\newcommand{\tty}{\tilde{\bm{y}}}
\newcommand{\w}{\bm{w}}
\renewcommand{\v}{{\bm{\nu}}}
\newcommand{\vt}{\v_t}

\renewcommand{\a}{\bm{a}}
\newcommand{\e}{\bm{e}} 

\newcommand{\et}{\bm{e}_t}
\newcommand{\Pt}{\bm{P}_t}

\newcommand{\at}{\bm{a}_t}

\newcommand{\I}{\bm{I}}

\newcommand{\Lam}{\bm{\Lambda}}

\newcommand{\T}{\mathcal{T}}
\newcommand{\J}{\mathcal{J}}

\newcommand{\D}{\bm{D}}
\newcommand{\A}{\bm{A}}

\newcommand{\old}{\mathrm{old}}

\newcommand{\Lhat}{\hat{\bm{L}}}

\renewcommand{\P}{\bm{P}}
\newcommand{\U}{\bm{U}}
\newcommand{\V}{\bm{V}}
\newcommand{\R}{\bm{R}}

\renewcommand{\b}{\bm{b}}
\newcommand{\Phat}{\hat{\bm{P}}}

\newcommand{\Span}{\operatorname{span}} 

\newcommand{\rank}{\operatorname{rank}}
\newcommand{\E}{\mathbb{E}}

\newcommand{\train}{\mathrm{train}}
\newcommand{\That}{\hat{\mathcal{T}}}

\newcommand{\SE}{\sin\theta_{\max}}

\newcommand{\that}{{\hat{t}}}

\newcommand{\M}{\bm{M}}

\newcommand{\X}{\bm{X}}
\newcommand{\Y}{\bm{Y}}
\newcommand{\W}{\bm{W}}
\newcommand{\Z}{\bm{Z}}

\newcommand\diagmat[2]{\begin{bmatrix} #1 & \bm{0} \\ \bm{0} & #2\end{bmatrix}} 

\newcommand{\ed}{\mathrm{end}}

\newcommand{\SVD}{\mathrm{SVD}}
\newcommand{\B}{\bm{B}}
\newcommand{\alphadel}{\alpha}
\renewcommand{\Re}{\mathbb{R}}

\newcolumntype{C}[1]{>{\centering\let\newline\\\arraybackslash\hspace{0pt}}m{#1}}

\newcommand{\init}{\mathrm{init}}

\newcommand{\offline}{\mathrm{offline}}
\newcommand{\bv}{b_v}
\newcommand{\bz}{b} 
\newcommand{\rmat}{r_{\mat}}
\newcommand{\semax}{\Delta_{\dist}}
\newcommand{\Aa}{\bm{Q}_1}
\newcommand{\Ba}{\bm{Q}_2}
\newcommand{\Ca}{\bm{Q}_3}

\newcommand{\subup}{\text{SubUpd}}

\newcommand{\st}{\bm{x}_t}

\newcommand{\bt}{\bm{b}_t}
\newcommand{\zt}{\bm{Z}_t}

\newcommand{\zz} {\varepsilon} 

\newcommand{\Tt}{\mathcal{T}_t}

\newcommand{\Thatt}{\hat{\mathcal{T}}_t}

\newcommand{\shatcs}{\xhat_{t,cs}}

\newcommand{\smin}{x_{\min}}

\newcommand{\pt}{\bm{P}}
\newcommand{\shatt}{\hat{\bm{x}}_t}

\newcommand{\itt}{\bm{I}_{\Tt}}
\newcommand{\bphi}{\bm{\Phi}}
\newcommand{\bpsi}{\bm{\Psi}}

\newcommand{\lthres}{\omega_{evals}} 

\newcommand{\tildej}{j}

\usepackage{tikz}
\usepackage{pgfplots}
\usepgfplotslibrary{groupplots}

\pgfplotsset{
        my stylecompare/.style={
			width=.4\textwidth,
            height=4.5cm,
            label style={font=\Large},
            title style={font=\Large},
            x tick label style={font =\small, /pgf/number format/1000 sep=},
	    axis lines=left,
	    major x tick style = transparent,
	    major y tick style = transparent,
	    every y tick label/.style={
 		   xshift=-.7cm, yshift=-2pt,anchor=south west,inner sep=0pt,font=\small, 
 		   scaled y ticks=false,
	    },
	    every x tick label/.style={font = \small},
            scaled x ticks=false,
        },
        my legend style compare/.style={
            legend entries={
            		GRASTA,
            		ORPCA,
            		s-ReProCS,	
            		NORST,	
            		Offline-NORST,
            		Alt Proj,
            		RPCA-GD,
            },
            legend style={
                at={(-.31,1.5)},
                anchor=north west,
            },
            legend columns=4,
	    legend style={font=\small},
        },
        cycle multi list={
        {blue, line width=0.6pt, mark=o,mark size=3pt}, 
        {black, line width=0.6pt, mark=square,mark size=2.5pt}, 
        {teal, line width=0.6pt, mark=oplus,mark size=2.5pt},
        {red, line width=0.6pt, mark=triangle,mark size=2.5pt},
        {red, solid, line width=0.5pt, mark=oplus,mark size=2.8pt}, 
        {olive, line width=0.6pt, mark=10-pointed star,mark size=2.5pt},  
        {cyan, line width=0.6pt, mark=Mercedes star,mark size=3pt}, 
        {teal, line width=0.6pt, mark=oplus,mark size=2.5pt}, 
		},
}

\pgfplotstableread[col sep = comma]{figures/final_files/err_SE_medrop_icml_mo.dat}\ltmodata
\pgfplotstableread[col sep = comma]{figures/final_files/err_SE_medrop_icml_bern.dat}\ltmodatabern

\usetikzlibrary{decorations.markings}
\usetikzlibrary{decorations.pathreplacing}
\def\MarkLt{4pt}
\def\MarkSep{2pt}

\tikzset{
  TwoMarks/.style={
    postaction={decorate,
      decoration={
        markings,
        mark=at position #1 with
          {
              \begin{scope}[xslant=0.2]
              \draw[line width=\MarkSep,white,-] (0pt,-\MarkLt) -- (0pt,\MarkLt) ;
              \draw[-] (-0.5*\MarkSep,-\MarkLt) -- (-0.5*\MarkSep,\MarkLt) ;
              \draw[-] (0.5*\MarkSep,-\MarkLt) -- (0.5*\MarkSep,\MarkLt) ;
              \end{scope}
          }
       }
    }
  },
  TwoMarks/.default={0.5},
}

\usetikzlibrary{positioning}

\usetikzlibrary{calc}

\tikzstyle{block}  = [rectangle, draw, rounded corners, text width=5cm, text centered, minimum height=1em]
\tikzstyle{smallblock}  = [rectangle, draw, rounded corners,text width=1.8cm, text centered, minimum height=1em]
\tikzstyle{input}  = [rectangle, draw, text width=1.2cm, text centered, minimum height=1em]
\tikzstyle{output}  = [rectangle, draw, text width=1.2cm, text centered, minimum height=1em]
\tikzstyle{block1}  = [rectangle, draw, rounded corners,text width=5cm, text centered, minimum height=1em]
\tikzstyle{blockl1}  = [rectangle, draw, rounded corners,text width=4cm, text centered, minimum height=1em]

\renewcommand{\zt}{\bm{z}_t}
\renewcommand{\semax}{\Delta}

\newcommand{\xmint}{x_{\min}}

\renewcommand{\dif}{\Delta}
\renewcommand{\rmat}{r_{\scriptscriptstyle{L}}}
\renewcommand{\S}{\bm{X}}
\renewcommand{\SE}{\mathrm{SE}}

\newcommand{\Sig}{\bm{\Sigma}}

\newcommand{\tmaxpca}{\alpha}
\renewcommand{\zz}{\varepsilon}
\renewcommand{\epsbnd}{H(\tmaxpca)}

\renewcommand{\epsden}{H_{\text{denom}}(\tmaxpca)}
\newcommand{\epsdenb}{G_{\text{denom}}(\tmaxpca)}

\newcommand{\Lamb}{\bar{\Lam}}
\newcommand{\avg}{\mathrm{avg}}

\newcommand{\cross}{\mathrm{cross}}
\newcommand{\noise}{\mathrm{noise}}

\renewcommand{\L}{\bm{L}}

\newcommand{\vp}{\bm{v}}
\renewcommand\thetheorem{\arabic{section}.\arabic{theorem}}

\renewcommand{\subsubsection}[1]{{\em \noindent #1.}}
\newtheorem{assu}[theorem]{Assumption}
\renewcommand{\eta}{\mu}
\renewcommand{\Pt}{\P_{(t)}}
\renewcommand{\SVD}{\mathrm{SVD}}

\usepackage{setspace}
\renewcommand{\cred}{\color{black}}

\begin{document}

\title{PCA in Data-Dependent Noise and Fast Robust Subspace Tracking}
\title{Fast Robust Subspace Tracking via
\\ PCA in Sparse Data-dependent Noise\thanks{Parts of this manuscript have appeared in ICML 2018 \cite{rrpcp_icml}, ISIT 2018 \cite{pca_dd_isit} and Allerton 2017 \cite{pca_dd}. 
}
}
\author{Praneeth Narayanamurthy and Namrata Vaswani\\
ECE dept, Iowa State University}

\maketitle

\begin{abstract}
This work studies the robust subspace tracking (ST) problem. Robust ST can be simply understood as a (slow) time-varying subspace extension of robust PCA. It assumes that the true data lies in a low-dimensional subspace that is either fixed or changes slowly with time. The goal is to track the changing subspaces over time in the presence of additive sparse outliers and to do this quickly (with a short delay).
We introduce a ``fast'' mini-batch robust ST solution that is provably correct under mild assumptions. Here ``fast'' means two things: (i) the subspace changes can be detected and the subspaces can be tracked with near-optimal delay, and (ii) the time complexity of doing this is the same as that of simple (non-robust) PCA. Our main result assumes piecewise constant subspaces (needed for identifiability), but we also provide a corollary for the case when there is a little change at each time.

A second contribution is a novel non-asymptotic guarantee for PCA in linearly data-dependent noise. An important setting where this is useful is for linearly data dependent noise that is sparse with support that changes enough over time. The analysis of the subspace update step of our proposed robust ST solution uses this result.
\end{abstract}

\section{Introduction}
Principal Components Analysis (PCA) is one of the most widely used and well studied dimension reduction techniques. It is solved via singular value decomposition (SVD) following by retaining the top $r$ singular vectors for getting an $r$-dimensional subspace approximation.
Robust PCA (RPCA) refers to PCA in the presence of outliers.
According to \cite{rpca}, it can be defined as the problem of decomposing a given data matrix into the sum of a low-rank matrix (true data) and a sparse matrix (outliers). The column space of the low-rank matrix then gives the desired principal subspace (PCA solution).
A common application of RPCA is in video analytics in separating a video into a slow-changing background image sequence (modeled as a low-rank matrix) and a foreground image sequence consisting of moving objects or people (modeled as a sparse matrix) \cite{rpca}. 
The RPCA problem has been extensively studied in the last decade since \cite{rpca,rpca2} introduced the principal components pursuit solution and obtained the first guarantees for it. Follow-up work by Hsu et al \cite{rpca_zhang}  studied it further.
Later work \cite{robpca_nonconvex,rpca_gd,rmc_gd} has developed provable non-convex solutions that are much faster. Alternating Projections or AltProj was the first such approach \cite{robpca_nonconvex}.

{\em Robust Subspace Tracking (ST)} can be simply understood as a (slow) time-varying subspace extension of RPCA. It assumes that the true data lies in a low-dimensional subspace that is either fixed or changes slowly with time. We focus on slow changing subspaces because it is not clear how to distinguish the effect of a sudden subspace change from that of an outlier. The goal is to track the changing subspaces over time in the presence of additive sparse outliers and to do this quickly (with a short delay).
Time-varying subspaces is a more appropriate model for long data sequences, e.g., long surveillance videos, since if a single subspace model is used, the resulting matrix may not be sufficiently low-rank. Moreover the tracking setting (short tracking delay) is needed for applications where near real-time estimates are needed, e.g., video-based surveillance (object tracking), monitoring seismological activity, or detection of anomalous behavior in dynamic social networks. 
%
%
While many heuristics exist for robust ST, e.g., \cite{rrpcp_allerton,rrpcp_perf,grass_undersampled, xu_nips2013_1, robust_admm, rst_lr, rst_sparse}, there has been little work on provably correct solutions \cite{rrpcp_aistats,rrpcp_dynrpca}. The first result \cite{rrpcp_aistats} needed many restrictive assumptions (most importantly it required assumptions on intermediate algorithm estimates) and a large tracking delay (the delay was proportional to $1/\zz^2$ to get a $\zz$ accurate estimate).
The second one \cite{rrpcp_dynrpca} significantly improved upon \cite{rrpcp_aistats}, but still required a very specific model on subspace change, needed an $\zz$-accurate initial subspace estimate in order to guarantee $\zz$-accurate recovery at later time instants, and its tracking delay was $r$-times sub-optimal. Our work builds on \cite{rrpcp_dynrpca} and removes these, and two other more technical, limitations that we explained later.

\subsubsection{Contributions}
This work has two contributions. (1) First, we introduce a ``fast'' mini-batch robust ST solution that is provably correct under mild assumptions. Here ``fast'' means two things: (i) the subspace changes can be detected and the subspaces can be tracked with near-optimal delay (the number of data samples required to track an $r$-dimensional subspace of $\Re^n$ to $\zz$ accuracy is within log factors of $r$); and (ii) the time complexity of doing this is just $O(n d r \log(1/\zz))$, which is, order-wise, the same as that of solving the basic (non-robust) PCA problem for an $n \times d$ matrix. Our main result assumes piecewise constant subspaces (needed for identifiability), but we also provide a corollary for the case when there is a little change at each time.
(2) Our second contribution is a novel non-asymptotic guarantee for PCA in data-dependent noise that satisfies certain simple assumptions. 
An important setting where these hold is for linearly data dependent noise that is sparse with enough support changes over time. This problem occurs in the subspace update step of our proposed robust ST solution. 
The PCA result is also of independent interest. As an example, it is useful for analyzing PCA and subspace tracking with missing data \cite{rrpcp_tsp}.

\subsubsection{Organization} We first summarize our notation and then provide a brief discussion of the significance of our PCA guarantee and how it is used in analyzing our robust ST solution next.
In Sec. \ref{pca_section}, we  present the result for PCA in data-dependent noise and its corollary for the sparse data-dependent noise case. In Sec. \ref{sec:norst}, we define the robust ST problem, state the assumptions required to ensure its identifiability, develop the nearly (delay) optimal robust subspace tracker (NORST) algorithm for solving it, and provide and discuss the correctness guarantee for it. Related work is discussed in detail in Sec. \ref{rel_work}. Two important extensions of our result are provided in Sec. \ref{extends}. We provide a proof of the correctness guarantee for NORST in Sec. \ref{sec:proof_norst}. Empirical evaluation on synthetic and real-world datasets is described in Sec. \ref{sec:sims}. We conclude and discuss future directions in Sec. \ref{conclude}.


\subsection{Notation}
We use the interval notation $[a, b]$ to refer to all integers between $a$ and $b$, inclusive, and we use $[a,b): = [a,b-1]$.  $\|.\|$ denotes the $l_2$ norm for vectors and induced $l_2$ norm for matrices unless specified otherwise, and $'$ denotes transpose. We use $\M_\T$ to denote a sub-matrix of $\M$ formed by its columns indexed by entries in the set $\T$.
In our algorithm statements, we  use $\hat{\L}_{t; \alpha} := [\lhat_{t-\alpha + 1}, \lhat_{t-\alpha+2}, \dots, \lhatt]$ and $\SVD_r[\M]$ to refer to the matrix of top of $r$ left singular vectors of the matrix $\M$.
A matrix $\P$ with mutually orthonormal columns is referred to as a {\em basis matrix}; it represents the subspace spanned by its columns.
For basis matrices $\P_1,\P_2$, $\SE(\P_1,\P_2):=\|(\I - \P_1 \P_1{}')\P_2\|$ quantifies the Subspace Error (distance) between their respective subspaces. This is equal to the sine of the largest principal angle between the subspaces. 
If $\P_1$ and $\P_2$ are of the same dimension, $\SE(\P_1, \P_2) = \SE(\P_2, \P_1)$.
{\em We reuse the letters $C,c$ to denote different numerical constants in each use with the convention that $C \ge 1$ and $c < 1$.}%

\cred
\subsection{Significance and novelty of our PCA result and its use to analyze Robust Subspace Tracking}
There is little existing work that explicitly studies PCA (solved via SVD) in the presence of data-dependent noise (work that exploits knowledge of the data-dependency structure of the noise)\footnote{Of course any work on PCA for an approximately low rank matrix makes no assumptions on true data or noise and thus does implicitly allow data-dependent noise as well. However, this type of work does not exploit knowledge of how the noise depends on the data.}.
Our work provides a guarantee for one such setting; the setting is motivated by PCA in sparse linearly data-dependent noise (PCA-SDDN). This problem occurs when studying the SVD solution for solving (i) PCA with missing data, (ii) ST with missing data, and (iii) robust ST (with outliers and with and without missing data). 
We briefly explain the technical novelty of our result here. Let $\W$ denote the sparse linearly data-dependent noise matrix corrupting a true low rank $r$ data matrix $\L$. We observe $\yt = \lt + \w_t$, $t=1,2, \dots, \alpha$ with $\lt = \P \at$, $\P$ is an $n \times r$ matrix with orthonormal columns and $r \ll n$. Since $\w_t$ is linearly data-dependent and sparse, without loss of generality, we can express it as $\w_t = \I_{\T_t}{} \M_{s,t} \lt$ with $\T_t = \mathrm{support}(\w_t)$ and $\M_{s,t}$ being the data-dependency matrix at time/column $t$. Let $b$ denote the maximum of  the fraction of nonzero entries in any row of $\W$.
We compute the PCA estimate, $\Phat$, as the $r$-SVD of $\Y:=[\y_1, \y_2, \dots, \y_\alpha] = \L + \W$.

(1) The sparsity of the noise along with a careful application of the Cauchy-Schwarz inequality implies that $\|\E[ \frac{1}{\alpha} \sum_t \w_t \w_t{}'] \| \le \sqrt{b} \max_t \|\E[\w_t \w_t{}'] \|$, i.e., the time-averaged noise power is at most $\sqrt{b}$ times its maximum instantaneous value. Thus, if $b$ is small enough (noise support changes sufficiently across columns), the former is much smaller than  the latter.
(2)  Since $\w_t$ depends on $\lt$, this means that the data-noise correlation $\E[\lt\wt{}']$ is not zero and, its time-averaged value, $\|\E[ \frac{1}{\alpha} \sum_t \l_t \w_t{}'] \|$, is in fact the dominant term in the perturbation $\|\Y \Y' - \L \L'\|$ that governs the subspace recovery error, $\SE(\Phat,\P)$. Again using Cauchy-Schwarz and sparsity of $\wt$, we can show that $\|\E[ \frac{1}{\alpha} \sum_t \l_t \w_t{}'] \| \le \sqrt{b} \max_t \|\E[\l_t \w_t{}'] \|$. Thus, even though signal-noise correlation is not zero (and is, in fact, proportional to signal power), its time-averaged value is $\sqrt{b}$ times smaller.
%
Since $\SE(\Phat,\P) \lessapprox  \frac{\|\Y \Y' - \L \L'\| }{ \lambda_{r}(\L \L') }$ when the numerator is small enough (by Davis-Kahan $\sin \theta$ theorem), the above two facts imply that
\begin{align*}
\SE(\Phat, \P) &\lessapprox \sqrt{b}  \frac{(2\max_t \|\M_{s,t} \bm\Sigma\| + \max_t \|\M_{s,t} \bm\Sigma \M_{s,t}'\|)}{\lambda_r(\bm\Sigma)} \\
&\le \sqrt{b}  (2q + q^2) f
\end{align*}
Here $\bm\Sigma:= \E[\lt \lt'] \evdeq \P \Lam \P'$, $q:= \max_t \|\M_{s,t} \P\|$, and $f$ denotes the condition number of $\Lam$ \footnote{If $\lt$ is not stationary, $\bm\Sigma:= \frac{1}{\alpha} \sum_t \E[\lt \lt'] $, in this case one needs to redefine $f = \max_t \|\E[\lt \lt']\| / \lambda_r(\bm\Sigma)$.}. $q$ can be understood as the noise-to-signal ratio and is thus a measure of the noise level. 
(3) Suppose that the  $\at$'s are i.i.d. and bounded, i.e., $\|\at\|^2 \le \mu r \lambda_{\max}(\bm\Sigma)$. Since the noise is data-dependent, and since we assume that our data $\lt$ is generated from a low ($r$) dimensional subspace, we can use the above facts and matrix-Bernstein \cite{tropp} to show that  $\SE(\Phat,\P) \le \epsilon$ with high probability, $1- 3 n^{-10}$, if the sample complexity $\alpha$ is $\Omega( \frac{q^2}{\epsilon^2} \kappa^2 r \log n)$. Thus, in order to achieve a recovery error $\epsilon$ that is fraction of the noise level, $q$, the required sample complexity is near optimal (is within log factors of $r$). 

In the above discussion we have assumed zero uncorrelated noise, but our actual result also handles that. This can model the fact that the true data is only approximately low rank.
Moreover it provides a guarantee for a more general setting than PCA-SDDN.

\subsubsection{Use to analyze Robust ST}
For solving the robust ST problem (recover $\lt$ and its subspace from $\yt := \lt + \xt$ where $\xt$ denotes the sparse outlier at time $t$), we develop a mini-batch algorithm that (a) processes the observed data to return an estimate of $\lt$, $\lhat_t$, at each time $t$; and (b) uses $\alpha$-mini-batches of $\lhat_t$ to compute a new estimate of the current subspace. This process is repeated $K$ times with $K$ new $\alpha$-length mini-batches for the current subspace; after this time, the algorithm enters a ``subspace change detect'' phase. Denote the estimate from the $k$-th iteration by $\Phat_k$.
Suppose that the processing is such that (i) $\lhat_t = \lt + \w_t$ where $\w_t$ is sparse and data-dependent noise whose support equals the set of outlier entries at time $t$; and (ii) $q_k:= \max_t \|\M_{s,t} \P\|$ is proportional to the subspace recovery error from iteration $k-1$, i.e., $q_k = C \ \SE(\Phat_{k-1},\P)$ and $q_k < 2$. In defining $q_k$, the max is taken over the mini-batch used in iteration $k$.
We can use our PCA result to show that $\SE(\Phat_k,\P) \lessapprox \sqrt{b} (2q_k + q_k^2)  \kappa \le \kappa \sqrt{b} 6q_k $ and thus $q_{k+1} = C \ \SE(\Phat_k,\P) \lessapprox 6C \kappa \sqrt{b} q_k $. Thus, if $b$ is small enough, $b < c/kappa^2$, then $q_{k+1}$, and hence, $\SE(\Phat_k,\P)$, decreases by a constant fraction in each new iteration (the decay is geometric).  Since the error in recovering $\lt$ satisfies $\|\lhat_t - \lt\|/\|\lt\| \le q_k $, this also decays geometrically with each iteration. 


\cbl

\renewcommand{\bv}{\vp_t}
\renewcommand{\bz}{b}
\newcommand{\Mtone}{\M_{1,t}}
\newcommand{\Mttwo}{\M_{2,t}}
\renewcommand{\epsdenb}{H_{denom}(\alpha)}

\section{PCA in Data-Dependent Noise}\label{pca_section}

%

\subsection{Problem Setting}

For $t = 1, 2, \cdots, \tmaxpca$ we are given $\yt \in \Re^n$ that satisfies
\begin{align}\label{eq:pca_dd}
\yt := \lt + \wt + \vp_t, \quad \text{where} \ \  \lt = \P \at, \  \ \wt = \M_t \l_t,
\end{align}
$\P$ is an $n \times r$ basis matrix with $r \ll n$; $\lt$ is the true data vectors that lies in an $r$-dimensional subspace of $\R^n$, $\Span(\P)$; $\at$'s are the projections of $\lt$'s onto this subspace; $\wt$ is data-dependent noise with $\M_t$ being the data-dependency matrix at time $t$; and $\vp_t$ is uncorrelated noise. This means that $\E[ \lt \vp_t{}'] = 0$ for all times $t$. Here $\at$ and $\vp_t$ are treated as random variables (r.v.), while everything else is deterministic.
The goal is to estimate $\Span(\P)$ from the observed data stream $\yt$, $t=1,2,\dots,\alpha$.

\subsection{SVD solution and guarantee for it}
\subsubsection{SVD Solution} We compute the subspace estimate $\Phat$ as the matrix of top $r$ left singular vectors of $\Y:=[\y_1, \y_2, \dots, \y_\alpha]$. Equivalently it is the matrix of top $r$  eigenvectors of $\frac{1}{\tmaxpca}\sum_t \yt\yt{}'$.

We make the following assumptions on the subspace coefficients, $\at$, and the uncorrelated noise, $\vp_t$.
\begin{assu}[Statistical Assumption on $\at$]\label{def_right_incoh}
Assume that  the $\at$'s are zero mean; mutually independent; have identical diagonal covariance matrix $\Lam$, i.e., that $\E[\at \at{}'] = \Lam$; and are bounded: $\max_t \|\at\|^2 \le \mu r \lambda_{\max}(\Lam)$. Define
$
\lambda^+: = \lambda_{\max}(\Lam), \ \lambda^-: = \lambda_{\min}(\Lam),  \ f : = \frac{\lambda^+}{\lambda^-}.
$
\end{assu}
As we explain in Sec. \ref{sec:norst}, this assumption is almost equivalent to assuming $\mu$-incoherence of the right singular vectors of the matrix $\L:=[\l_1,\l_2, \dots, \l_\alpha]$. We call it $\mu$ statistical right incoherence there.
\begin{assu}[Statistical Assumption on $\vp_t$]
Assume that $\vp_t$  is uncorrelated with $\lt$, i.e., $\E[\lt \vp_t{}' ] = 0$, and $\vp_t$'s are zero-mean, 
 independent and identically distributed (i.i.d.) with covariance $\Sig_v := \E[\vp_t \vp_t{}']$, and are bounded.
Let $\lambda_v^+ := \|\Sig_v\|$ be the noise power and let $r_v := \frac{\max_t \|\vp_t\|_2^2}{\lambda_v^+}$ be the effective noise dimension.
\label{vt_assu}
\end{assu}
For a decomposition of the data-dependency matrix $\M_t$ as  $\M_t = \M_{2,t} \M_{1,t}$ with $\|\M_{2,t}\|=1$, let 
\begin{align}\label{eq:m2t_bnd_new}
q:= \max_t \|\M_{1,t} \P\| , \ \text{and} \\
b := \norm{\frac{1}{\tmaxpca} \sum_{t=1}^\alpha \M_{2,t} \M_{2,t}{}'}. 
\end{align}
Observe that $b \le \max_t \|\M_{2,t}\|^2=1$. In many settings, for example, when $\wt$ is sparse with changing support, $b$ is much smaller than one. Our result given below exploits this fact.

\begin{theorem}[PCA in Data-Dependent Noise]\label{mainthm_pca}
Consider the data $\yt$ defined by \eqref{eq:pca_dd}; and assume that Assumptions \ref{def_right_incoh} and \ref{vt_assu} hold.
Also assume that $\w_t = \M_t \lt$ with the parameters $b,q$ satisfying $b < 1$, $q < 2$, and  
$
4 \sqrt{b} q f + \frac{\lambda_v^+}{\lambda^-} + \epsbnd + \epsdenb < 1.
$
Here,
\begin{gather}
\epsbnd := C \sqrt{\eta} q f  \sqrt{\frac{r \log n}{\tmaxpca}}    +  C \sqrt{\eta} \sqrt{\frac{\lambda_v^+}{\lambda^-}  f}  \sqrt{\frac{r \log n}{\tmaxpca}}, \nonumber \\
\epsdenb := C \sqrt{\eta}  f \sqrt{\frac{r \log n}{\tmaxpca}}.
\end{gather}
Then, with probability at least $1 - 10n^{-10}$, the matrix of top $r$ eigenvectors of $\frac{1}{\tmaxpca}\sum_t \yt\yt{}'$, $\Phat$, satisfies
\begin{align*}
\SE(\Phat, \P) \leq \frac{4 \sqrt{\bz}qf + \frac{\lambda_v^+}{\lambda^-}  + \epsbnd}{1 - 4 \sqrt{\bz}qf - \frac{\lambda_v^+}{\lambda^-} -\epsbnd - \epsdenb}
\end{align*}
\end{theorem}
Theorem \ref{mainthm_pca} is proved in Appendix \ref{proofs_pca_section}. It uses the Davis-Kahan $\sin \Theta$ theorem \cite{davis_kahan} followed by matrix Bernstein \cite{tropp} to bound each term.
%
To understand Theorem \ref{mainthm_pca} simply, first assume that $\vp_t=0$ and $\epsbnd,\epsdenb$ are small enough ($\alpha$ is large enough).
From the definition of $q$, the instantaneous signal-noise correlation $\|\E[\lt \wt{}']\| \le q \lambda^+$ and the instantaneous data-dependent noise power $\|\E[\wt \wt{}']\| \le q^2 \lambda^+$. Thus $q^2$ is the data-dependent noise-to-signal ratio. Also, $\lambda^+$ and $\lambda^-$ quantify the maximum and the minimum signal power respectively.
The PCA subspace recovery error depends on the ratio between the sum of (time-averaged values of) signal-noise correlation and noise power and the minimum signal space eigenvalue $\lambda^-$.
%
By Cauchy-Schwarz, it is not hard to see that the time-averaged values of both these quantities satisfies $\|\frac{1}{\tmaxpca} \sum_{t=1}^\tmaxpca \E[\wt \wt{}']\| \le \sqrt{\bz} q^2 \lambda^+$ and  $\|\frac{1}{\tmaxpca} \sum_{t=1}^\tmaxpca \E[\lt \wt{}']\| \le \sqrt{\bz}q \lambda^+$.
Thus, if $b \ll 1$, the time-averaged values are significantly smaller than the instantaneous ones and this is what helps us get a small bound on the subspace recovery error. For a constant $c_1<1$, by assuming $b < (c_1/4f)^2$, we can ensure that $\SE(\Phat,\P) \le c_1 q$, i.e., the subspace recovery error is a fraction of $q$.

In the general case when $\vp_t \neq 0$, 
we can guarantee that $\SE(\Phat,\P)$ is at most $c_1 \max( q,  \lambda_v^+ / \lambda^-)$. 

\subsection{Application to PCA in Sparse Data-Dependent Noise (PCA-SDDN)}
An important application of the above result is for data-dependent noise, $\wt$, that is sparse.
In this work we will show how a guarantee for PCA in sparse data-dependent noise (PCA-SDDN) helps obtain a  fast and delay-optimal robust ST algorithm. 
If we set $\M_{2,t} = \I_{\T_t}$ then $\wt$ is sparse with support $\T_t$. Thus for $t =1,2, \cdots, \alpha$ 
\bea \label{yt_mod}
\y_t:= \lt + \wt + \vp_t, \text{ where } \lt = \P \at, \  \wt = \I_{\T_t} \M_{s,t} \lt,
\eea
The assumption on $b$ is now equivalent to a bound on the maximum fraction of non-zero entries in any row of $\W:=[\w_1, \cdots, \w_{\tmaxpca}]$. To see why this is true, notice that $b = \frac{1}{\alpha} \|\sum_{t=1}^\alpha \I_{\T_t} \I_{\T_t}{}'\|$. The matrix $\sum_t \I_{\T_t} \I_{\T_t}{}'$ is a diagonal matrix with $(i,i)$-th entry equal to the number of times $t$ for which $i \in \Tt$. This is the same as the number of nonzero entries in the $i$-th row of $\W$.
Using this fact we get the following corollary.%

\begin{corollary}[PCA in Sparse Data-Dependent Noise] 
Assume that  $\yt$'s satisfy \eqref{yt_mod}, Assumptions \ref{def_right_incoh}, \ref{vt_assu} hold, and $q:=\max_t \|\M_{s,t} \P\| \le 2$. Let $\bz$ denote the maximum fraction of nonzeros in any row of the noise matrix $[\w_1, \w_2, \dots, \w_\tmaxpca]$, and let $g: = \frac{\lambda_{v}^+}{\lambda^-}$.
For an $\epsilon_\SE > 0$, if
\[
4 \sqrt{\bz}q f +  g < 0.4 \epsilon_\SE,   
\]
and if 
\[
\tmaxpca \ge \tmaxpca^* := C \max\left( \frac{q^2 f^2}{\epsilon_\SE^2} r \log n, \frac{g f}{\epsilon_\SE^2} \max(r_v,r) \log n\right),
\]
then w.p. at least $1- 10n^{-10}$, $\SE(\Phat,\P) \le \epsilon_\SE$.
\label{cor:noisy_pca_sddn}
\end{corollary}
This corollary follows from Theorem \ref{mainthm_pca} by picking $\tmaxpca$ large enough so that $\epsbnd < \epsilon_\SE/10$ and $\epsdenb < 1/10$ (since this term appears in the denominator, we do not need it to be smaller than $\epsilon_\SE$, just a constant upper bound suffices).

Corollary \ref{cor:noisy_pca_sddn} shows that it is possible to achieve recovery error that is a fraction of $q$,  i.e, $\epsilon_\SE = c_1 q$, if (i) $4\sqrt{\bz} f \le 0.8 c_1$ (the data-dependent noise support changes enough over time so that $\bz$ is small), (ii) $\lambda_v^+ \le 0.8 c_1 \epsilon_\SE  \lambda^-$ (the uncorrelated noise power is small enough), and (iii) $\tmaxpca \ge  C \max( f^2 r \log n, f\frac{1}{\epsilon_\SE} \max(r_v,r) \log n)$.  Notice that the sample complexity $\tmaxpca$ increases with $1/\epsilon_\SE = 1/(c_1 q)$. However, if we can make a stronger assumption that $\lambda_v^+ \le 0.8 c_1 \epsilon_\SE^2  \lambda^-$, then we only need $\tmaxpca \ge  C \max( f^2 r \log n, f \max(r_v,r) \log n ) $. Furthermore if  $r_v \le C r$, then just $\tmaxpca \ge C f^2 r \log n$ suffices. Treating $f$ as a numerical constant, observe that this sample complexity is order-wise near-optimal: $r$ is the minimum number of samples needed to even define a subspace.

In particular, in the setting when $\vp_t = 0$, 
if the noise support changes enough so that $b$ is small enough, we can estimate the subspace to a fraction of the square root of the noise level, $q$, using just order $r \log n$ samples. The reason this is possible is because the $\at$'s are bounded and $\wt = \M_t \P \at$ and so the ``randomness'' in $\w_t$ is only $r$-dimensional (this has implications for what matrix Bernstein returns for the required sample complexity). 
When $\vp_t \neq 0$, we have a similar result: if $\vp_t$ has effective dimension that is of order $r$, we can still track to $\epsilon_\SE = c \max(q, \sqrt{g})$, here $g$ is the square root of uncorrelated noise level.



\subsection{Generalizations of Theorem \ref{mainthm_pca}}
For notational simplicity, in Theorem \ref{mainthm_pca}, we have provided a simple result that suffices for the correctness proof of our robust ST algorithm.
We state and prove a much more general result in the Supplement given in the ArXiv version of this work \cite[Appendix IV]{jsait_arxiv} that relaxes this result in three ways.
First,   it replaces the identically distributed assumption on $\at$ and $\vp_t$  by the following:
let $\bar\Lam := \sum_t \Lam_t / \alpha$,  $\lambda_{\avg}^-:= \lambda_{\min}(\bar\Lam)$,
$\lambda_{\max}^+:= \max_t \lambda_{\max}(\Lam_t) $ and $\lambda_{v,\max}^+:= \max_t \lambda_{\max}(\Sig_{v,t})$.
It requires that  the distributions are ``similar'' enough so that  $f:=\lambda_{\max}^+ / \lambda_{\avg}^-$ is bounded by a numerical constant and $\lambda_{v,\max}^+$ replaces $\lambda_v^+$ in $\epsbnd$ and $\epsdenb$ expressions.


Secondly, it replaces  $\lambda_v^+$  by $\|\P' \Sig_{v} \P_\perp\|$ in the numerator, while $-\lambda_v^+$ in the denominator gets replaced by $- (\lambda_{\max}( \Sig_v -  \P \P' \Sig_v \P \P') -  \lambda_{\min}(\P' \Sig_v \P) ) $. 
Here again, in case of time-varying statistics, the minimum eigenvalues get replaced by the minimum eigenvalue of the average covariance matrix while the maximum ones get replaced by the maximum eigenvalue over all times $t$.
%
Thirdly, we also provide a guarantee for the case when $\at$'s and $\vp_t$'s are sub-Gaussian random vectors. In this case, the required sample complexity increases to order $n$ instead of $\max(r,r_v) \log n$ that we have for the bounded case result given above.

These last two changes allow us to recover the well known result for PCA under the Gaussian spiked covariance model (uncorrelated isotropic noise) \cite{nadler} as a special case of our most general result. Spiked covariance means $\wt=0$ and $\Sig_v = \lambda_v^+ \I$. Thus, $q=0$,  $\|\P' \Sig_v \P_\perp\| = 0$ and $\|\Sig_v - \P' \Sig_v \P\| - \|\P' \Sig_v \P\| = 0$ and so we get the following corollary.
\begin{corollary}[Spiked Covariance Model, Gaussian noise \cite{nadler}]
In the setting of Theorem \ref{mainthm_pca}, if $\wt=0$ (no data-dependent noise), $\Sig_v =  \lambda_v^+ \I$, and $\at$, $\vp_t$ are Gaussian,
then,
w.p. at least $1 - 5 \exp(-c n)$,
$
\SE(\Phat, \P) \leq \frac{\epsbnd}{1  -\epsbnd - \epsdenb}.
$
with $\epsbnd = C  \sqrt\eta \sqrt{ g  f} \sqrt{\frac{n}{\tmaxpca}} $, $\epsden =  C  \sqrt\eta  f \sqrt{\frac{n}{\tmaxpca}}$ and $g = \frac{\lambda_{v}^+ }{ \lambda^-}$.

If $\at$, $\vp_t$ are bounded then $\epsbnd, \epsdenb$ are as given in Theorem \ref{mainthm_pca}.
\end{corollary}

Notice that, under the spiked covariance model, as long as we let the sample complexity $\alpha$ grow with the noise level $g$, we do not need any bound on noise power. For example, the noise power $\lambda_v^+$ could even be larger than $\lambda^-$. This is possible because, under this model, $\E[\sum_t \yt \yt' /\alpha] = \P \Lam \P' + \lambda_v^+ \I$. Thus, its matrix of top $r$ eigenvectors equals $\P$. As a result, the error between $\Phat$ and $\P$ is only due to the fact that we are using a finite $\alpha$ to approximate the expected value. In other words, we only have statistical error. The ``bias'' terms are zero. 

\section{Nearly Optimal Robust Subspace Tracking (NORST)}\label{sec:norst}

In this section, we define the robust ST problem, explain the assumptions needed to make it identifiable, and then explain our proposed mini-batch solution and its guarantee. 

\subsection{Problem setting and algorithm design constraints}
At each time $t$, we observe a data vector $\yt \in \Re^n$ that satisfies%
\bea
\yt := \lt + \x_t + \v_t, \text{ for } t = 1, 2, \dots, \tmax
\label{orpca_eq}
\eea
where  $\v_t$ is small unstructured noise, $\xt$ is the sparse outlier vector, and $\lt$ is the true data vector that lies in a fixed or slowly changing low-dimensional subspace of $\Re^n$, i.e.,
\[
\lt = \P_{(t)} \a_t
\]
where $\P_{(t)}$ is an $n \times r$ basis matrix with $r \ll n$ and with $\|(\I - \P_{(t-1)}\P_{(t-1)}{}')\P_{(t)}\|$ small compared to $\|\P_{(t)}\|=1$. 
We use $\T_t$ to denote the support set of $\xt$. As an example, in the video application, $\yt$ is the video image at time/frame $t$, $\lt$ is the background  at time $t$, $\T_t$ is the support of the foreground at $t$, and $\x_t$ equals the difference between foreground and background images on $\T_t$ while being zero everywhere else. Slow subspace change is typically a valid assumption for background images of videos taken using a static camera.
Given a good initial subspace estimate, $\Phat_{0}$, the goal is to develop a mini-batch algorithm to track $\Span(\P_{(t)})$ and $\lt$ either immediately or within a short delay.
A by-product is that $\x_t$, and $\T_t$ can also be tracked accurately. The initial subspace estimate, $\Phat_0$, can be computed  by applying a few iterations of any existing RPCA solutions, e.g., PCP \cite{rpca} or AltProj \cite{robpca_nonconvex}, on the first order $r$ data points, i.e., on $\Y_{[1,t_\train]}$, with $t_\train=Cr$.

\subsubsection{Dynamic RPCA} This is the offline version of the above problem. Define matrices $\L,\S,\V,\Y$ with $\L = [\l_1,\l_2, \dots \l_{\tmax}]$ and with $\Y,\S,\V$ similarly defined.  The goal is to recover $\L$ and its column space with accuracy $\zz$. We use $\rmat$ to denote the rank of $\L$. 
The maximum fraction of nonzeros in any row (column) of the outlier matrix $\S$ is denoted by $\outfracrow$ ($\outfraccol$).

\subsubsection{Algorithm constraints}
We will develop a nearly real-time tracking algorithm that (i) computes {\em an} online estimate of $\xt$ and its support $\T_t$, and of $\lt$ immediately at each time $t$ using the previous subspace estimate, $\Phat_{(t-1)}$, and observed data $\yt$; (ii) it updates the subspace estimates in a mini-batch fashion; and (iii) it provides improved smoothing estimates of all quantities after a delay that is within log factors of $r$. 
As we explain in Sec. \ref{tradeoffs}, recovering $\xt$, $\T_t$, and $\lt$, one at a time is the only way to obtain improved row-wise outlier tolerance compared to standard RPCA. However with doing this, correct recovery requires one extra assumption: slow enough subspace change compared to the minimum outlier magnitude. 


\subsection{Nearly Optimal Robust ST (NORST) via Recursive Projected Compressive Sensing (CS): main idea}
The algorithm begins with an initial subspace estimate $\Phat_0$. At each time $t$, we use $\Phat_{(t-1)}$ and $\yt$ to solve a noisy projected compressive sensing (CS) problem to estimate $\xt$ and its support $\T_t$ from $\tty_t = \bpsi \xt + \bt$. Here $\bpsi = \I -  \Phat_{(t-1)}\Phat_{(t-1)}{}'$, $\tty_t = \bpsi \yt$, and $\bt = \bpsi \lt +\bpsi \vt $ (is small under the slow subspace change assumption). This step uses $l_1$ minimization followed by thresholding to estimate $\T_t$, and Least Squares (LS) on $\That_t$ to get $\xhat_t$.
We compute $\lhat_t$ by subtraction as $\lhat_t = \yt - \xhat_t$. Every $\alpha$ time instants, we update the subspace estimate by solving the PCA problem using the previous $\alpha$ $\lhat_t$'s as observed data, i.e., by $r$-SVD on $\Lhat_{t;\alpha}$. This  is repeated $K$ times, each time with a new set of $\alpha$ $\lhat_t$'s. At this point, the algorithm enters the subspace change detect phase. The complete algorithm is specified in Algorithm \ref{algo:auto-reprocs-pca}, and explained in detail in Sec \ref{algo_details}. Besides $\alpha$ and $K$, it has two other parameters: $\xi$ (assumed upper bound on $\|\bt\|$) and $\omega_{supp}$ (threshold used for support recovery).



\subsection{Identifiability and other assumptions}
For this discussion assume that $\vt =0$. At each time $t$ we have just one $n$-length observed data vector $\yt$ but the subspace $\Pt$ is specified by $nr$ scalars (it is an $r$-dimensional subspace of $\Re^n$).
Thus, even if we had perfect data $\yt=\lt$ available, it would be impossible to exactly recover each different $\Pt$. 
One way to address this is by assuming that the $\Pt$'s do not change for at least $r$ time instants. 
\begin{assu}[Piecewise Constant Subspace Change]\label{def_pw_ss}
Let $t_1, \dots t_j, \dots t_J$ denote the subspace change times. Let $t_0=1$ and $t_{J+1}=\tmax$.
Assume that 
\[
\P_{(t)} = \P_j \text{ for all } t \in [t_j, t_{j+1}), \ j=1,2,\dots, J,
\]
with $t_{j+1} - t_j > r$. Since $\yt=\lt +\xt$ (is imperfect), our guarantee needs a larger lower bound than $r$.
\end{assu}


Even with the above assumption, a sparse $\xt$ and its support $\T_t$ cannot be correctly distinguished from $\lt = \P_j \at$ without more assumptions. Correct recovery of $\xt$ and $\T_t$ requires that (i) the $\xt$'s are sparse enough (ensured by bounding the maximum allowed outlier fractions per column), (ii) the columns of $\P_j$ are not sparse (ensured by the standard incoherence/denseness assumption from the RPCA literature \cite{matcomp_candes,rpca,robpca_nonconvex}), and (iii) the $\at$'s are bounded. (iv) Correct support recovery also requires subspace change that is slow enough compared to the minimum nonzero entry of $\xt$ (minimum outlier magnitude), denoted $\xmint$.
%
%
%
%
Correct subspace update requires that (v) the $r \times \alpha$ sub-matrices formed by a mini-batch of $\at$'s are well-conditioned, and (vi) the outlier support $\T_t$ changes enough over time so that there is at least  one outlier-free observation of each scalar entry of $\lt$ in each mini-batch of $\yt$'s.  One way to ensure (v) is to assume that the $\at$'s are i.i.d. while (vi) can be ensured by bounding the maximum fraction of outliers in any row of any $\alpha$-mini-batch sub-matrix of $\X$. We use $\outfracrow(\alpha)$ to denote this quantity.
We summarize the above assumptions on $\P_j$'s and $\at$'s in Assumption \ref{defmu}, those on the outlier fractions in  Assumption \ref{def_outfrac}, and slow subspace change compared to $\xmint$ in Assumption \ref{slow_ss}.

\begin{assu}[$\mu$-Incoherence]\label{defmu}
Assume the following.
\begin{enumerate}
\item (Left Incoherence) Assume that $\P_j$'s are $\mu$-incoherent with $\mu$ being a numerical constant. This means that
$
\max_{i=1,2,.., n} \|(\P_j)^{(i)}\|^2 \le \mu r/ n.
$
Here $\P^{(i)}$ denotes the $i$-th row of $\P$.

\item (Statistical Right Incoherence) Assume Assumption \ref{def_right_incoh}, i.e., the subspace coefficients $\at$ are zero mean, mutually independent, have identical diagonal covariance matrix $\Lam:=\E[\at \at']$, and are bounded: $\max_t \|\at\|^2 \le \mu r \lambda_{\max}(\Lam)$. Let $\lambda^+$ ($\lambda^-$), $f := \lambda^+/\lambda^-$ denote the maximum (minimum) eigenvalue and condition number of $\Lam$.
\end{enumerate}
\end{assu}
The second assumption above allows us to obtain high probability upper bounds on the tracking delay of our approach. As we explain later in Sec. \ref{why_stat_right_incoh}, it can be interpreted as a statistical version of right singular vectors' incoherence. The incoherence assumption on $\P_j$ is nearly equivalent to left singular vectors' incoherence. It is exactly equivalent if we consider the sub-matrices $\L_j:=[\l_{t_j},\l_{t_j+1}, \dots, \l_{t_{j+1}-1}]$.

\cred
\begin{assu}[Outliers are spread out]\label{def_outfrac}
Let $\outfraccol:=\max_t|\T_t|/n$; let $\outfracrow(\alpha)$ be the maximum fraction of nonzeros per row of any sub-matrix of $\X_{[t_\train,d]}$ with $\alpha$ consecutive columns, and let $\outfracrow_{\init}$ be the maximum fraction of outliers  per row of any sub-matrix of $\X_{[1,t_\train]}$.
Assume that $\outfraccol \leq \frac{c_1}{\mu r}$, $\outfracrow(\alpha) \leq \frac{c_2}{f^2}$, and $\outfracrow_{\init} \le \frac{c_3}{r}$.%
\end{assu}

\begin{assu}[Slow subspace change] \label{slow_ss}
Let $\xmint: = \min_t \min_{i \in \Tt} |(\xt)_i|$ and let $\SE_j:= \SE(\P_{j-1}, \P_j)$.
Assume that  $\SE_j \le 0.8$ and $\SE_j \le \frac{c_4}{\sqrt{r}} \frac{\xmint}{\sqrt{\lambda^+}}$.
\end{assu}

{\em The order notation used here and below assumes that $f, \mu$ are constants.}

\subsection{Guarantees}
Before stating our main result, we define a few terms next.
\begin{definition}
Let the mini-batch size $\alpha := C f^2  r \log n$, 
the number of subspace update iterations needed to get an $\zz$ accurate estimate, $K =  K(\zz):=C \log (\dif/\zz)$, where $\Delta:=\max_j \SE_j$, 
noise power $\lambda_v^+ := \max_t \|\E[\vt\vt{}']\|$, and effective noise dimension, $r_v := \frac{\max_t \|\vt\|^2}{\lambda_v^+}$.
Recall from Algorithm  \ref{algo:auto-reprocs-pca} that $\that_j$ denotes the time at which the $j$-th subspace change is detected.

\end{definition}

We have the following result.

\begin{theorem}
Assume that Assumptions \ref{def_pw_ss}, \ref{defmu}, \ref{def_outfrac}, and \ref{slow_ss} hold. Assume that the noise $\vt$ is bounded, i.i.d. over time,  independent of $\T_t$, uncorrelated with $\lt$, i.e., $\E[\lt \vt{}'] = 0$, and with $r_v \leq Cr$, and $ \sqrt{\lambda_v^+/\lambda^-} < 0.01$. Also, assume that $\lt$'s and $\T_t$'s are independent.

Pick an $\zz$ that satisfies $c \sqrt{\lambda_v^+/\lambda^-} \leq \zz  \leq \min \left( c_3 \frac{1}{\sqrt{r}} \frac{\xmint}{\sqrt{\lambda^+}} , 0.01 \right)$.
Consider Algorithm \ref{algo:auto-reprocs-pca} with $K = K(\zz)$ as defined above,  $\alpha=Cf^2 r \log n$, $\omega_{evals} = 2 \zz^2 \lambda^+$,
$\zeta = \xmint/15$ and $\omega_{supp} = \xmint/2$.
If

\begin{enumerate}

\item $\max_t \|\vt\| \le c_5  \xmint  $,

\item  $t_{j+1}-t_j > (K+2)\alpha$, and $\SE_j > 9  \sqrt{f} \zz$

\item  initialization\footnote{This can be satisfied by using $C \log r$ iterations of AltProj \cite{robpca_nonconvex} on the first $t_\train = C r$ data samples.}:  $\SE(\Phat_0,\P_0) \le \min \left( c_6 \frac{1}{\sqrt{r}} \frac{\xmint}{\sqrt{\lambda^+}} , 0.25 \right)$;
\end{enumerate}

then, w.p. at least $1 - 10 \tmax n^{-10} $, 
\ben
\item $t_j \le \that_j \le t_j+2 \alpha$,
\begin{align*}
&\SE(\Phat_{(t)}, \P_{(t)})  \le \\
 &\left\{
\begin{array}{ll}
(\zz + \SE_j) & \text{ if }  t \in [t_j, \that_j+\alpha), \\
 (0.3)^{k-1} (\zz + \SE_j) & \text{ if }  t \in [\that_j+(k-1)\alpha, \that_j+ k\alpha), \\
\zz   & \text{ if }  t \in [\that_j+ K\alpha+\alpha, t_{j+1}),
\end{array}
\right.
\end{align*}
and $\|\lhat_t-\lt\| \le 1.2 \SE(\Phat_{(t)}, \P_{(t)})  \|\lt\|  + \|\vt\|$;

\item $\That_t =\T_t$ and the bound on $\|\xhat_t - \xt\|$ is the same as that on $\|\lhat_t-\lt\|$.
\een
The time complexity is $O(n \tmax r  \log (1/\zz) )$ and memory complexity is $O(n \alpha) = O(f^2 n r \log n)$.
\label{thm1}
\end{theorem}
\begin{proof} We prove this in Sec. \ref{sec:proof_norst}. \end{proof}

We have the following corollary for the smoothing NORST algorithm (last few lines of Algorithm \ref{algo:auto-reprocs-pca}). This is also a mini-batch approach with mini-batch size $(K+2) \alpha$ (instead of $\alpha$ for NORST).
\renewcommand{\offline}{\mathrm{smoothing}}
\begin{corollary}\label{cor:thm1}[Smoothing NORST for dynamic RPCA]
Under the assumptions of Theorem \ref{thm1}, the following also hold:
 $\SE(\Phat_{(t)}^{\offline}, \P_{(t)})  \le \zz$, $\|\lhat_t^{\offline}-\lt\| \le \zz \|\lt\| + \|\vt \|$ at all times $t$. Its time complexity is $O(n \tmax r  \log (1/\zz) )$ and memory complexity is $O(K n \alpha) = O(n r \log n \log (1/\zz) )$. All these quantities are computed within a delay of at most $(K+2)\alpha$.
\end{corollary}

The above result guarantees that NORST can detect subspace changes in delay at most $\alpha = C r\log n $ and track them to $\zz$ accuracy in delay at most $(K+2) \alpha =C r\log n  \log (\Delta/\zz) $. The corollary for smoothing NORST guarantees that, with this delay, each column of $\L$, $\lt$, is recovered to $\zz$ relative accuracy.
The minimum delay needed to compute an $r$-dimensional subspace even with perfect data $\yt=\lt$ is $r$. Thus, our result guarantees near optimal detection and tracking delay (``near optimal'' means that it is within log factors of the minimum delay). Moreover, the required lower bound on the delay between subspace change times is also near optimal.
Quick and reliable change detection is an important feature, e.g., this feature has been used in \cite{selin_reprocs} to detect structural changes in a dynamic social network.

 When the extra unstructured noise $\vt=0$, we can track to any $\zz>0$ otherwise we can track to $\zz  \ge \sqrt{\lambda_v^+/\lambda^-}$ (square root of the noise level). It is possible to slightly relax this requirement to $\zz \ge {\lambda_v^+/\lambda^-}$ by picking a larger $\alpha$,  $\alpha = C  (r \log n) (\lambda^- / \lambda_v^+)$, but it cannot be eliminated. 
The reason is that at each time $t$, we have an {\em under-determined} set of equations corrupted by unstructured noise $\vt$. Even assuming the subspace is known or has been perfectly estimated, it is under-determined: we have $n + r$ unknowns at each time $t$  but only $n$ observed scalars. This is also true for any other under-determined problem as well, e.g., standard RPCA or CS\footnote{To address a reviewer comment, one cannot get a consistent estimator for our problem, nor for standard RPCA or CS. Consistent estimator means that the recovery error goes to zero as the number of observed data points increases. For example for Least Squares estimation, one can show the estimator is consistent. But this is true because number of observed data points increases while the number of unknowns remains constant. In our case, the number of unknowns also increases with time $t$: at each $t$, we have $(n+r)$ unknowns even if $\Pt$ has been estimated.}.

Notice also that we have assumed that the ``effective noise dimension'', $r_v \in O(r)$. This requirement can be eliminated if we set $\alpha = C f^2 \max(r,r_v) \log n$. 

\cred
From the perspective of recovering the true data $\lt$, both $\vt$ and $\xt$ are noise or perturbations. The difference is that $\vt$ is a vector of small disturbances or modeling errors, while $\xt$ is a  sparse outlier vector with few nonzero entries. By definition, an outlier is an infrequent but large disturbance.
Our result tolerates what can be called ``bi-level perturbations'': the small perturbation $\vt$ needs to be small enough and the minimum outlier magnitude $\xmint$ needs to be large enough so that $\|\vt\| \le 0.2 \xmint$ (minimum outlier magnitude).
Moreover, $\xmint$ also needs to be large enough to satisfy Assumption \ref{slow_ss}. The need for both these assumptions is explained in Sec. \ref{tradeoffs}.
%
Assuming that $\xmint^2$ is of order ${\lambda^+}$ (signal power), Assumption \ref{slow_ss} requires that $\SE_j$ be $O(1/\sqrt{r})$. However this is not as restrictive as it may seem. The reason is that $\SE(.)$ is only measuring the sine of the largest principal angle. If all principal angles are roughly equal, then, this still allows the chordal subspace distance  ($l_2$ norm of the vector of sines of all $r$ principal angles) \cite{chordal_dist} to  be $O(1)$.
\cbl

Our result assumes a minor lower bound on $\SE_j$. This is needed to guarantee reliable subspace change detection. Changes that are smaller than order $\zz$ cannot be detected when the previous subspace is only tracked to accuracy $\zz$. However, such changes also increase the tracking error only by an extra factor of $\zz$ and hence can be treated as noise. If change detection is not important, then, as we explain in Sec. \ref{extends}, we can use a simpler NORST algorithm that does not need the lower bound.

Consider the piecewise constant subspace change assumption. In practice, e.g., in the video application, typically the subspaces change by a little at each time. This can be modeled as piecewise constant subspaces plus modeling error $\vt$. We explain this point in  Sec. \ref{extends} where we also provide a corollary for this setting. This corollary explains why the NORST algorithm ``works''  (gives good, but not perfect, subspace estimates and estimates of $\lt$) for real videos or for simulated data generated so that $\Pt$ changes a little at each $t$; see Sec. \ref{sec:sims} and more detailed experiments in \cite{rrpcp_review}.

To keep the theorem statement simple, we have used tighter bounds than required. Define the intervals $\J_{j,1} = [t_j, \that_j + \alpha)$, $\J_{j,k} := [\that_j + (k-1)\alpha, \that_j + k\alpha)$ for $k=2,3,\dots, K$, and $\J_{j,K+1} = [\that_j + (K+1)\alpha, t_{j+1})$. For $t \in \J_{j,k}$, for $k=1,2,\dots, K$, we only need $0.3^{k-1} (\zz+ \SE_j) \sqrt{r \lambda^+} \le c \min_{i \in \Tt} |(\xt)_i|$, i.e., the required lower bound on the minimum outlier magnitude at time $t$ decreases as the subspaces get estimated better. For the outliers $\xt$ for $t \in \J_{j,K+1} $, we do not require any lower bound. Secondly, if the outlier vector is such that some entries are very small while the others are large enough, then we can treat the smaller entries as ``noise'' $\vt$. This will work as long as these small entries are small enough so that the sum of their squares is sufficiently smaller than the square of the magnitude of the larger entries, i.e., for $t \in \J_{j,k}$, we can split $\xt$ as $\xt = (\xt)_{small} + (\xt)_{large}$ with the two components being such that $c \x_{t,large,min} \ge \|(\xt)_{small}\|$  and $c  \x_{t,large,min} \ge 0.3^{k-1} (\zz+ \max_j \SE_j)\sqrt{r \lambda^+}$.
Finally, if we also state the PCA-SDDN result in its most general form, the subspace error decay rate of $0.3$ can be replaced by $(6\sqrt{b_0} f)$ with $b_0:= \outfracrow$, so this requirement becomes $c  \x_{t,large,min} \ge (6\sqrt{b_0} f)^{k-1} (\zz+ \max_j \SE_j)\sqrt{r \lambda^+}$. With this change, the expression for $K$ becomes $K = \left \lceil \frac{\log \left( \semax / \zz \right)}{-\log(6 \sqrt{b_0} f)}  \right \rceil.$
Thus, a smaller $b_0$ means that the subspace error decays faster. This, in turn, means that a smaller $K$ suffices (faster tracking and a smaller required lower bound on $t_{j+1}-t_j$). It also means a smaller lower bound is needed on the outlier magnitudes at most times.



\subsection{How slow subspace change (Assumption \ref{slow_ss}) enables improved outlier tolerance}
\label{tradeoffs}
We explain here how the use of Assumption \ref{slow_ss} enables improved outlier tolerance. Briefly, the reason is we recover each outlier $\xt$ and its support $\T_t$ individually. To understand things simply, assume $\vt=0$.

Given a good previous subspace estimate, $\Phat_{(t-1)}$, slow subspace change implies that $\SE(\Phat_{(t-1)},\Pt)$ is small. 
Consider an $\alpha$ length interval $\J$ during with $\Phat_{(t-1)} = \Phat$ (computed in the previous $\alpha$ interval).
To exploit slow subspace change, we project each $\yt$ orthogonal to $\Phat$ to get $\tty_t : = \bpsi \st + \b_t $ where $\b_t:= \bpsi \lt $ is small because of above. Here $\bpsi:= \I - \Phat \Phat{}'$. 
Now $\bt$ itself does not have any structure. But, the matrix $\B_{\J}$ formed by the $\b_t$'s for $t \in \J$, is low rank with rank $r$ \footnote{The effective (stable) rank, of $\B_{\J}$  will be less than $r$ only if we assume more structure on subspace change, e.g., if we assume that only a few subspace directions change. Its exact rank will still be $r$.}.
Accurately recovering $\S_\J$ from $\tilde{\Y}_\J:= \bpsi \S_\J + \B_\J$ when $\B_\J$ has rank $r$ is impossible if the fraction of outliers in any row or in any column of $\S_\J$ is more than $c/r$. The reasoning is the same as that used for standard RPCA \cite{robpca_nonconvex}: we can construct a sparse matrix $\S_\J$ with rank $1/\max(\outfracrow,\outfraccol)$.
%
Thus if $\outfracrow = c$, we can construct a sparse $\S_\J$ with rank $1/c= C \ll r$ \footnote{A simple way to do this would be as follows. Let $b_0 = \outfracrow$ and suppose $b_0$ is a constant (is more than order $1/r$).
Let the support and nonzero entries of $\X_\J$ be constant for the first $b_0 \alpha$ columns; after this, move the nonzero entries down in such a way that there is no overlap of supports; and repeat this every $b_ 0 \alpha$ columns. With this, the $\outfracrow = b_0$ and the rank of $\X_\J$ is $\alpha/(b_0 \alpha) = 1/b_0$ since there are only $1/b_0$ unique vectors in this matrix construction.}.
If the rank of $\S_\J$ is less than $r$, that of $\bpsi \S_\J$ will also be less than $r$, making the recovery problem un-identifiable: if we try to find a matrix $\hat\B_\J$ of rank at most $r$ and a matrix $\hat\X_\J$ that is the sparsest and both satisfy $\tilde{\Y}_\J:= \bpsi \hat\S_\J + \hat\B_\J$, it is possible that we get the solution $\hat\B_\J = \B_\J + \bpsi \hat\S_\J$ and $\hat\X_\J = \bm{0}$. Because rank of $\bpsi \hat\S_\J$ is less than $r$ and that of $\B_\J$ is $r$, it is possible that the sum still has rank $r$. 



Thus, if we would like to improve row-wise outlier tolerance to $O(1)$, we cannot jointly recover all columns of  $\S_{\J}$ by exploiting the low rank structure of $\B_{\J}$. 
The only other way to proceed is as we do: recover them one $\st$ at a time from $\tty_t$. Here we can only use the fact that $\|\bt\|$ is small due to slow subspace change. The problem of recovering  a single $\st$ from $\tty_t$ is a standard noisy CS problem \cite{candes_rip}, with small noise $\bt$. To our best knowledge, there are no entry-wise recovery guarantees for CS. One can only bound $\|\xhat_{t,cs} - \xt\|$ by a constant (that depends on the restricted isometry constant of $\bpsi$), $C$, times $\|\bt\|$. Here $\xhat_{t,cs}$ is the output of the CS step (line 7 of Algorithm \ref{algo:auto-reprocs-pca}).
With this, correct support recovery, $\That_t = \T_t$, is ensured only if $\xmint > 2 C \|\bt\| $.
The worst case bound on $\|\bt\|$ comes from when the subspace has changed but the change has not been detected so that $\Phat_{(t-1)} = \Phat_{j-1}$ and $\Pt = \P_j$. At this time, $\|\bt\| \le  \max_j \SE(\Phat_{j-1}, \P_j) \sqrt{r}  \sqrt{\lambda^+}$. Also, we can show that $ \SE(\Phat_{j-1}, \P_j) \le \SE_j + \zz$. Thus, exact support recovery is guaranteed if Assumption \ref{slow_ss} holds and $\zz$ is chosen as specified in the theorem.
%
When $\vt \neq 0$, the bound on $\|\bt\|$ contains a $\|\vt\|$ term. In this case, exact support recovery also needs $\|\vt\| \le c_5 \xmint$.

Exact support recovery followed by LS on the recovered support and then subtraction to get $\lhat_t$ implies that $\lhat_t$ satisfies $\lhat_t  = \lt + \et$ with $\et:= -  \I_{\T_t} \underbrace{(\I_{\T_t}{}' \bpsi \I_{\T_t}{})^{-1} \I_{\T_t}{} \bpsi }_{\M_{s,t}} \lt$. Notice that $\et$ is sparse and linearly data-dependent and, conditioned on $\Phat$ and the support sets $\T_t$, the matrix $\M_{s,t}$ is deterministic. So we can apply the PCA-SDDN result from the previous section. It also needs statistical right incoherence, $q := \max_t \|\M_{s,t} \P_j\| \le C \SE(\Phat,\P_j) $ (holds by left incoherence and $\outfraccol < c/r$), and  $\outfracrow (\alpha) \le c$ (constant row-wise outlier fraction bound). If the support recovery were incorrect, the estimated support $\That_t$ would depend on $\b_t$ and hence on $\lt$. This would mean that, even conditioned on $\Phat$ and $\T_t$, the matrices $\M_{s,t}$ are not deterministic making the PCA-SDDN result inapplicable.%

\begin{algorithm}[t!]
\caption{\small{NORST Algorithm. We obtain $\Phat_0$ by $C (\log r)$ iterations of AltProj on $\Y_{[1,t_\train]}$, $t_\train=C r$.}}
\label{algo:auto-reprocs-pca}
\label{norst}
\begin{algorithmic}[1]
\State \textbf{Input}:  $\Phat_0$, $\yt$; ~ \textbf{Output}:  $\shatt$, $\lhatt$,  $\Phat_{(t)}$; ~  \textbf{Parameters:} $\omega_{supp}$, $\xi$, $\alpha$, $K$,  $\lthres$
\State $\Phat_{(t_\train)} \leftarrow \hat{\pt}_{0}$;  $\tildej~\leftarrow~1$, $k~\leftarrow~1$
\State $\mathrm{phase} \leftarrow \mathrm{update}$; $\that_{0} \leftarrow t_\train$;
\For {$t > t_\train$}
\State $\bpsi \leftarrow \bm{I} - \hat{\pt}_{(t-1)}\hat{\pt}_{(t-1)}{}'$
\State $\tty_t \leftarrow \bpsi \yt$.
\State $\xhat_{t,cs} \leftarrow \arg\min_{\tilde{\bm{x}}} \norm{\tilde{\bm{x}}}_1 \ \text{s.t.}\ \norm{\tilde{\bm{y}}_t - \bpsi \tilde{\bm{x}}} \leq \xi$.
\State $\That_t \leftarrow \{i:\ |\xhat_{t,cs}| > \omega_{supp} \}$.
\State $\xhat_t \leftarrow \I_{\That_t} ( \bpsi_{\That_t}{}' \bpsi_{\That_t} )^{-1} \bpsi_{\That_t}{}'\tty_t$.
\State $\lhat_t \leftarrow \yt - \xhat_t$
\If{$\text{phase} = \text{detect}$ and $t = \hat{t}_{j-1, fin} + u\alpha$}
\State $\bphi \leftarrow (\I - \Phat_{j-1}\Phat_{j-1}{}')$.
\State $\bm{B} \leftarrow \bphi\Lhat_{t, \alpha}$ with $\Lhat_{t, \alpha}:=[\lhat_{t-\alpha+1}, \lhat_{t-\alpha+2},\dots \lhat_t]$.  
\If {$\lambda_{\max}(\bm{B}\bm{B}{}') \geq \alpha \lthres$}
\State $\text{phase} \leftarrow \text{update}$, $\hat{t}_j \leftarrow t$,
\EndIf
\EndIf
\If {$\text{phase} = \text{update}$}
\If {$t = \that_j + u \alpha - 1$ for $u = 1,\ 2,\ \cdots,$}
\State $\Phat_{j, k} \leftarrow  \SVD_r[\Lhat_{t; \alpha}]$, $\Phat_{(t)} \leftarrow \Phat_{j,k}$, $k \leftarrow k + 1$.
\Else
\State $\Phat_{(t)} \leftarrow \Phat_{(t-1)}$
\EndIf
\If{$t = \that_j + K\alpha - 1$}
\State $\hat{t}_{j, fin} \leftarrow t$, $\Phat_{j} \leftarrow \Phat_{(t)}$
\State $k \leftarrow 1$, $j \leftarrow j+1$, $\text{phase} \leftarrow \text{detect}$.
\EndIf
\EndIf
\EndFor
\State {\bf Smoothing NORST: }
At $t = \that_j + K \alpha$, for all $t \in [\that_{j-1}+ K \alpha,  \that_j + K \alpha-1]$,
\State $\Phat_{(t)}^{\offline} \leftarrow
basis([\Phat_{j-1}, \Phat_j])$, where $basis(\M)$ refers to a basis matrix that has span equal to $\Span(\M)$.
\State $\bpsi \leftarrow \I - \Phat_{(t)}^{\offline} \Phat_{(t)}^{\offline}{}'$;  \
 $\xhatt^{\offline} \leftarrow \I_{\That_t} (\bpsi_{\That_t}{}'\bpsi_{\That_t})^{-1} \bpsi_{\That_t}{}' \yt$; \
 $\lhatt^{\offline} \leftarrow \yt - \xhatt^{\offline}$.

\end{algorithmic}
\end{algorithm}
\vspace{-.5cm}

\subsection{Understanding Statistical Right Incoherence}\label{why_stat_right_incoh}

\newcommand{\vv}{\bm{v}} 

Let $\L_j := \L_{[t_j, t_{j+1})}$.
From our assumptions, $\L_j = \P_j \A_j$ with $\A_j:= [\a_{t_j},\a_{t_j+1},\dots \a_{t_{j+1}-1}]$, the columns of $\A_j$ are zero mean, mutually independent, have identical covariance $\Lam$, $\Lam$ is diagonal, and  bounded. Let $d_j := t_{j+1}-t_j$.
Define a diagonal matrix $\Sigma$ with $(i,i)$-th entry $\sigma_i$ satisfying $\sigma_i^2 := \sum_t (a_t)_i^2 / d_j$. Define a $d_j \times r$ matrix $\tilde\V$ with the $t$-th entry of the $i$-th column being $(\tilde\vv_i)_t: = (\a_t)_i/ (\sigma_i \sqrt{d_j})$. Clearly, $\L_j = \P_j \Sigma \tilde{\V}'$ and each column of $\tilde\V$ is unit 2-norm. This can be interpreted as an approximation to the SVD of $\L_j$; we say approximation because the columns of $\tilde\V$ are not necessarily exactly mutually orthogonal.
However, if $d_j$ is large enough, one can argue using  scalar Hoeffding inequality (applicable because $\at$'s are bounded), that, whp,
(i) the columns of $\tilde\V$ are approximately mutually orthogonal, i.e. $|\tilde\vv_i' \tilde\vv_j | \le \epsilon$ for all $i \neq j$; and
(ii) $0.99 \lambda_i \le \sigma_i^2 \le 1.01 \lambda_i $ for all $i=1,2, \dots, r$. Thus, by the boundedness assumption on the $\at$'s,  the $t$-th row of $\tilde\V$ satisfies $\sum_{i=1}^r (\tilde\vv_i)_t^2  \le (1/d_j) (1/\min_i \sigma_{i}^2) \|\at\|^2 \le  (1/d_j) (1/\lambda^-) \mu r \lambda^+ = f \mu r / d_j$. This is the standard incoherence assumption with parameter $f \mu$.
Thus, whp, the approximate right singular vectors' matrix $\tilde\V$ of $\L_j$ satisfies the standard incoherence assumption.


\subsection{Nearly Optimal Robust ST via ReProCS (NORST-ReProCS): details}\label{algo_details}
Algorithm \ref{norst} uses the Recursive Projected Compressive Sensing framework introduced in \cite{rrpcp_perf}. 
It starts with a ``good'' estimate of the initial subspace. This can be obtained  by using a few iterations 
of AltProj applied to  $\Y_{[1,t_\train]}$ with $t_\train=Cr$.
It then iterates between (a) {\em Projected Compressive Sensing (CS) / Robust Regression\footnote{Robust Regression (with a sparsity model on the outliers) assumes that observed data vector $\y$ satisfies $\y = \Phat \a + \x + \b$ where $\Phat$ is a tall matrix (given), $\a$ is the vector of (unknown) regression coefficients, $\x$ is the (unknown) sparse outliers, $\b$ is (unknown) small noise/modeling error. An obvious way to solve this is by solving $\min_{\a,\x} \lambda \|\x\|_1+ \|\y - \Phat \a - \x\|^2$. In this, one can solve for $\a$ in closed form to get $\hat\a = \Phat'(\y-\x)$. Substituting this, the minimization simplifies to $\min_{\x} \lambda \|\x\|_1+ \|(\I - \Phat \Phat') (\y - \x)\|^2$. This is equivalent to the Lagrangian version of the projected CS problem that NORST solves (given in line 7 of  Algorithm \ref{algo:auto-reprocs-pca}). 
}} in order to estimate the sparse outliers, $\xt$'s, and hence the $\lt$'s, and (b) {\em Subspace Update} to update the estimates $\Phat_{(t)}$.
Projected CS proceeds as follows. At time $t$, if the previous subspace estimate, $\Phat_{(t-1)}$, is accurate enough, because of slow subspace change, projecting $\yt$ onto its orthogonal complement will nullify most of $\lt$. We compute $\tty_t:= \bpsi  \yt$ where $\bpsi := \I - \Phat_{(t-1)}\Phat_{(t-1)}{}'$. Clearly $\tty_t = \bpsi \xt + \bpsi (\lt+\vt)$ and $\| \bpsi  (\lt+\vt)\|$ is small due to slow subspace change and small $\vt$. Recovering $\xt$ from $\tty_t$ is now a CS / sparse recovery problem in small noise \cite{candes_rip}. 
We compute $\xhat_{t,cs}$ using noisy $l_1$ minimization followed by thresholding based support estimation to obtain $\That_t$. A Least Squares (LS) based debiasing step on $\That_t$ returns the final $\xhat_t$. We then estimate $\lt$ as $\lhat_t = \yt - \xhatt$.

The $\lhatt$'s are then used for the Subspace Update step 
%
which toggles between the ``detect'' phase and the ``update'' phase. It starts in the ``update'' phase with $\that_0 = t_\train$. We then perform $K$ $r$-SVD steps with the $k$-th one done at $t = \that_0 + k \alpha-1$. Each such step uses the last $\alpha$ estimates, i.e., uses $\Lhat_{t; \alpha}$. Thus at $t = \that_0 + K \alpha - 1$, the subspace update of $\P_0$ is complete. At this point, the algorithm enters the ``detect'' phase.
For any $j$, if the $j$-th subspace change is detected at time $t$, we set  $\that_j=t$. At this time, the algorithm enters the ``update'' (subspace update) phase. We then perform $K$ $r$-SVD steps with  the $k$-th $r$-SVD step done at $t = \that_j + k \alpha-1$ on $\Lhat_{t;\alpha}$. Thus, at $t = \that_{j,fin} = \that_j + K\alpha - 1$, the update is complete. At this $t$, the algorithm enters the ``detect'' phase.

To understand the change detection strategy, consider the $j$-th subspace change. Assume that the previous subspace $\P_{j-1}$ has been accurately estimated by $t= \that_{j-1,fin} = \that_{j-1}+K\alpha-1$ and that $\that_{j-1,fin} < t_j$. Let $\Phat_{j-1}$ denote this estimate. At this time, the algorithm enters the ``detect'' phase in order to detect the next ($j$-th) change. Let $\B_t:=(\I-\Phat_{j-1}\Phat_{j-1}{}')  \Lhat_{t;\alpha}$. At every $t = \that_{j-1,fin} + u \alpha-1$, $u=1,2,\dots$, we detect change by checking if the maximum singular value of $\B_t$ is above a pre-set threshold, $ \sqrt{\lthres \alpha}$, or not.
We claim that, with high probability (whp), under assumptions of Theorem \ref{thm1}, this strategy has no ``false subspace detections'' and correctly detects change within a delay of at most $2\alpha$ samples. The former is true because, for any $t$ for which $[t-\alpha+1, t] \subseteq [ \that_{j-1,fin}, t_j)$, all singular values of the matrix $\B_t$ will be close to zero (will be of order $\zz \sqrt{\lambda^+}$) and hence its maximum singular value will be below $ \sqrt{\lthres \alpha}$. Thus, whp, $\that_j \ge t_j$. To understand why the change {\em is} correctly detected within $2 \alpha$ samples, first consider $t =\that_{j-1,fin} + \lceil \frac{t_j - \that_{j-1,fin}}{\alpha} \rceil \alpha:= t_{j,*} $. Since we assumed that $\that_{j-1,fin}< t_j$ (the previous subspace update is complete before the next change), $t_j$ lies in the interval $[t_{j,*}-\alpha+1, t_{j,*}]$. Thus, not all of the $\lt$'s in this interval satisfy $\lt = \P_j \at$. Depending on where in the interval $t_j$ lies, the algorithm may or may not detect the change at this time.  However, in the {\em next} interval, i.e., for $t \in [t_{j,*}+1, t_{j,*} + \alpha]$,  all of the $\lt$'s satisfy $\lt = \P_j \at$. We can prove that, whp, $\B_t$ for this time $t$ {\em will} have maximum singular value that is above the threshold.
Thus, if the change is not detected at $t_{j,*}$, whp, it {\em will} get detected at $t_{j,*} + \alpha$. Hence, whp, either $\that_j = t_{j,*}$, or $\that_j =t_{j,*} + \alpha$, i.e., $t_j \le \that_j \le t_j + 2\alpha$.

\subsubsection{Algorithm parameters}
Algorithm \ref{norst} assumes knowledge of 4 model parameters: $r$, $\lambda^+$, $\lambda^-$ and $\xmint$ to set the algorithm parameters.
The initial dataset used for estimating $\Phat_0$ (using AltProj) can be used to get an accurate estimate of $r$, $\lambda^-$ and $\lambda^+$ using standard techniques. Thus one really only needs to set $\xmint$. If continuity over time is assumed, we can let it be time-varying and set it as $\min_{i \in \That_{t-1}}|(\xhat_{t-1})_i|$ at $t$.

\subsubsection{Time complexity} The time complexity is $O(n d r\log(1/\epsilon))$. We explain this in Appendix \ref{timecomp_derive}.%

\section{Related Work} \label{rel_work}

We first briefly discuss related work on PCA and then discuss robust PCA and subspace tracking papers.
While there has been a large amount of work in the last decade on finite-sample guarantees for PCA \cite{nadler,normal_approx} and related problems, such as sparse PCA \cite{vince_vu_annals,sparsepca_cai} and kernel PCA \cite{kpca_first, kernelpca_nips} most of these assume either the spiked covariance model (noise is modeled as being isotropic) \cite{nadler, kernelpca_nips} or that the observed data $\yt$ is i.i.d. \cite{nadler,normal_approx} or consider noiseless settings \cite{vince_vu_annals, sparsepca_cai} (typical in sparse PCA). The setting that we study involves linearly data dependent noise $\wt = \M_t \lt$ with the dependency matrix $\M_t$ being time-varying. Thus, the noise is clearly not isotropic. Moreover, this also means  that the observed data $\yt = \lt + \wt + \vp_t$ cannot be identically distributed over time.  In fact, our guarantee is interesting only in the setting where $\M_t$ changes enough over time so that the time-averaged expected value of signal-noise correlation and of noise power is sufficiently smaller than their respective instantaneous values. 

We should mention also that, while many sophisticated eigenvector perturbation bounds exist in the literature \cite{svd_topic_est, rel_perturb, rel_perturb2}, these are designed for different settings than the one we are interested in. For our setting, only the classical Davis-Kahan sin theta theorem \cite{davis_kahan} applies. In our analysis, we need to bound the sine of the largest principal angle between the true and estimated subspaces, because this helps us get a bound on the ``noise''/error seen by the projected compressed sensing step at the next time instant. Thus, \cite{svd_topic_est}, which only provides coordinate-wise bounds, cannot be used. The perturbation seen by our sample covariance matrix is additive and our observed data $\yt$ is not identically distributed, and thus the results of \cite{rel_perturb,rel_perturb2} do not apply either.


The robust PCA (RPCA) problem has been extensively studied since the first two papers by Candes et al and Chandrasekharan at al \cite{rpca,rpca2} and follow-up work by Hsu et al \cite{rpca_zhang} all of which studied a convex optimization solution, called Principal Components Pursuit or PCP. A faster non-convex solution, called Alternating Projections or AltProj, was introduced in \cite{robpca_nonconvex}. Later work has studied a projected gradient descent based approach, RPCA-GD \cite{rpca_gd}. The problem of RPCA with partial support knowledge was studied in \cite{zhan_pcp_jp}.
All RPCA guarantees assume $\mu$-incoherence of left and right singular vectors of $\L$ (needed to ensure that $\L$ is not sparse). One way to ensure that $\X$ is not low rank is to assume that an entry of $\X$ is nonzero with probability $\rho$ independent of all others (Bernoulli model) and to assume a bound on $\rho$. This was assumed in \cite{rpca}. This can sometimes be a strong assumption, e.g., in the video setting, it requires that foreground objects are one pixel wide and jump around completely randomly over time. But, if it holds, and if another stronger left-right incoherence assumption holds\footnote{$\max_{i,j} |\U \V')_{i,j}| \le \sqrt{ \frac{\mu r }{nd} }$ where $\U,\V$ are the matrices of left and right singular vectors of $\L$}, then $\rho \in O(1)$ (linear sparsity) can be tolerated while also allowing the rank of $\L$, $\rmat$ to be grow nearly linearly with $\min(n,d)$ \cite{rpca}.
The other approach to ensure that $\X$ is low rank is to assume a bound of $O(1/\rmat)$ on the maximum fraction of nonzeros (outliers) in any row or in any column of $\X$. This is assumed in most of the later works \cite{rpca2,rpca_zhang,robpca_nonconvex,rpca_gd,rmc_gd}. 

Our work provides a fast mini-batch solution to the related problem of robust subspace tracking (RPCA with explicitly assuming slowly changing subspaces). Because we replace right incoherence by its statistical version, we are able to obtain guarantees on detection and tracking delay of our approach and show that both are nearly optimal (are within log factors of the minimum required delay $r$). 
This also means that the memory complexity of NORST is also near optimal: we only need to store $\alpha$ $n$-length vectors in memory with $\alpha = C r \log n$. Of course, any RPCA approach could also be applied in a mini-batch fashion on $\alpha$-consecutive column sub-matrices, and then it will also have the same memory complexity. We assume this here in our discussion. With this assumption, $\outfracrow$ gets replaced by $\outfracrow(\alpha)$ and $\rmat$ gets replaced by $r$ for the RPCA guarantees as well.

Because we assume a lower bound on the minimum outlier magnitudes that is proportional to $\SE_j$, we obtain the following improvement in outlier tolerance (explained in Sec. \ref{tradeoffs}).
Treating $f$ as a constant, {\em for any mini-batch after $t_\train$}, we only need $\outfracrow(\alpha) \in O(1)$. For standard RPCA, unless a random model on outlier support is assumed, $\max(\outfraccol,\outfracrow(\alpha)) \in O(1/r)$ is needed \cite{robpca_nonconvex}. For the video application, this implies that NORST tolerates slow moving and occasionally static foreground objects much better than standard RPCA methods that do not assume slow subspace change.  This is also corroborated by our experiments on real videos, e.g., see Fig \ref{fig:mr_full} in Sec. \ref{sec:sims} and also see a more detailed and quantitative evaluation on real data provided in \cite{rrpcp_review}.
Since our algorithm needs to be initialized with a standard batch RPCA approach such as AltProj \cite{robpca_nonconvex} applied to the first $t_\train = C r$ data points, for this initial short batch, we do need AltProj assumptions to hold and this is why we need  $\outfracrow_{\init} \le \frac{c_3}{r}$.
%
%
For the per column fraction, we also need $\outfraccol \in O(1/r)$. Thus, the overall fraction of outliers allowed in a given matrix is still $O(1/r)$, which is the same as standard RPCA, but these can be less spread out row-wise (some rows could have many more outliers than others).%
\cbl

Moreover, we are able to guarantee that each column of $\L$, $\lt$, is recovered to $\zz$ relative accuracy and that the support of outliers can be recovered exactly. Neither is guaranteed by existing RPCA results, these only guarantee $\|\Lhat - \L\|_F \le \zz$.

Finally, in terms of time complexity, the NORST complexity of $O(n \alpha r\log(1/\zz))$ per mini-batch is comparable to that of simple (non-robust) PCA. In comparison to RPCA solutions, this is much faster than PCP \cite{rpca,rpca2,rpca_zhang} which needs $O(n \alpha^2 \frac{1}{\zz})$ and $r$-times faster than AltProj \cite{robpca_nonconvex} which needs $O(n \alpha r^2 \log(1/\zz))$. RPCA-GD \cite{rpca_gd} is as fast as NORST but requires an even tighter outlier fractions' bound than other RPCA solutions: $\max(\outfracrow,\outfraccol) \in O(1/ r^{3/2})$. 



Our work builds upon the simple-ReProCS (s-ReProCS) solution and guarantee \cite{rrpcp_dynrpca} and removes many of its limitations. 
S-ReProCS assumes a specific model of slow subspace change:  {\em only one} subspace direction can change at each change time, and the amount of change needs to be bounded. Even with this assumption, its tracking delay is of order $r \log n \log (1/\zz)$. Since only one direction is changing, this delay is $r$-times sub-optimal. The same is true for its required lower bound on subspace change times. 
A second limitation of s-ReProCS is that, in order to track subspaces to $\zz$ accuracy, it requires the initial subspace estimate to also be $\zz$ accurate. This, in turn, implies that one needs to run the AltProj or PCP algorithm on the initial mini-batch to convergence. Instead, our approach only requires the initial subspace error to be $O(1/\sqrt{r})$. Thus, only order $\log r$ iterations of AltProj suffice to initialize our algorithm.
Thirdly, the s-ReProCS guarantee needs a stronger statistical right incoherence assumption than ours: it needs an entry-wise bound of $\max_t \max_{i=1,2,\dots, r} |(\at)_i|^2 \le {\eta \lambda^+}$. Lastly, we develop important extensions of our main result for (i) only tracking subspace changes (without detecting the change), and (ii) for subspaces changing by a little at each time $t$.


An earlier version of Theorem \ref{thm1} appeared in ICML 2018 \cite{rrpcp_icml}. The results of the current manuscript improve upon the ICML result in various ways: we need a weaker statistical right incoherence assumption, a weaker lower bound on $\SE_j$, and we develop two important extensions of our main result for subspace changes at each time and for applications not requiring change detection. 
Moreover, \cite{rrpcp_icml} did not prove the result for PCA in data-dependent noise, but only used the result proved in our older ISIT paper \cite{pca_dd_isit}.
The problem of ST with missing data is a special case of robust ST, while ST with missing data and outliers is a simple generalization of robust ST. Interesting guarantees for both of these follow as easy corollaries of either our current result or of its earlier version from  \cite{rrpcp_icml}. A corollary of the result of  \cite{rrpcp_icml} for ST-miss is presented in \cite{rrpcp_tsp19}. 
In comparison to the result of  \cite{rrpcp_tsp19}, a similarly derived ST-miss corollary of our current result has all the advantages mentioned earlier in this paragraph.

\section{Extensions: subspace change at each time, subspace tracking without detection} \label{extends}

\newcommand{\tL}{\tilde\L}
\subsection{Subspace changing at each time}\label{change_every_t}
Suppose $\yt = \tilde\l_t + \xt$ where  $\tilde\l_t = \P_{(t)} \tilde\a_t$, $\Pt$ changes by a little at each time $t$, but has more significant changes at certain times $t_j$. 
We show here how this case can be handled by treating the error generated by changes at each time $t$ as extra unstructured noise $\vt$.
%
Assume that $\tilde\at$'s are zero mean, bounded, and i.i.d. with diagonal covariance matrix $\tilde\Lam$. Let $\tilde\lambda^+$ be its maximum eigenvalue and $\tilde{f}$ the condition number.
%
Define $\P_j$ as the matrix of top $r$ left singular vectors of the matrix $\tL_j:= [\tilde\l_{t_j}, \tilde\l_{t_j+1}, \dots, \tilde\l_{t_{j+1}-1}]$, or equivalently of $[\P_{(t_j)}, \P_{(t_j+1)}, \dots, \P_{(t_{j+1}-1)}]$.
Let $\at : = \P_j' \tilde\l_t$, $\lt : = \P_j \at$ and $\vt:= \tilde\l_t - \lt = \P_{j,\perp} \tilde\l_t$.%

Another way to understand the above is that we are expressing $\tL_j = \L_j + \V_j$ where $\L_j$ is the rank-$r$ SVD of $\tL$, while $\V_j$ is the rest. While  $\L_j \V_j'= 0$, we cannot say anything about individual vectors $\lt \vt'$ or their expected value. In general, $ \E[\lt \vt{}'] \neq 0$. But even then, we can always use Cauchy-Schwarz to get the bound $\|\E[\lt \vt']| \| \le \sqrt{\lambda^+ \lambda_v^+}$.
Thus, to analyze this case, we need to modify Corollary \ref{cor:noisy_pca_sddn} for PCA-SDDN as follows: we now need $4 \sqrt{b} q f  + \frac{\lambda_v^+}{\lambda^-}+ \sqrt{\frac{\lambda_v^+}{\lambda^-} f} < 0.4 \epsilon_\SE$.  
There is no change to the required lower bound on $\alpha$.
From our definition of $\vt$, $\lambda_v^+ \le \SE(\P_j, \Pt)^2 \tilde\lambda^+$. Using $\lambda^+ \le \tilde\lambda^+$, $\tilde\lambda^- < \lambda^-$, a simple sufficient condition to ensure that the third term is small ($\lambda_v^+/\lambda^- \le 0.01 \zz^2/f$) is $ \SE(\P_j, \Pt)^2 \le 0.01 \zz^2 /\tilde{f}^2$.
\begin{corollary}[Subspace changing at each $t$]
Consider the setting defined in the first paragraph above.
If $ \SE(\P_j, \Pt)^2 < 0.01 \zz^2 /\tilde{f}^2$, Theorem \ref{thm1} applies with $\P_j$, $\lt$, and $\vt$ as defined above.
\end{corollary}

\subsection{NORST-NoDet: NORST without subspace change detection}\label{norst_nodet}
A simpler version of the NORST algorithm that does not detect change is as follows. The robust regression (projected CS) step is exactly as explained earlier. The subspace update step is much simpler:  it just updates $\Phat_{(t)}$ as the top $r$ left singular vectors of $\Lhat_{t; \alpha}$ once every $\alpha$ frames. We refer to it as NORST-NoDet. We have the following guarantee for it.

\begin{theorem}
Consider Algorithm \ref{algo:norst_nodet} with parameters set as $\alpha = C f^2 \mu r \log n $,  $\zeta = \xmint/15$ and $\omega_{supp} = \xmint/2$.
Assume everything stated in Theorem \ref{thm1} except the lower bound on $\SE_j$.
Then, w.p. at least $1 - 10 \tmax n^{-10} $,
\[
\SE(\Phat_{(t)}, \P_{(t)}) \le
 \left\{
\begin{array}{ll}
\min(4 f  \SE_j, 1) & \text{ if }  t \in \J_1, \\
 (0.3)^{k-1}  \min( 4 f  \SE_j, 1 ) & \text{ if }  t \in \J_k, \\
\zz:=  c \sqrt{\lambda_v^+/\lambda^-}    & \text{ if }  t \in \J_K,
\end{array}
\right.
\]

where $\J_1 = [\lfloor t_j/\alpha \rfloor \alpha ,  (\lfloor t_j/\alpha  \rfloor + 1)\alpha)$, $J_k = [\lfloor t_j/\alpha  \rfloor +1 )\alpha)+(k-1)\alpha, ( \lfloor t_j/\alpha)+ (k+1)  ) \alpha  \rfloor $ for $k =2, 3 , \cdots, K-1$ and $\J_K = [ (\lfloor t_j/\alpha  \rfloor +(K+1) ) \alpha, \lfloor t_{j+1}/\alpha \rfloor \alpha )$. \\
%
The time complexity is $O(n \tmax r \log (1/\zz) )$ and memory complexity is $O(n \alpha) = O(f^2 n r \log n \log (1/\zz) )$.
\label{thm1_newnorst}
\end{theorem}

The advantage of NORST-NoDet is that it does not require a lower bound on the amount of change, $\SE_j$, and it needs fewer algorithm parameters (does not need $K$ or $\omega_{evals}$).  The disadvantage is it does not detect subspace change, we cannot obtain a ``smoothing'' version of it that solves the dynamic RPCA problem to $\zz$ accuracy at all times, and  its subspace error bound is larger for the intervals during which the subspace changes, $ [\lfloor t_j/\alpha \rfloor \alpha ,  (\lfloor t_j/\alpha  \rfloor + 1)\alpha)$.
For times $t$ in this interval, the  bound is $\min(4 f  \SE_j, 1)$. Assuming small enough $\zz$, this is larger than  $(\zz + \SE_j)$ which is the NORST bound for this interval. The reason is  NORST stops tracking after the current subspace has been estimated accurately enough and until the next change is detected. During this period, it uses $\Phat_{j-1}$ as the estimate. But NORST-NoDet updates the subspace in every interval. For the change interval, the rank of $\L_{t;\alpha}$ is more than $r$. It can be $2r$ in general. This is why it is not possible to guarantee a better bound for the $r$-SVD estimate. At the same time, without extra assumptions, it is not possible to obtain a guarantee for  $2r$-SVD estimate either. 

For analyzing the change interval we use the following modification of PCA-SDDN. Its proof is in Appendix \ref{proofs_pca_section}. The proof of Theorem \ref{thm1_newnorst} is given in Appendix \ref{sec:proof}.

\begin{corollary}\label{pca_ssch}
Assume that $\yt = \lt + \wt + \vp_t$ with $\wt = \M_{2,t} \M_{1,t} \lt$, with $\lt = \P_0 \at$ for $t \in [1, \alpha_0]$ and $\lt = \P \at$ for $t \in [\alpha_0+1, \alpha]$, and $\SE(\P_0, \P) \le \Delta$.
Assume also that Assumptions \ref{def_right_incoh}, \ref{vt_assu} hold,  $\max_t \max(\|\M_{1,t} \P_0\|, \|\M_{1,t} \P\|) \le q < 1$, and the fraction of nonzeros in any row of the noise matrix $[\w_1, \w_2, \dots, \w_\tmaxpca]$ is equal to $\bz$.
Let $g: = \frac{\lambda_{v}^+}{\lambda^-}$. If $\Delta < c / f$,
and if
$\tmaxpca \ge \tmaxpca^*  = C \max\left( \frac{q^2 f^2}{\epsilon_\SE^2} r \log n, \frac{g f}{\epsilon_\SE^2} \max(r_v,r) \log n\right)$
then w.p. at least $1- 10n^{-10}$,
\begin{align*}
\SE(\Phat,\P) &\le 1.1 \left(3 ((\alpha_0 / \alpha) \Delta  + 4 \sqrt{\bz} q) f + \frac{\lambda_{v}^+ }{ \lambda^-} \right) \\
&\le 3.3 \Delta f  + 4.4 \sqrt{\bz} q f + 1.1 \frac{\lambda_{v}^+ }{ \lambda^-} .
\end{align*}
\end{corollary}

\begin{algorithm}[t!]
\caption{\small{NORST-NoDet}} 
\label{algo:norst_nodet}
\begin{algorithmic}[1]
\State \textbf{Input}:  $\Phat_0$, $\yt$; \  \textbf{Output}:  $\shatt$, $\lhatt$,  $\Phat_{(t)}$; \ \textbf{Parameters:} $\omega_{supp}$, $\xi$, $\alpha$
\State $\Phat_{(t_\train)} \leftarrow \hat{\pt}_{0}$;
\For {$t > t_\train$}
\State Lines 6-11 of Algorithm \ref{norst}
\If {$t = t_{\train} + u \alpha - 1$ for $u = 1,\ 2,\ \cdots,$}
\State $\Phat_u \leftarrow \SVD_r[\Lhat_{t; \alpha}]$, $\Phat_{(t)} \leftarrow \Phat_{u}$
\Else
\State $\Phat_{(t)} \leftarrow \Phat_{(t-1)}$
\EndIf
\EndFor
\end{algorithmic}
\end{algorithm}

\section{Proof  of correctness of the NORST algorithm} \label{sec:proof_norst}
In this section we state the three main lemmas and explain how they help prove Theorem \ref{thm1}. After this, we prove the three lemmas.

\subsection{Main Lemmas}



We define or recall a few things first.
\begin{enumerate}


\item Recall $\Delta = \max_j \SE(\P_{j-1}, \P_j)$, let $\Delta_0 = \SE(\Phat_0, \P_0)$; recall $c \sqrt{\lambda_v^+/\lambda^-} < \zz \leq 0.01 < 0.2$

\item Let $\Phat_{j,0} = \Phat_{j-1}$ and recall (from Algorithm) that $\Phat_{j-1} = \Phat_{j-1,K}$: 

\item Constants for Theorem \ref{thm1}: $c_1=c_2 = 0.001$ (bounds on $\outfraccol, \outfracrow(\alpha)$), and $c_3= 1/(30 \sqrt{\mu})$. 
 We use $b_0 = c_2/f^2$ to denote the bound on $\outfracrow(\alpha)$.

\item  Let $q_0 := 1.2(\zz + \SE_j)$, $q_{k} = 1.2\max( q_{k-1}/4, \zz) $. Clearly $q_k = \max(0.3^k q_0, 1.2 \zz)$.



\end{enumerate}
First consider the simpler case when $t_j$'s are known. In this case $\that_j = t_j$. Define the events
\bi
\item  $\Gamma_{0,0}:= \{\text{assumed bound on } \SE(\Phat_0,\P_0)\}$,
\item $\Gamma_{0, k}:= \Gamma_{0,k-1} \cap \{\SE(\Phat_{0,k},\P_0) \le \SE(\Phat_0,\P_0) \}$, 
\item  $\Gamma_{j, 0} := \Gamma_{j-1, K}$,
$\Gamma_{j, k} := \Gamma_{j, k-1} \cap \{\SE(\Phat_{j, k}, \P_j) \le q_{k-1}/4 \}$ for $j=1,2,\dots, J$ and $k=1,2,\dots, K$.

\item Using the expression for $K$ given in the theorem, and since $\Phat_j = \Phat_{j,k}$ (from the Algorithm), it follows that $\Gamma_{j,K}$ implies $\SE(\Phat_j, \P_j)=\SE(\Phat_{j, K}, \P_j) \leq \zz$.
 \ei
Observe that, if we can show that $\Pr(\Gamma_{J,K} | \Gamma_{0, 0}) \geq 1 - dn^{-10}$ we will have obtained all the subspace recovery bounds of Theorem \ref{thm1}. The next two lemmas, Lemmas \ref{lem:reprocspcalemone} and \ref{lem:reprocspcalemk}, applied sequentially help show that this is true. The first one proves that $\Pr (\Gamma_{j,1} | \Gamma_{j,0}) \ge 1 - 10 n^{-10}$, the second one proves that $\Pr (\Gamma_{j,k} | \Gamma_{j,k-1}) \ge 1 - 10 n^{-10}$ for $k=1,2,\dots, K$.
The bounds on  $\| \lt -\lhat_t\|$ follow easily. 

To prove the actual result with $t_j$ unknown, we also need  Corollary \ref{cor:etbnds} and Lemma \ref{lem:sschangedet}  which proves that the change detection step works as desired. Moreover, we will need a different definition of $\Gamma_{j,0}$; we cannot set it equal to $\Gamma_{j-1,K}$. The proof is given in Appendix \ref{proof_of_thm1}. 

\begin{lem}[first update interval]\label{lem:reprocspcalemone}
Under the conditions of Theorem \ref{thm1}, conditioned on $\Gamma_{j, 0}$,
\begin{enumerate}
\item for all $t \in [\hat{t}_j, \hat{t}_j + \alpha)$,
$\norm{\bm\Psi (\lt + \vt)} \le (\zz +  \semax) \sqrt{\mu r \lambda^+} + \sqrt{r_v \lambda_v^+} < \xmint/15$,
$\norm{\xhat_{t,cs} - \xt} \le 7 \xmint/15 < \xmint/2$,  $\That_t = \T_t$,
the error $\et := \xhatt - \xt$ satisfies
\begin{align}\label{eq:etdef}
\et = \itt \left( \bpsi_{\Tt}{}'\bpsi_{\Tt} \right)^{-1} \itt{}' \bpsi (\lt + \vt) 
\end{align}
and $\norm{\et} \leq 1.2 [(\zz  + \semax) \sqrt{\mu r \lambda^+} + \sqrt{r_v \lambda_v^+}]$.
Here $\bpsi = \I - \Phat_{j,0} \Phat_{j,0}{}'$. Recall we let $\Phat_{j,0} = \Phat_{j-1}$.

\item w.p. at least $1 - 10n^{-10}$,  $\Phat_{j,1}$ satisfies $\SE(\Phat_{j, 1}, \P_j) \leq \max(q_{0}/4, \zz)$, i.e., $\Gamma_{j, 1}$ holds.

\end{enumerate}
\end{lem}

\begin{lem}[$k$-th update interval]\label{lem:reprocspcalemk}
Under the conditions of Theorem \ref{thm1}, conditioned on $\Gamma_{j, k-1}$,
\begin{enumerate}

\item for all $t \in [\hat{t}_j + (k-1)\alpha, \hat{t}_j + k\alpha - 1)$,  all claims of the first part of Lemma \ref{lem:reprocspcalemone} holds,
$\norm{\bm\Psi (\lt + \vt)} \le \max(0.3^{k-1} (\zz +  \semax), \zz) \sqrt{\mu r \lambda^+} + \sqrt{r_v \lambda_v^+} $,
$\et$ satisfies \eqref{eq:etdef},  and
$\norm{\et} \leq \max((0.3)^{k-1} \cdot 1.2 (\zz+  \semax), \zz) \sqrt{\mu r \lambda^+} + \sqrt{r_v\lambda_v^+}$. 
Here $\bpsi = \I - \Phat_{j,k-1} \Phat_{j,k-1}{}'$.

\item w.p. at least $1 - 10n^{-10}$, $\Phat_{j, k}$ satisfies $\SE(\Phat_{j, k}, \P_j) \leq \max(q_{k-1}/4, \zz)$, i.e., $\Gamma_{j, k}$ holds.

\end{enumerate}
\end{lem}

\begin{remark}
For the case of $j=0$, in both the lemmas above, $\semax$ gets replaced by $\SE(\Phat_0,\P_0)$.
\end{remark}


\begin{corollary}\label{cor:etbnds}
Under the conditions of Theorem \ref{thm1}, the following also hold.
\begin{enumerate}
\item For all $t \in[t_j, \that_j)$, conditioned on $\Gamma_{j-1, K}$, all claims of the first item of Lemma \ref{lem:reprocspcalemone} hold. 

\item For all $t \in[\that_j + K\alpha, t_{j+1})$, conditioned on $\Gamma_{j, K}$, the first item of Lemma \ref{lem:reprocspcalemk} holds with $k=K$. 
\end{enumerate}
Thus,  for all times $t$,  under appropriate conditioning, $\et$ satisfies \eqref{eq:etdef}.
\end{corollary}


The following lemma shows that, whp, we can detect subspace change within $2\alpha$ time instants without any false detections. Recall that the detection threshold $\lthres = 2\zz^2 \lambda^+$.

\begin{lem}[Subspace Change Detection]\label{lem:sschangedet}
Assume that the conditions of Theorem \ref{thm1} hold.
\ben
\item Consider an $\alpha$-length time interval $\J^{\alpha} \subset [t_j, t_{j+1}]$  during which $\Phat_{(t-1)} = \Phat_{j-1}$ so that $\bpsi = \I - \Phat_{j-1} \Phat_{j-1} {}'$. Let  $\bphi = \bpsi$.
Assume that $\SE(\Phat_{j-1}, \P_{j-1}) \leq \zz$  and $\et$ satisfies  \eqref{eq:etdef}.
Then, w.p. at least $1 - 10n^{-10}$,
\begin{align*}
\hspace{-1cm}\lambda_{\max}\left( \frac{1}{\alpha}\sum_{t \in \J^{\alpha}} \bphi \lhatt \lhatt{}' \bphi \right) & \geq 0.59 \lambda^- \SE_j (\SE_j - 8 \zz) > \lthres
\end{align*}
since $\SE_j > 9 \sqrt{f} \zz$.

\item 
Consider an $\alpha$-length time interval $\J^{\alpha} \subset [t_j, t_{j+1}]$ during which $\Phat_{(t-1)} = \Phat_{j}$ so that $\bpsi = \I - \Phat_{j} \Phat_{j} {}'$. Let  $\bphi = \bpsi$. Assume that $\SE(\Phat_{j}, \P_{j}) \leq \zz$  and $\et$ satisfies  \eqref{eq:etdef}.
Then, w.p. at least $1 - 10n^{-10}$,
\begin{align*}
\lambda_{\max}\left( \frac{1}{\alpha}\sum_{t \in \J^{\alpha}} \bphi \lhatt \lhatt{}' \bphi \right) &\leq 1.37 \zz^2 \lambda^+ < \lthres
\end{align*}

\een
\end{lem}

\subsection{Proof of the first two lemmas}\label{lemmas_proof}

The projected CS proof (item one of the first two lemmas) uses the following lemma from \cite{rrpcp_perf} that relates the $s$-Restricted Isometry Constant (RIC), $\delta_s(.)$ \cite{candes_rip} of $\I - \P \P'$ to incoherence of $\P$.
\begin{lem}\label{kappadelta} [{\cite{rrpcp_perf}}]
For an $n \times r$ basis matrix $\P$,
(1) $\delta_s(\I - \P \P')  = \max_{|\T| \le s} \|\I_\T{}' \P\|^2$; and
(2) $ \max_{|\T| \le s} \|\I_\T{}' \P\|^2 \le s \max_{i=1,2,\dots,n} \|\I_i{}' \P\|^2 \le s \mu r / n$.
\end{lem}
The last bound of the above lemma used the definition of the incoherence parameter $\mu$. We will apply this lemma with $s =  \outfraccol \cdot n$. The subspace update step proof (item 2 of the first two lemmas) uses Corollary \ref{cor:noisy_pca_sddn} for PCA-SDDN and the following simple lemma proved in the Supplement given in the Arxiv version of this work \cite[Appendix V]{jsait_arxiv}).
%
\begin{lem}\label{lem:sumprinang}
Let $\Aa$, $\Ba$ and $\Ca$ be $r$-dimensional subspaces in $\Re^n$ such that $\SE(\Aa, \Ba) = \Delta_1$ and $\SE(\Ba, \Ca) = \Delta_2$. Then, $\Delta_1 - 2 \Delta_2 \leq \SE(\Aa, \Ca) \leq \Delta_1 + \Delta_2$.
\end{lem}

\begin{proof}[Proof of Lemma \ref{lem:reprocspcalemone}]
 {\em Proof of item 1. }
First consider $j>0$.
We have conditioned on the event $\Gamma_{j,0}:= \Gamma_{j-1,K}$. This implies that $\SE(\Phat_{j-1}, \P_{j-1}) \leq \zz$.

For the interval $t \in [\that_j, \that_j + \alpha)$, $\Phat_{(t-1)} = \Phat_{j-1}$ and thus $\bpsi = \I - \Phat_{j-1} \Phat_{j-1}{}'$ (from Algorithm).
Let $s:= \outfraccol \cdot n$.
For the sparse recovery step, we need to bound the $2s$-RIC of $\bpsi$. To do this, we obtain bound on $\max_{|\T| \leq 2s} \|\I_{\T}{}' \Phat_{j-1}\|$ as follows. Consider any set $\T$ such that $|\T| \leq 2s$. Then,
\begin{align*}
\norm{\I_{\T}{}' \Phat_{j-1}} &\leq \norm{\I_{\T}{}'(\I - \P_{j-1} \P_{j-1}{}') \Phat_{j-1}} + \norm{\I_{\T}{}' \P_{j-1} \P_{j-1}{}' \Phat_{j-1}} \\
&\leq \SE(\P_{j-1}, \Phat_{j-1}) + \norm{\I_{\T}{}' \P_{j-1}} \\
&= \SE(\Phat_{j-1}, \P_{j-1}) + \norm{\I_{\T}{}' \P_{j-1}}
\end{align*}
Using Lemma \ref{kappadelta}, and the bound on $\outfraccol$ from Theorem \ref{thm1},
\begin{align}\label{eq:dense}
\max_{|\T| \leq 2s} \|\I_{\T}{}' \P_{j-1}\|^2 \leq 2s \max_i \|\I_i{}'\P_{j-1}\|^2 \leq \frac{2s\mu r}{n} \leq 0.01
\end{align}
Thus, using $\SE(\Phat_{j-1}, \P_{j-1}) \leq \zz$, (where $c \sqrt{\lambda_v^+/\lambda^-} \leq \zz \leq 0.01$),
\begin{align*}
\max_{|\T| \leq 2s} \norm{\I_{\T}{}' \Phat_{j-1}} \leq \zz + \max_{|\T| \leq 2s}\norm{\I_{\T}{}' \P_{j-1}} \leq \zz + 0.1
\end{align*}

Finally, using Lemma \ref{kappadelta}, $\delta_{2s}(\bpsi) \leq 0.11^2 < 0.15$. Hence
\begin{align*}
\norm{\left(\bpsi_{\Tt}{}'\bpsi_{\Tt}\right)^{-1}} \leq \frac{1}{1 - \delta_s(\bpsi)} \leq \frac{1}{1 - \delta_{2s}(\bpsi)} \leq \frac{1}{1- 0.15} < 1.2.
\end{align*}
When $j=0$, there are some minor changes. From the initialization assumption, we have $\SE(\Phat_0, \P_0) \leq 0.25$.  Thus, $ \max_{|\T| \leq 2s} \norm{\I_{\T}{}' \Phat_{0}} \leq 0.25 + 0.1 = 0.35$. Thus, using Lemma \ref{kappadelta}, $\delta_{2s}(\bpsi_0) \leq 0.35^2 < 0.15$. The rest of the proof given below is the same for $j=0$ and $j>0$.

Next we bound norm of $\b_t:=\bpsi (\lt + \vt)$. 
\begin{align*}
\norm{\b_t} = \norm{\bpsi (\lt + \vt)} &\leq \norm{(\I - \Phat_{j-1} \Phat_{j-1}{}') \P_j \at} + \|\vt\| \leq \SE(\Phat_{j-1}, \P_j) \norm{\at} + \sqrt{r_v \lambda_v^+} \\
&\overset{(a)}{\leq} (\zz + \SE(\P_{j-1}, \P_j))\sqrt{\mu r \lambda^+} + \sqrt{r_v \lambda_v^+} 
\end{align*}
where $(a)$ follows from Lemma \ref{lem:sumprinang} with $\Aa = \Phat_{j-1}$, $\Ba = \P_{j-1}$ and $\Ca = \P_j$. Under the assumptions of Theorem \ref{thm1}, the RHS of $(a)$ is bounded by $ x_{\min} / 15$. This is why we have set $\xi = x_{\min}/15$ in the Algorithm. Using these facts, and $\delta_{2s} (\bpsi) \leq 0.15$, the CS guarantee from \cite[Theorem 1.3]{candes_rip} implies that
\begin{align*}
\norm{\xhat_{t,cs} - \xt} &\leq  7 \xi = 7x_{\min}/15 < x_{\min}/2
\end{align*}
Consider support recovery. From above,
\begin{align*}
| (\shatcs - \st)_i | \leq \norm{\shatcs - \st} \leq 7x_{\min}/15 < x_{\min}/2
\end{align*}
The Algorithm sets $\omega_{supp} = \smin/2$. Consider an index $i \in \Tt$. Since $|(\st)_i| \geq \smin$,
\begin{align*}
\smin - |(\shatcs)_i| &\le  |(\st)_i| - |(\shatcs )_i| \ \\
&\le | (\st - \shatcs )_i | < \frac{\smin}{2}
\end{align*}
Thus, $|(\shatcs)_i| > \frac{\smin}{2} = \omega_{supp}$ which means $i \in \Thatt$. Hence $\Tt \subseteq \Thatt$. Next, consider any $j \notin \Tt$. Then, $(\st)_j = 0$ and so
\begin{align*}
|(\shatcs)_j| = |(\shatcs)_j)| - |(\st)_j| &\leq |(\shatcs)_j -(\st)_j| < \frac{\smin}{2}
\end{align*}
which implies $j \notin \Thatt$ and $\Thatt \subseteq \Tt$ implying that $\Thatt = \Tt$.

With $\That_t = \Tt$ and since $\Tt$ is the support of $\xt$, $\xt = \I_{\Tt} \I_{\Tt}{}' \xt$, and so
\begin{align*}
\shatt &= \bm{I}_{\Tt}\left(\bpsi_{\Tt}{}'\bpsi_{\Tt}\right)^{-1}\bpsi_{\Tt}{}'(\bpsi \lt + \bpsi \st + \bpsi \vt) \\
&= \bm{I}_{\Tt}\left(\bpsi_{\Tt}{}'\bpsi_{\Tt}\right)^{-1} \I_{\Tt}{}' \bpsi(\lt + \vt) + \st
\end{align*}
since $\bpsi_{\Tt}{}' \bpsi = \I_{\Tt}' \bpsi' \bpsi = \I_{\Tt}{}' \bpsi$. Thus $\et = \shatt-\st$ satisfies
\begin{align*}
\et &= \bm{I}_{\Tt}\left(\bpsi_{\Tt}{}'\bpsi_{\Tt}\right)^{-1} \I_{\Tt}{}' \bpsi(\lt + \vt)   : = (\e_{\l})_t + (\e_{\v})_t ,   \\
\norm{\et} &\leq \norm{\left(\bpsi_{\Tt}{}'\bpsi_{\Tt}\right)^{-1}} \norm{\itt{}'\bpsi (\lt + \vt)} \\
&\le 1.2  \norm{\itt{}'\bpsi (\lt + \vt)}
\end{align*}

\renewcommand{\bz}{b}
\emph{Proof of Item 2}:  Recall that $q_0 :=  1.2(\zz + \SE_j)$, $q_{k} = 1.2\max( q_{k-1}/4, \zz) = \max(0.3^k q_0, 1.2 \zz)$. From our definition of $K$, $0.3^K q_0 = \zz$. Thus, for $k \le K$, $\max( q_{k-1}/4, \zz)  = q_{k-1}/4$. 

We are considering the interval $[\that_j, \that_j + \alpha)$. Since $\lhat_t = \lt - \et + \vt$ with $\et$ satisfying \eqref{eq:etdef}, updating $\Phat_{(t)}$ from the $\lhatt$'s is a problem of PCA in sparse data-dependent noise (SDDN). To analyze this, we use Corollary \ref{cor:noisy_pca_sddn}.
Define $(\e_{\l})_t = \bm{I}_{\Tt}\left(\bpsi_{\Tt}{}'\bpsi_{\Tt}\right)^{-1} \I_{\Tt}{}' \bpsi \lt$ and $(\e_{\v})_t = \bm{I}_{\Tt}\left(\bpsi_{\Tt}{}'\bpsi_{\Tt}\right)^{-1} \I_{\Tt}{}' \bpsi \vt$.
%
Recall from the Algorithm that we compute $\Phat_{j, 1}$ as the top $r$ eigenvectors of $\frac{1}{\alpha}\sum_{t = \that_j}^{\that_j + \alpha - 1} \lhat_t \lhat_t{}'$. In the notation of  Corollary \ref{cor:noisy_pca_sddn}, $\yt \equiv \lhat_t$, $\wt \equiv (\e_{\l})_t$, $\vp_t \equiv (\e_{\v})_t + \vt$, $\lt \equiv \lt$, $\M_{s,t} = -\left( \bpsi_{\T_t}{}' \bpsi_{\T_t}\right)^{-1} \bpsi_{\T_t}{}'$, $\Phat = \Phat_{j,1}$, $\P = \P_j$, and so $\norm{\M_{s,t} \P} = \| \left( \bpsi_{\T_t}{}' \bpsi_{\T_t}\right)^{-1} \bpsi_{\T_t}{}' \P_j\| \leq  1.2(\zz + \SE_j) :=  q_0$. Also,  $\lambda_v^+ \equiv 2.2 \lambda_{v}^+$ since $\E[(\e_{\v})_t (\e_{\v})_t{}'] \leq (1.2)^2 \lambda_v^+$.
And $\bz \equiv b_0$ which is the upper bound on $\outfracrow(\alpha)$.  Applying Corollary \ref{cor:noisy_pca_sddn} with $q \equiv q_0$, $\bz \equiv b_0$ and using $\varepsilon_{\text{SE}} = \max(q_0/4, \zz)$, observe that we require
\begin{align*}
4\sqrt{b_0}  q_0 f + (2.2)^2 \lambda_v^+/\lambda^- \leq 0.4 \max(q_0/4, \zz).
\end{align*}
From above, $\max(q_0/4, \zz)  = q_0/4$ (if the max is $\zz$ we stop the tracking).
%
The required bound holds since $q_0/4 \ge \zz  > c \sqrt{\lambda_v^+/\lambda^-}$ (from Theorem) and $\sqrt{b_0} = 0.01/f$. Corollary \ref{cor:noisy_pca_sddn} also requires $\tmaxpca \ge \tmaxpca_*$ which is defined in it. Our choice of $\alpha = C f^2 \mu r \log n$ satisfies this since $q_0^2/\varepsilon_\SE^2 = 4^2$ and $(\lambda_v^+/\lambda^-)/\varepsilon_\SE^2 < C$.
%
Thus, by Corollary \ref{cor:noisy_pca_sddn}, with probability at least $1 - 10n^{-10}$, $\SE(\Phat_{j, 1}, \P_j) \leq \max(q_0/4, \zz)$. 
\end{proof}

\renewcommand{\old}{\mathrm{old}}
\begin{remark}[Clarification about conditioning]
In the proof above we have used Corollary \ref{cor:noisy_pca_sddn} for $\lhat_t$'s for $t \in \J^\alpha:=[\that_j, \that_j+\alpha)$. This corollary
assumes that, for $t \in \J^\alpha$, $\at$'s are mutually independent and $\M_{s,t}$'s are  deterministic matrices. 
Let $\y_\old:=\{\y_1,\y_2, \dots, \y_{\that_j-1}\}$.
We apply Corollary \ref{cor:noisy_pca_sddn} conditioned on $\y_\old$, for a $\y_\old \in \Gamma_{j,0}$.
Conditioned on $\y_\old$, clearly, the matrices $\M_{s,t}$ used in the proof above are deterministic. Also $\y_\old$ is independent of the $\at$'s for $t \in \J^\alpha$ and thus, even conditioned on $\y_\old$, the $\at$'s for $t \in  \J^\alpha$ are mutually independent.  Corollary \ref{cor:noisy_pca_sddn} tells us that, for any $\y_\old \in \Gamma_{j,0}$, conditioned on $\y_\old$, w.p. at least $1 - 10n^{-10}$, $\SE(\Phat_{j,1}, \P_j) \leq \max(q_0/4, \zz)$. Since this holds with the same probability for all $\y_\old \in  \Gamma_{j,0}$, it also holds with the same probability when we condition on $\Gamma_{j,0}$. Thus, conditioned on $\Gamma_{j,0}$, with this probability, $\Gamma_{j,1}$ holds.
An analogous argument also applies for the next proof.%
\end{remark}

\begin{proof}[Proof of Lemma \ref{lem:reprocspcalemk}]
We first present the proof for the $k = 2$ case and then generalize it for an arbitrary $k$. 
Consider $k=2$. We have conditioned on $\Gamma_{j,1}$. This implies that $\SE(\Phat_{j,1}, \P_j) \leq q_0/4$.
We consider the interval $t \in [\that_j + \alpha,\that_j + 2\alpha)$. For this interval, $\Phat_{(t-1)} = \Phat_{j,1}$ and thus $\bpsi = \I - \Phat_{j,1} \Phat_{j,1}{}'$. Consider any set $\T$ such that $|\T| \leq 2s$. We have
\begin{align*}
\norm{\I_{\T}{}' \Phat_{j, 1}} &\leq \norm{\I_{\T}{}'(\I - \P_{j} \P_{j }{}') \Phat_{j, 1}} + \norm{\I_{\T}{}' \P_{j} \P_{j}{}' \Phat_{j, 1}} \\
&\leq \SE(\P_{j}, \Phat_{j, 1}) + \norm{\I_{\T}{}' \P_{j}} = \SE(\Phat_{j, 1}, \P_{j}) + \norm{\I_{\T}{}' \P_{j}}
\end{align*}
The equality holds since $\SE$ is symmetric for subspaces of the same dimension.
Using $\SE(\Phat_{j,1}, \P_j) \leq \max(q_0/4, \zz)$, \eqref{eq:dense},
\begin{align*}
\max_{|\T| \leq 2s} \norm{\I_{\T}{}' \Phat_{j, 1}} &\leq \max(q_0/4, \zz) + \max_{|\T| \leq 2s}\norm{\I_{\T}{}' \P_{j}} \\
& \leq \max(q_0/4, \zz) + 0.1.
\end{align*}
By the assumptions of Theorem \ref{thm1}, $q_0 \leq 0.96$ and $\zz \leq 0.2$. Using this and Lemma \ref{kappadelta},
\begin{align*}
\delta_{2s}(\bpsi_j) &= \max_{|\T| \leq 2s} \norm{\I_{\T}{}' \Phat_{j, 1}}^2 \leq 0.35^2 < 0.15 \\
&\implies \norm{\left(\bpsi_{\Tt}{}'\bpsi_{\Tt}\right)^{-1}} \leq 1.2.
\end{align*}
Finally,
\begin{align*}
\norm{\b_t} = \norm{\bpsi (\lt + \vt)} &\leq \norm{(\I - \Phat_{j, 1} \Phat_{j, 1}{}') \P_j \at} +  \|\vt\| \\
&\leq \max(q_0/ 4, \zz) \sqrt{\mu r \lambda^+} + \sqrt{r_v \lambda_v^+}
\end{align*}
The rest of the proof is the same\footnote{Notice here that, we could have loosened the required lower bound on $\xmint$ for this interval in the case when there is no noise} and this ensures exact support recovery and the expression for $\et$.

\emph{Proof of Item 2}: Again, updating $\Phat_{(t)}$ using $\lhatt$'s is a PCA-SDDN problem. We use Corollary \ref{cor:noisy_pca_sddn}.  We compute $\Phat_{j,2}$ as the top $r$ eigenvectors of $\frac{1}{\alpha}\sum_{t = \that_j + \alpha}^{\that_j + 2\alpha - 1} \lhat_t \lhat_t{}'$.
From item 1, $\et$ satisfies \eqref{eq:etdef} for this interval.
In the notation of Corollary \ref{cor:noisy_pca_sddn}, $\yt \equiv \lhat_t$, $\wt \equiv (\e_{\l})_t$, $\lt \equiv \lt$, $\vp_t \equiv (\e_{\v})_t + \vt$, $\P \equiv \P_j$, $\Phat \equiv \Phat_{j,2}$, and $\M_{s,t} = -\left( \bpsi_{\T_t}{}' \bpsi_{\T_t}\right)^{-1} \bpsi_{\T_t}{}'$. So $\norm{\M_{s,t} \P_j} = \| \left( \bpsi_{\T_t}{}' \bpsi_{\T_t}\right)^{-1} \bpsi_{\T_t}{}' \P_j\| \leq 1.2 \max(q_0/4, \zz) := q_1$.
Applying  Corollary \ref{cor:noisy_pca_sddn} with $q \equiv q_1$, $\bz \equiv b_0$ ($b_0$ bounds $\outfracrow(\alpha)$), and setting $\varepsilon_{\text{SE}} = \max(q_1/4, \zz)$, observe that we require
\begin{align*}
4\sqrt{b_0} q_1 f + (2.2)^2 \lambda_v^+/\lambda^- \leq 0.4 \max(q_1/4, \zz)
\end{align*}
Once again recall that the max is $q_1/4$. The above bound holds since $\sqrt{b_0}f \leq 0.01$ and $q_1/4 > \zz > \sqrt{\lambda_v^+/\lambda^-}$.
%
Corollary \ref{cor:noisy_pca_sddn} also requires $\tmaxpca \ge \tmaxpca_*$. Our choice of $\alpha = C f^2 \mu r \log n$ satisfies this requirement since $q_1^2/\varepsilon_\SE^2 = 4^2$ and $(\lambda_v^+/\lambda^-)/\varepsilon_\SE^2 < C$.
Thus, from Corollary \ref{cor:noisy_pca_sddn}, with probability at least $1 - 10n^{-10}$, $\SE(\Phat_{j, 2}, \P_j) \leq \max(q_1/4, \zz)$. 

\textbf{(B) General $k$: } We have conditioned on $\Gamma_{j,k-1}$. This implies that $\SE(\Phat_{j,k-1}, \P_j) \leq \max(q_{k-1}/4, \zz)$.
Consider the interval $[\that_j + (k-1)\alpha, \that_j + k\alpha)$. In this interval, $\Phat_{(t-1)} = \Phat_{j,k-1}$ and thus $\bpsi = \I - \Phat_{j,k-1} \Phat_{j,k-1}{}'$.
Using the same idea as for the $k=2$ case, we have that for the $k$-th interval, $q_{k-1} = \max(0.3^{k-1} q_0,\zz)$. Pick $\varepsilon_{\text{SE}} = \max(q_{k-1}/4, \zz)$. From this it is easy to see that
\begin{align*}
\delta_{2s}(\bpsi) &\leq \left(\max_{|\T| \leq 2s} \norm{\I_\T{}' \Phat_{j, k-1}}\right)^2 \\
&\leq (\SE(\Phat_{j,k-1}, \P_j) + \max_{|\T| \leq 2s} \norm{\I_\T{}'\P_j})^2\\
& \overset{(a)}{\leq} (\SE(\Phat_{j, k-1}, \P_j) + 0.1)^2 \\
&\leq \left[\max\left(0.3^{k-1} (\zz + \SE(\P_{j-1}, \P_j), \zz\right) + 0.1\right]^2 < 0.15
\end{align*}
where $(a)$ follows from \eqref{eq:dense}. Also, as before, 
\begin{align*}
&\norm{\bpsi (\lt + \vt)} \leq \SE(\Phat_{j, k-1}, \P_j) \norm{\at} + \|\vt\| \\
&\leq \max\left(0.3^{k-1} (\zz + \SE(\P_{j-1}, \P_j)), \zz\right) \sqrt{\mu r \lambda^+} + \sqrt{r_v \lambda_v^+} \\
&\overset{(a)}{\leq} \max\left(0.3^{k-1} (\zz  + \semax), \zz\right) \sqrt{\mu r  \lambda^+} + \sqrt{r_v \lambda_v^+}
\end{align*}

\emph{Proof of Item 2}: Again, updating $\Phat_{(t)}$ from $\lhatt$'s is a problem of PCA in sparse data-dependent noise given in Corollary \ref{cor:noisy_pca_sddn}. From {\em Item 1} of this lemma we know that, for $t \in [\that_j + (k-1) \alpha, \that_j + k\alpha]$,  $\et$ satisfies \eqref{eq:etdef}. We update the subspace, $\Phat_{j, k}$ as the top $r$ eigenvectors of $\frac{1}{\alpha}\sum_{t = \that_j + (k-1) \alpha}^{\that_j + k\alpha - 1} \lhat_t \lhat_t{}'$. In the setting above $\yt \equiv \lhat_t$, $\wt \equiv (\e_{\l})_t$, $\lt \equiv \lt$, $\vp_t \equiv (\e_{\v})_t + \vt$, and $\M_{s,t} = -\left( \bpsi_{\T_t}{}' \bpsi_{\T_t}\right)^{-1} \bpsi_{\T_t}{}'$, and so $\norm{\M_{s,t} \P_j} = \| \left( \bpsi_{\T_t}{}' \bpsi_{\T_t}\right)^{-1} \bpsi_{\T_t}{}' \P_j\| \leq 1.2 \max(q_{k-2}/4, \zz) := q_{k-1}$. Applying  Corollary \ref{cor:noisy_pca_sddn} with  $q \equiv q_{k-1}$, $\bz \equiv b_0$ ($b_0$ bounds $\outfracrow(\alpha)$), and setting $\varepsilon_{\text{SE}} = \max(q_{k-1}/4, \zz)$, we require%
$
4\sqrt{b_0} q_{k -1} f + \lambda_v^+/\lambda^- \leq 0.4 \max(q_{k-1}/4, \zz).
$
This holds as explained earlier and hence, by Corollary \ref{cor:noisy_pca_sddn}, the result follows.
\end{proof}

\subsection{Proof of Lemma \ref{lem:sschangedet}}

\begin{proof} 
\emph{Proof of Item 1}:  We are considering an $\alpha$-consecutive frames interval $\J^\alpha$ in $[t_j, t_{j+1})$ during which $\Phat_{(t-1)} = \Phat_{j-1}$. Thus $\bpsi = \bphi  = \I - \Phat_{j-1}  \Phat_{j-1}{}'$.  
Recall from earlier that at all times $t$, $\lhat_t = \l_t - \et + \vt$, where $\et = (\e_{\l})_t + (\e_{\v})_t$,  $\wt \equiv (\e_{\l})_t = \bm{I}_{\Tt}\left(\bpsi_{\Tt}{}'\bpsi_{\Tt}\right)^{-1} \I_{\Tt}{}' \bpsi \lt$ is sparse and data-dependent noise, and $\vp_t \equiv (\e_{\v})_t + \vt$ is small unstructured noise.
As in the earlier proofs, $\wt = (\e_{\l})_t$ can be expressed as $\wt = \I_{\T_t} \M_{s,t} \lt$ where  $\M_{s,t} = \left(\bpsi_{\Tt}{}'\bpsi_{\Tt}\right)^{-1} \I_{\Tt}{}' \bpsi$. Thus, $q = q_0= 1.2 \SE(\Phat_{j-1}, \P_j) \le 1.2 (\zz + \SE_j)$ and $b = b_0$.
Let
\begin{align*}
\frac{1}{\alpha} \sum_t \bphi \lhat_t \lhat_t{}' \bphi  = \frac{1}{\alpha} \sum_t \bphi \lt\lt{}' \bphi' + \bphi \mathrm{noise} \bphi  +  \bphi \mathrm{cross}  \bphi
\end{align*}
where $\mathrm{noise}  = \frac{1}{\alpha} \sum_t  \wt \wt'  + \frac{1}{\alpha} \sum_t \vp_t \vp_t' $ and  $\mathrm{cross}$ contains the cross terms.
By Weyl's inequality,
\begin{align}
\lambda_{\max}\left( \frac{1}{\alpha} \sum_{t \in \J^\alpha } \bphi \lhat_t\lhat_t{}' \bphi \right)  \ge
\lambda_{\max}\left( \frac{1}{\alpha} \sum_{t} \bphi \lt\lt{}' \bphi \right) -  \|\bphi \mathrm{cross} \bphi \|
\label{lower}
\end{align}
Using Corollary \ref{pca_corol} from Appendix \ref{proofs_pca_section}, w.p. at least $1 - 10 n^{-10}$, if $\alpha$ is as given in our Theorem,
\begin{align}
 \| \bphi \mathrm{cross} \bphi' \|  & 
 \le 2.02 \sqrt{b} \|\bphi \P_j \| q_0 \lambda^+ 
 \label{cross_bnds_from_pca}
 \end{align}
Since $\|\bphi \P_j \| \le q = 1.2 (\zz + \SE_j)$, using the above two inequalities,
 \begin{align}\label{eq:lmaxzero}
&\lambda_{\max}\left( \frac{1}{\alpha} \sum_{t \in \J^\alpha} \bphi \lhat_t\lhat_t{}' \bphi \right) \geq \nonumber \\
&\lambda_{\max}\left( \underbrace{   \frac{1}{\alpha} \sum_{t \in \J^{\alpha}} \bphi \P_j \at \at{}' \P_j{}' \bphi }_{\mathrm{Term1}}  \right) -  2.02 \sqrt{b} (1.2 (\zz + \SE_j)^2 ) \lambda^+
\end{align}

We bound the first term of \eqref{eq:lmaxzero}, $\mathrm{Term1}$, as follows. Let $\bphi \P_j \overset{QR}{=} \bm{E}_j \bm{R}_j$ be its reduced QR decomposition. Thus $\bm{E}_j$ is an $n \times r$ matrix with orthonormal columns and $\bm{R}_j$ is an $r \times r$ upper triangular matrix. Let
\begin{align*}
\bm{A} := \bm{R}_j\left( \frac{1}{\alpha} \sum_{t \in \J^{\alpha}} \at \at{}' \right) \bm{R}_j{}'.
\end{align*}
Observe that $\mathrm{Term1}$ can also be written as
\begin{align}\label{eq:decomp_T}
\mathrm{Term1} = \begin{bmatrix}
\bm{E}_j & \bm{E}_{j, \perp}
\end{bmatrix}
\diagmat{\bm{A}}{\bm{0}}
\begin{bmatrix}
\bm{E}_j{}' \\ \bm{E}_{j, \perp}{}'
\end{bmatrix}
\end{align}
and thus $\lambda_{\max}(\bm{A}) = \lambda_{\max}(\mathrm{Term1})$. We work with $\lambda_{\max}(\bm{A})$ in the sequel. We will use the following simple claim.
\begin{claim}\label{claim:psd}
If $\bm{X} \succeq 0$ (i.e., $\bm{X}$ is a p.s.d matrix), where $\bm{X} \in \Re^{r \times r}$, then $\bm{R} \bm{X} \bm{R}{}' \succeq 0$ for all $\bm{R} \in \Re^{r \times r}$.
\end{claim}
\begin{proof}
Since $\bm{X}$ is p.s.d., $\bm{y}{}' \bm{X} \bm{y} \geq 0$ for any vector $\bm{y}$. Use this with $\bm{y} = \bm{R}{}' \bm{z}$ for any $\bm{z} \in \Re^{r}$. We get $\bm{z}{}' \bm{R} \bm{X} \bm{R}{}' \bm{z} \geq 0$. Since this holds for all $\bm{z}$, $\bm{R}\bm{X}\bm{R}{}' \succeq 0$.
\end{proof}
\noindent By Lemma \ref{hp_bnds} from Appendix \ref{proofs_pca_section},  with $\epsilon_0 = 0.01 \lambda^-$,
\begin{align*}
\Pr\left( \frac{1}{\alpha} \sum_t \at \at{}' - (\lambda^- - \epsilon_0) \bm{I} \succeq 0 \right) \geq 1 - 2n^{-10}
\end{align*}

By Claim \ref{claim:psd} from above, with probability $ 1- 2n^{-10}$,
\begin{align*}
&\bm{R}_j \left( \frac{1}{\alpha}\sum_t \at \at{}' - (\lambda^- - \epsilon_0) \bm{I} \right) \bm{R}_j{}' \succeq 0 \\
& \implies \lambda_{\min}\left(\bm{R}_j \left( \frac{1}{\alpha}\sum_t \at \at{}' - (\lambda^- - \epsilon_0) \bm{I} \right) \bm{R}_j{}'\right) \geq 0
\end{align*}
Using Weyl's inequality, with the same probability,
\begin{align*}
\lambda_{\min}\left(\bm{R}_j \left( \frac{1}{\alpha}\sum_t \at \at{}' - (\lambda^- - \epsilon_0) \bm{I} \right) \bm{R}_j{}'\right) \\
\leq \lambda_{\max}\left(\bm{R}_j \left( \frac{1}{\alpha}\sum_t \at \at{}'\right) \bm{R}_j{}'\right) -(\lambda^- - \epsilon_0)\lambda_{\max}\left(\bm{R}_j\bm{R}_j{}'\right)
\end{align*}
and so,
\begin{align}
\lambda_{\max}(\bm{A}) \geq (\lambda^- - \epsilon_0)\lambda_{\max}(\bm{R}_j\bm{R}_j{}').
\label{lammax_A}
\end{align}

Using Lemma \ref{lem:sumprinang} and since $\SE(\Phat_{j-1}, \P_{j-1}) \leq \zz$ we get
\begin{align}\label{eq:nine}
\lambda_{\max}(\bm{R}_j \bm{R}_j{}') = \|\bm{R}_j\|^2 = \SE^2(\Phat_{j-1}, \P_j)
&\geq (\SE_j - 2 \zz)^2
\end{align}
Thus, combining \eqref{eq:lmaxzero}, \eqref{eq:decomp_T}, \eqref{lammax_A}, \eqref{eq:nine}, w.p. at least $1 - 10n^{-10}$,
\begin{align*}
&\lambda_{\max} \left( \frac{1}{\alpha} \sum_{t \in \J^{\alpha}} \bphi \lhatt \lhatt{}' \bphi \right) \\
\geq &\ 0.99 \lambda^-(\SE_j - 2\zz)^2 -  2.02 \sqrt{b_0} (1.2 (\zz + \SE_j)^2 ) \lambda^+ \\
\geq &  0.99 \lambda^- \SE_j (0.6 \SE_j - 4.8 \zz)  = 0.59 \lambda^- \SE_j (\SE_j - 8 \zz)
\end{align*}
In the above, we used $\sqrt{b_0} f = 0.1$.
Since $\SE_j > 9\sqrt{f} \zz$, $0.59 \lambda^- \SE_j (\SE_j - 8 \zz) > 5 \lambda^+ \zz^2 > \omega_{evals}$.%

%

\noindent \emph{Proof of Item 2}:
We proceed as in the proof of item 1 except that now  $\bphi = \bpsi = \I - \Phat_j \Phat_j'$. Thus, $q = q_K=  \zz$ and $\|\bphi \P_j\| \le q_K$.
Using Weyl's inequality and Corollary \ref{pca_corol} from Appendix \ref{proofs_pca_section}, w.p. at least $1 - 10 n^{-10}$,
\begin{align*}
&\lambda_{\max}\left( \frac{1}{\alpha} \sum_{t \in \J^\alpha } \bphi \lhat_t\lhat_t{}' \bphi \right) \\
& \le \lambda_{\max}\left( \frac{1}{\alpha} \sum_t \bphi \lt\lt{}' \bphi \right) + \|\bphi \mathrm{cross} \bphi \| + \lambda_{\max}(\bphi \mathrm{noise} \bphi ) \\
& \le \lambda_{\max}\left( \underbrace{\frac{1}{\alpha} \sum_{t \in \J^{\alpha}} \bphi \P_j \at \at{}' \P_j{}' \bphi}_{\mathrm{Term1}} \right)  \\
&+  2.02 \sqrt{b} \|\bphi \P_j \| q_K \lambda^+  + 1.01 \sqrt{b} q_K^2 \lambda^+  + \zz^2 \lambda^-
 \end{align*}
%

Proceeding as before to bound $\lambda_{\max}(\mathrm{Term1})$, define $\bphi \P_j \overset{QR}{=} \bm{E}_j \bm{R}_j$, define $\bm{A}$ as before, we know $\lambda_{\max}(\mathrm{Term1}) = \lambda_{\max}(\bm{E}_j{}' ( \mathrm{Term1} ) \bm{E}_j) = \lambda_{\max}(\bm{A})$. Further,
\begin{align*}
\lambda_{\max}(\bm{A}) &= \lambda_{\max}\left( \bm{R}_j \left( \frac{1}{\alpha} \sum_{t \in \J^{\alpha}} \at \at{}' \right) \bm{R}_j{}' \right) \\
&\overset{(a)}{\leq} \lambda_{\max}\left( \frac{1}{\alpha} \sum_{t \in \J^{\alpha}} \at \at{} \right) \lambda_{\max}(\bm{R}_j \bm{R}_j{}')
\end{align*}
where $(a)$ uses Ostrowski's theorem \cite[Theorem 5.4.9]{hornjohnson}.
We have
\begin{align*}
\lambda_{\max}(\bm{R}_j \bm{R}_j{}') = \sigma_{\max}^2(\bm{R}_j) = \|\bphi \P_j\|^2 \leq \zz^2
\end{align*}
and we can bound $\lambda_{\max}(\frac{1}{\alpha} \sum_{t \in \J^{\alpha}} \at \at{}')$ using the first item of Lemma \ref{hp_bnds}.
Combining all of the above, and using $\|\bphi \P_j\| \le q_K \leq \zz$ and $b_0 f^2 = 0.01$,
w.p. at least $1 - 10n^{-10}$,
\begin{align*}
\lambda_{\max}\left( \frac{1}{\alpha} \sum_{t \in \J^{\alpha}} \bphi \lhatt \lhatt{}' \bphi \right) 
{\leq} 1.37 \zz^2 \lambda^+
\end{align*}
Recall that $\lthres = 2\zz^2\lambda^+$ and thus, $1.37 \zz^2 \lambda^+ < \lthres$.
\end{proof}

\pgfplotstableread[col sep = comma]{figures/final_files/MO_online_finalSE_TIT.dat}\ltmodata
\pgfplotstableread[col sep = comma]{figures/final_files/bern_online_finalSE_TIT.dat}\ltmodatabern

\begin{figure*}[t!]
\centering
\begin{tikzpicture}
    \begin{groupplot}[
        group style={
            group size=2 by 1,
            y descriptions at=edge left,
        },
        my stylecompare,
        enlargelimits=false,
        width = .5\linewidth,
        height=4cm,
    ]
       \nextgroupplot[
            my legend style compare,
            legend style={at={(0,1.5)}},
            legend columns = 7,
            xlabel={\small{$t$}},
            ylabel={\small{$\log(\SE(\Phat_{(t)}, \P_{(t)}))$}},
        ]
	        \addplot table[x index = {0}, y index = {1}]{\ltmodata};
	        \addplot table[x index = {0}, y index = {2}]{\ltmodata};
	        \addplot table[x index = {0}, y index = {3}]{\ltmodata};
	        \addplot table[x index = {0}, y index = {4}]{\ltmodata};
	        \addplot table[x index = {0}, y index = {5}]{\ltmodata};
	               \nextgroupplot[
            ymode=log,
            xlabel={\small{$t$}},
        ]
	        \addplot table[x index = {0}, y index = {1}]{\ltmodatabern};
	        \addplot table[x index = {0}, y index = {2}]{\ltmodatabern};
	        \addplot table[x index = {0}, y index = {3}]{\ltmodatabern};
	        \addplot table[x index = {0}, y index = {4}]{\ltmodatabern};
	        \addplot table[x index = {0}, y index = {5}]{\ltmodatabern};
    \end{groupplot}
\end{tikzpicture}
\\
\begin{center}
\resizebox{0.9\linewidth}{!}{
\begin{tabular}{cccccccc} \toprule
Outlier Model & GRASTA & s-ReProCS & ORPCA & {\bf NORST} & RPCA-GD & AltProj & \bf{smoothing-NORST} \\
  & ($0.02$ ms) & ($0.9$ ms) & ($1.2$ms) & ($\bm{0.9}$ {\bf ms}) & ($7.8$ms) & ($4.6$ms) & ($\bm{1.7}$\bf{ms}) \\ \toprule
Moving Object & $0.630$ & $0.598$ & $0.861$ & $\bm{4.23 \times 10^{-4}}$ & $4.283$ & $4.441$ & $\bm{8.2 \times 10^{-6}}$ \\
Bernoulli & $6.163$ & $2.805$ & $1.072$ & $\bm{0.002}$ & $0.092$& $0.072$ & $\bm{2.3 \times 10^{-4}}$ \\ \bottomrule
\end{tabular}
}
\end{center}
\caption{{\small \textbf{Top:} Left plot illustrates the $\lt$ error for outlier supports generated using Moving Object Model and right plot illustrates the error under the Bernoulli model. The values are plotted every $k\alpha - 1$ time-frames. \textbf{Bottom:} Comparison of $\|\Lhat - \L\|_F/\|\L\|_F$ for Online and offline RPCA methods. Average time for the Moving Object model is given in parentheses. The offline (batch) methods are performed once on the complete dataset.
}}
\vspace{-0.1in}
\label{fig:Comparison}
\end{figure*}
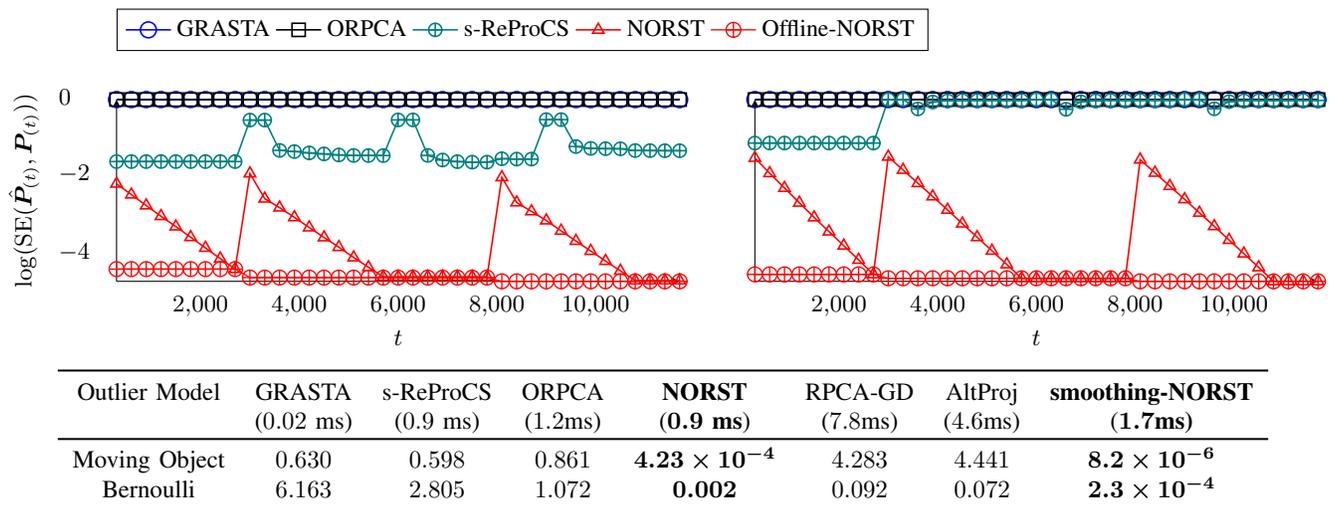

\section{Empirical Evaluation}\label{sec:sims}
In this section we present numerical experiments on synthetic and real data to validate our theoretical claims. Extra experimental details are presented in the Supplementary Material.


\subsubsection{Synthetic Data}
First we compare the results of NORST and smoothing-NORST with RST, Online RPCA, and static RPCA methods. We generate the changing subspaces, $\P_j = e^{\gamma_j \bm{B}_j} \P_{j-1}$ as done in \cite{grass_undersampled} where $\gamma_j$ controls the amount of subspace change and $\bm{B}_j$'s are skew-symmetric matrices. In the first experiment we used the following parameters. $n = 1000$, $d = 12000$, $J = 2$, $t_1 = 3000$, $t_2 = 8000$, $r = 30$, $\gamma_1 = 0.001$, $\gamma_2 = \gamma_1$. We set $\alpha= 300$. Next, we generate the coefficients $\at \in \mathbb{R}^{r}$ as independent zero-mean, bounded random variables. They are $(\at)_i \overset{i.i.d}{\sim} unif[-q_i, q_i]$ where $q_i = \sqrt{f} - \sqrt{f}(i-1)/2r$ for $i = 1, 2, \cdots, r - 1$ and $q_{r} = 1$. thus the condition number is $f$ and we selected $f=50$. For the sparse supports, we considered two models according to which the supports are generated. First we use Model G.24 \cite{rrpcp_dynrpca} which simulates a moving object pacing in the video. For the first $t_{\train} = 100$ frames,  we used a smaller fraction of outliers with parameters $s/n = 0.01$, $b_0 = 0.01$. For $t > t_\train$ we used $s/n = 0.05$ and $b_0 = 0.3$. Secondly, we used the Bernoulli model to simulate sampling uniformly at random, i.e., each entry of the matrix, is independently selected with probability $\rho = 0.01$ for the first $t_{\train}$ frames and with probability $\rho = 0.3$ for subsequent frames. The sparse outlier magnitudes for both support models are generated uniformly at random from the interval $[x_{\min}, x_{\max}]$ with $x_{\min} = 10$ and $x_{\max} = 20$.

We initialized the s-ReProCS and NORST algorithms using AltProj applied to $\Y_{[1,t_\train]}$ with $t_\train=100$. For the parameters to AltProj we used used the true value of $r$, $15$ iterations and a threshold of $0.01$. This, and the choice of $\gamma_1$ and $\gamma_2$ ensure that $\SE(\Phat_\init, \P_0) \approx \SE(\P_1, \P_0) \approx \SE(\P_2, \P_1) \approx 0.01$. The other algorithm parameters are set as mentioned in the theorem, i.e., $K = \lceil \log(c/\varepsilon) \rceil = 8$, $\alpha = C r \log n = 300$, $\omega = x_{\min}/2 = 5$ and $\xi = 7 x_{\min}/15 = 0.67$, $\lthres = 2 \varepsilon^2 \lambda^+ = 7.5 \times 10^{-4}$. For the other online methods we implement the algorithms without modifications. The regularization parameter for ORPCA was set as with $\lambda_1 = 1 / \sqrt{n}$ and $\lambda_2 = 1 / \sqrt{d}$ according to \cite{xu_nips2013_1}. Wherever possible we set the tolerance as $10^{-6}$ and $100$ iterations to match that of our algorithm. As shown in Fig. \ref{fig:Comparison}, NORST is significantly better than all the RST methods - s-ReProCS \cite{rrpcp_dynrpca}, and two popular heuristics - ORPCA \cite{xu_nips2013_1} and GRASTA \cite{grass_undersampled}.

We also provide a comparison of smoothing techniques in Fig \ref{fig:Comparison}. To ensure a valid time comparison, we implement the static RPCA methods on the entire data matrix $\Y$. Although, we could implement the static techniques on disjoint batches of size $\alpha$, we observed that this did not yield significant improvement in terms of reconstruction accuracy, while being considerably slower, and thus we report only the latter setting. As can be seen, smoothing NORST outperforms all static RPCA methods, both for the moving object and the Bernoulli models. For the batch comparison we used PCP, AltProj and  RPCA-GD. We set the regularization parameter for PCP $1/\sqrt{n}$ in accordance with \cite{rpca}. The other known parameters, $r$ for Alt-Proj, outlier-fraction for RPCA-GD, are set the ground truth. For all algorithms we set the threshold as $10^{-6}$ and the number of iterations to $100$. All results are averaged over $100$ independent runs.

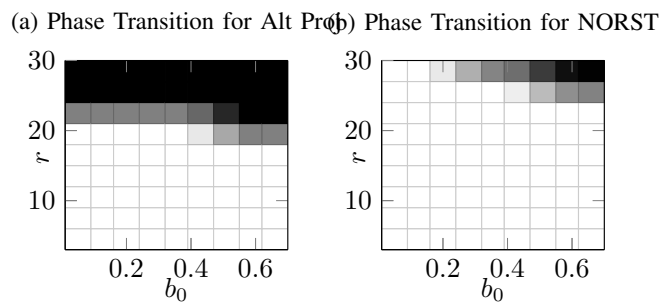
\begin{figure}[t!]
\begin{center}
\resizebox{\linewidth}{!}{%
\begin{tikzpicture}
    \begin{groupplot}[
        group style={
            group size=2 by 1,
            horizontal sep=1.2cm
        },
        width = .5\linewidth,
        height = 4cm
    	]
        \nextgroupplot[
		        view={0}{90},
				xlabel={$b_0$},
               	ylabel={$r$},
                colormap/blackwhite,
                title={\small{(a) Phase Transition for Alt Proj}},
				x label style={at={(axis description cs:0.5,-0.1)},anchor=north},
				y label style={at={(axis description cs:-0.1,0.4)},anchor=west},
        ]
        \node [text width=1em,anchor=north west] at (rel axis cs: 0,1)
                {\subcaption{\label{fig:phasetransvsnbnd}}};

                	\addplot3[surf] file {figures/final_files/PhaseTransAltProj.dat};
                			
		            \nextgroupplot[
		       view={0}{90},
               xlabel={$b_0$},
               ylabel={$r$},
               colormap/blackwhite,
               title={\small{(b) Phase Transition for NORST}},
               scaled y ticks=false, tick label style={/pgf/number format/fixed},
               x label style={at={(axis description cs:0.5,-0.1)},anchor=north},
               y label style={at={(axis description cs:-0.1,0.4)},anchor=west},
        ]
    \node [text width=1em,anchor=north west] at (rel axis cs: 0,1)
                {\subcaption{\label{fig:phasetransvsngausr}}};

        	\addplot3[surf] file {figures/final_files/PhaseTransNORST.dat};
        	
        \end{groupplot}
\end{tikzpicture}%
}
\end{center}%
\caption{Empirical probability that $\|\Lhat - \L\|_F / \|\L\|_F < 0.5$ for AltProj and for smoothing NORST. Note that NORST indeed has a much higher tolerance to outlier fraction per row as compared to AltProj. Black denotes $0$ and white denotes $1$.}
\vspace{-0.5cm}
\label{fig:phase_trans}
\end{figure}

In Fig. \ref{fig:phase_trans} we validate our claim of  NORST admitting a higher fraction of outliers per row. We only compare with AltProj since it is has the highest tolerance among other methods. We chose $10$ different values of each of $r$ and $b_0$ (we slightly misuse notation here to let $b_0 := \outfracrow$ for this section only). For each pair of $b_0$ and $r$ we implemented NORST and ALtProj over $100$ independent trials and computed the relative error, $\|\Lhat - \L\|_F / \|L\|_F$ for each run. We illustrate the fraction of times the error seen by each algorithm is less than a threshold, $0.5$. We chose this threshold since for smaller values, AltProj consistently failed. As can be seen, NORST is able to tolerate a much larger fraction of outlier-per-row as compared to AltProj.


\pgfplotstableread[col sep = comma]{figures/final_files/xmin_variation.dat}\xmindata
\pgfplotsset{every axis title/.append style={at={(.5,1.15)}}}
\begin{figure*}[t!]
\centering
\begin{tikzpicture}
    \begin{groupplot}[
        group style={
            group size=2 by 1,
            y descriptions at=edge left,
        },
        my stylecompare,
        enlargelimits=false,
        width = .5\linewidth,
        height=4cm,
        ymin=-5, ymax = 1
    ]
       \nextgroupplot[
            legend entries={
					$x_{\min} = 0.5$,
            		$x_{\min} = 5$,
            		$x_{\min} = 10$,	
            		},
            legend style={at={(1.5,1.25)}},
            legend columns = 3,
            xlabel=$t$,
            ylabel={\small{$\log(\SE(\Phat_{(t)}, \P_{(t)}))$}},
            title={\small AltProj},
        ]

        	        \addplot[red, line width=1.6pt, mark=square,mark size=3.5pt, mark repeat=3] table[x index = {0}, y index = {7}]{\xmindata};
	        \addplot[cyan, line width=1.6pt, mark=o,mark size=3.5pt, mark repeat=3] table[x index = {0}, y index = {8}]{\xmindata};
	        \addplot[olive, line width=1.6pt, mark=Mercedes star,mark size=3pt, mark repeat=3] table[x index = {0}, y index = {9}]{\xmindata};

	               \nextgroupplot[
            xlabel=$t$,
            title={\small smoothing NORST}
        ]
	   	        \addplot[red, line width=1.6pt, mark=square,mark size=3.5pt, mark repeat=3] table[x index = {0}, y index = {4}]{\xmindata};
	        \addplot[cyan, line width=1.6pt, mark=o,mark size=3.5pt, mark repeat=3] table[x index = {0}, y index = {5}]{\xmindata};
	        \addplot[olive, line width=1.6pt, mark=Mercedes star,mark size=3pt, mark repeat=3] table[x index = {0}, y index = {6}]{\xmindata};

    \end{groupplot}
\end{tikzpicture}
\caption{In the above plots we show the variation of the subspace errors for varying $x_{\min}$. In particular, we set all the non-zero outlier values to $x_{\min}$. 
The results are averaged over $100$ independent trials.}
\vspace{-0.3in}
\label{fig:xmin_var}
\end{figure*}
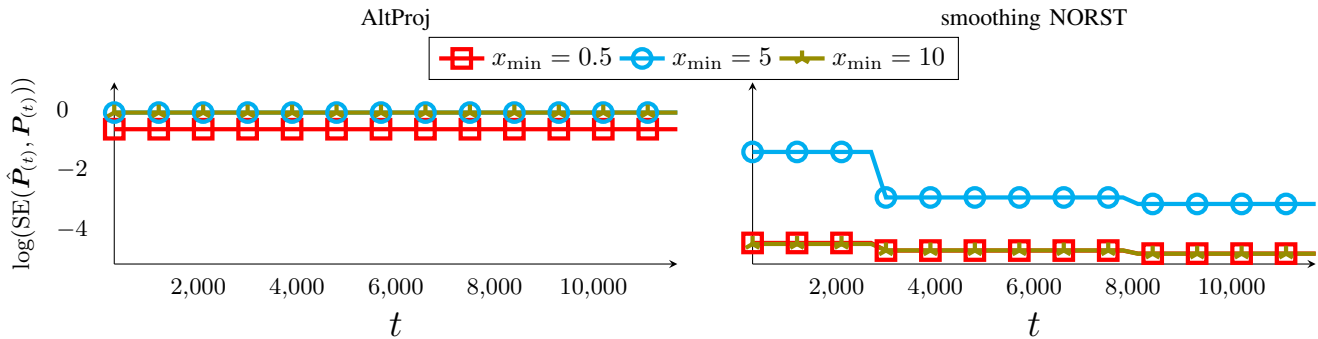

In the third experiment we analyze the effect of the lower bound on the outlier magnitude $x_{\min}$ with the performance of NORST and AltProj. We show the results in Fig. \ref{fig:xmin_var}. The only change in data generation is that we now choose three different values of $x_{\min} = \{0.5, 5, 10\}$, and we set all the non-zero entries of the sparse matrix to be equal to $x_{\min}$. This is actually harder than allowing the sparse outliers to take on any value since for a moderately low value of $x_{\min}$ the outlier-lower magnitude bound of Theorem \ref{thm1} is violated. This is indeed confirmed by the numerical results presented in Fig. \ref{fig:xmin_var}. (i) When $x_{\min} = 0.5$, NORST works well since now all the outliers get classified as the small unstructured noise $\vt$. (ii) When $x_{\min} = 10$, NORST still works well because now $\xmint$ is large enough so that the outlier support is mostly correctly recovered. (iii) But when $x_{\min} = 5$ the NORST reconstruction error stagnates around $10^{-3}$. All AltProj errors are much worse than those of NORST because the outlier fraction per row is the same as in the first experiment and thus the effect of varying $\xmint$ is not pronounced. 

\subsubsection{Real Data}
We also evaluate our algorithm for the task of Background Subtraction. For the AltProj algorithm we set $r = 40$. The remaining parameters were used with default setting.
For NORST, we set $\alpha = 60$, $K = 3$, $\xi_t = \|\bm{\Psi} \hat{\bm{\ell}}_{t-1}\|_2$. We found that these parameters work for most videos that we verified our algorithm on. For RPCA-GD we set the ``corruption fraction'' $\alpha = 0.2$ as described in their paper. 

We use two standard datasets, the Meeting Room (MR) and the Lobby (LB) sequences. LB is a relatively easy sequence since the background is static for the most part, and the foreground occlusions are small in size. As can be seen from Fig. \ref{fig:mr_full} (first two rows), most algorithms perform well on this dataset.  MR is a challenging data set since the color of the foreground (person) is very similar to the background curtains, and the size of the object is very large. Thus, NORST is able to outperform all methods, while being fast.  


\begin{figure*}[t!]
\begin{center}
\resizebox{.8\linewidth}{!}{
\begin{tabular}{@{}c@{}c@{}c@{}c@{}c@{}c@{}}
\\    \newline
	\includegraphics[width=0.11\linewidth, height=1.5cm]{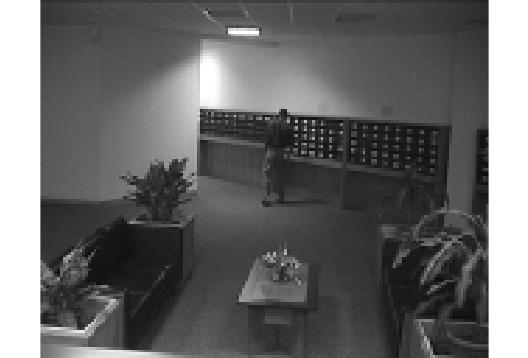}
&
	\includegraphics[width=0.11\linewidth, height=1.5cm, trim={.7cm, 0cm, .7cm, 0cm}, clip]{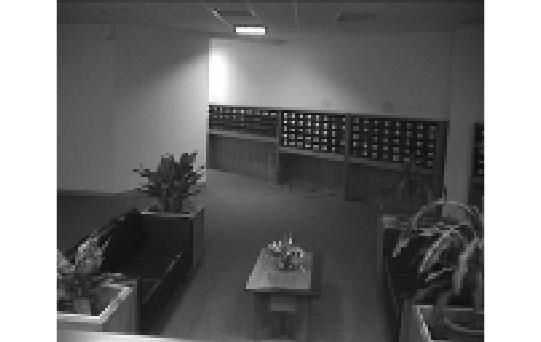}
&
	\includegraphics[width=0.11\linewidth, height=1.5cm, trim={.7cm, 0cm, .7cm, 0cm}, clip]{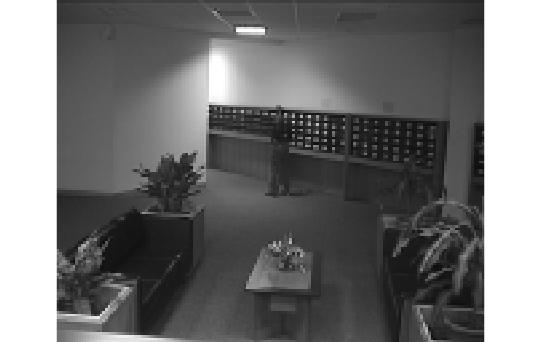}
&
	\includegraphics[width=0.11\linewidth, height=1.5cm, trim={.7cm, 0cm, .7cm, 0cm}, clip]{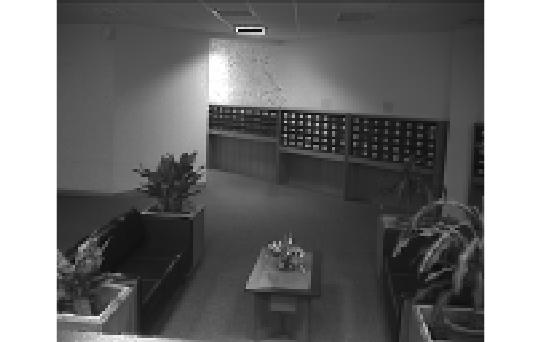}
&
	\includegraphics[width=0.11\linewidth, height=1.5cm, trim={.7cm, 0cm, .7cm, 0cm}, clip]{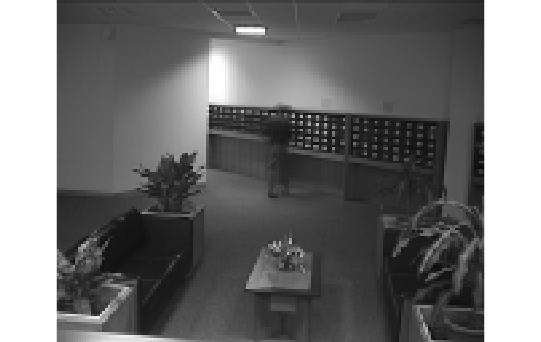}
&
	\includegraphics[width=0.11\linewidth, height=1.5cm, trim={1.25cm, 0cm, 1.2cm, 0cm}, clip]{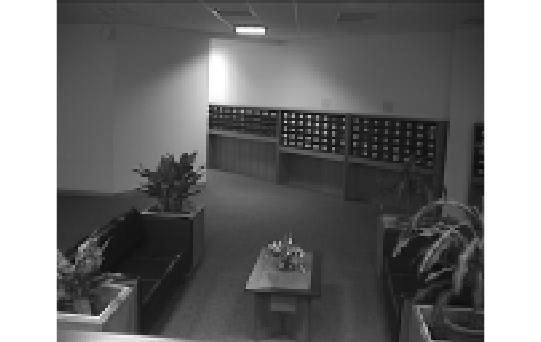}
\\    \newline
	{\includegraphics[width=0.11\linewidth, height=1.5cm, trim={.8cm, 0cm, 1cm, 0cm}, clip]{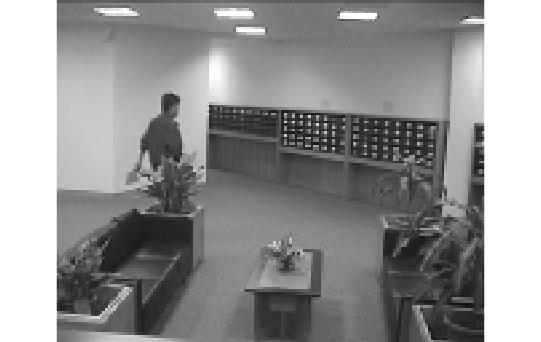}}
&
	{\includegraphics[width=0.11\linewidth, height=1.5cm, trim={.7cm, 0cm, .7cm, 0cm}, clip]{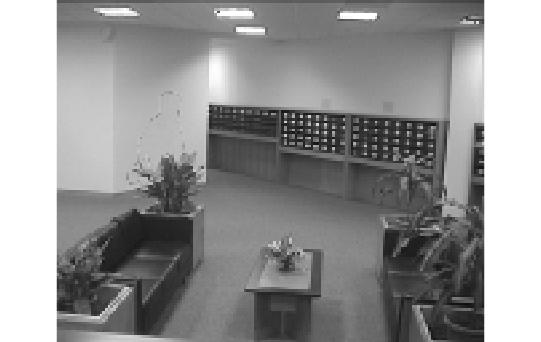}}
&
	{\includegraphics[width=0.11\linewidth, height=1.5cm, trim={1.25cm, 0cm, 1.2cm, 0cm}, clip]{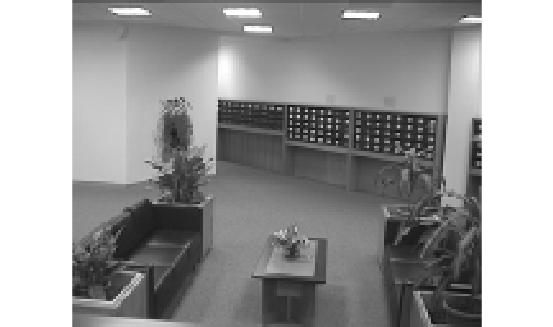}}
&
	{\includegraphics[width=0.11\linewidth, height=1.5cm, trim={1.2cm, 0cm, 1.2cm, 0cm}, clip]{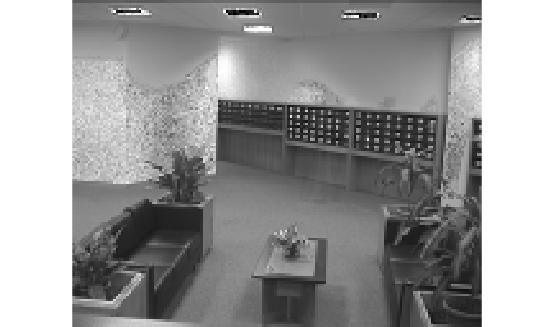}}
&
	{\includegraphics[width=0.11\linewidth, height=1.5cm, trim={1.2cm, 0cm, 1.25cm, 0cm}, clip]{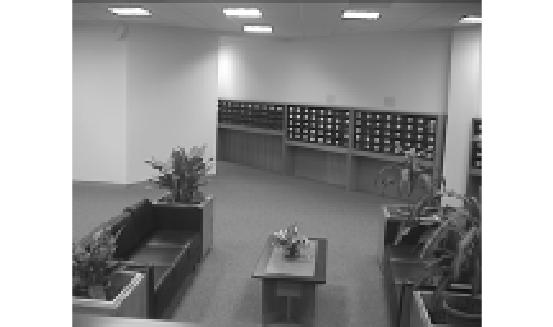}}
&
	{\includegraphics[width=0.11\linewidth, height=1.5cm, trim={.7cm, 0cm, .7cm, 0cm}, clip]{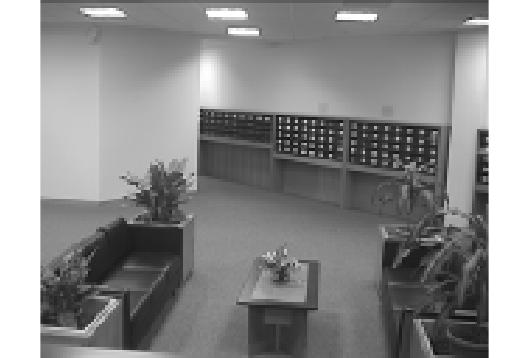}}
		\\ \newline 
			\tiny{Original} & \tiny{NORST($16.5$ms)} & \tiny{AltProj($26.0$ms)} & \tiny{RPCA-GD($29.5$ms)} & \tiny{GRASTA($2.5$ms)} & \tiny{PCP ($44.6$ms)} \\ \newline		
		\includegraphics[width=0.11\linewidth, height=1.5cm, trim={1.2cm, 0cm, 1.25cm, 0cm}, clip]{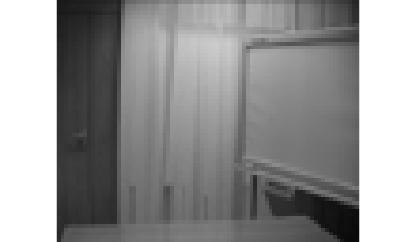}
&
	\includegraphics[width=0.11\linewidth, height=1.5cm, trim={1.2cm, 0cm, 1.25cm, 0cm}, clip]{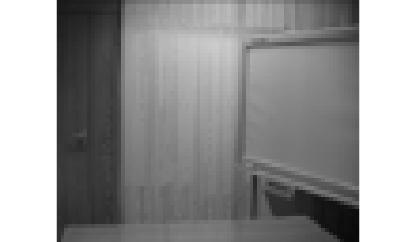}
&
	\includegraphics[width=0.11\linewidth, height=1.5cm, trim={1.1cm, 0cm, 1.15cm, 0cm}, clip]{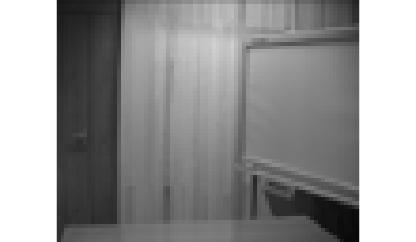}
&
	\includegraphics[width=0.11\linewidth, height=1.5cm, trim={1.05cm, 0cm, 1.25cm, 0cm}, clip]{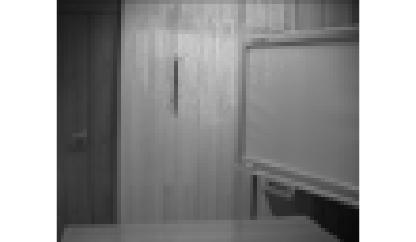}
&
	\includegraphics[width=0.11\linewidth, height=1.5cm, trim={0.38cm, 0cm, 0.4cm, 0cm}, clip]{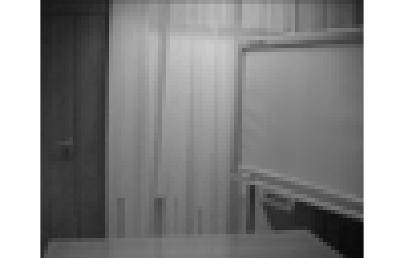}
&
	\includegraphics[width=0.11\linewidth, height=1.5cm, trim={2.1cm, 0cm, 2.2cm, 0cm}, clip]{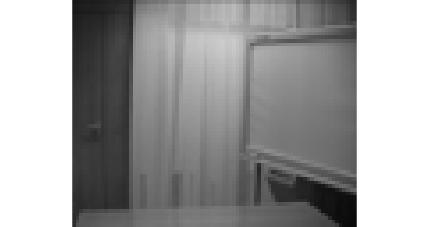}
\\    \newline
	{\includegraphics[width=0.11\linewidth, height=1.5cm, trim={1.8cm, 0cm, 1.8cm, 0cm}, clip]{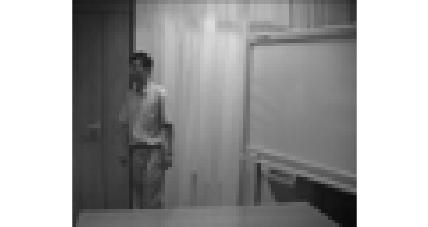}}
&
	{\includegraphics[width=0.11\linewidth, height=1.5cm, trim={1.2cm, 0cm, 1.25cm, 0cm}, clip]{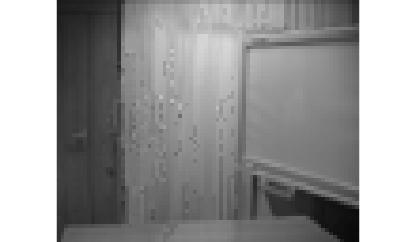}}
&
	{\includegraphics[width=0.11\linewidth, height=1.5cm, trim={1.7cm, 0cm, 1.7cm, 0cm}, clip]{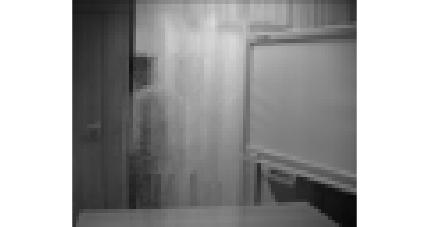}}
&
	{\includegraphics[width=0.11\linewidth, height=1.5cm, trim={1.65cm, 0cm, 1.8cm, 0cm}, clip]{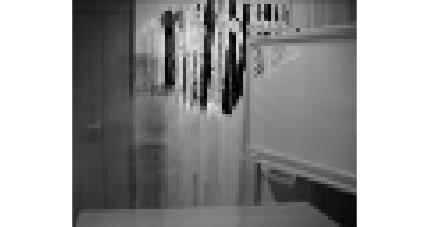}}
&
	{\includegraphics[width=0.11\linewidth, height=1.5cm, trim={1cm, 0cm, 1.15cm, 0cm}, clip]{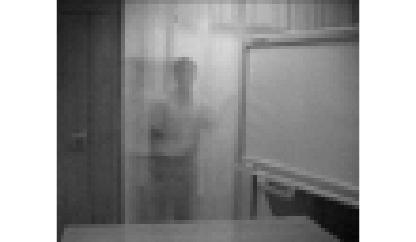}}
&
	{\includegraphics[width=0.11\linewidth, height=1.5cm, trim={1.5cm, 0cm, 1.6cm, 0cm}, clip]{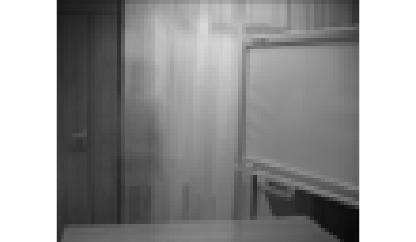}}  \\ \newline
	\tiny{Original} & \tiny{NORST($72.5$ms)} & \tiny{AltProj($133.1$ms)} & \tiny{RPCA-GD($113.6$ms)} & \tiny{GRASTA($18.9$ms)}  & \tiny{PCP($240.7$ms)} \\
\end{tabular}
}

\caption{\footnotesize Comparison of visual performance in Foreground Background separation. The first two rows are for the LB dataset and the last two rows are for the MR dataset. The time taken by each algorithm (per frame) in milliseconds is provided in parenthesis.}
\label{fig:mr_full}
\end{center}
\vspace{-1cm}
\end{figure*}

\section{Conclusions and Future Directions} \label{conclude}
In this work we developed a fast and (nearly) delay optimal robust subspace tracking solution that we called NORST. NORST is a mini-batch algorithm with memory complexity that is also nearly optimal. It detects subspace changes and tracks them to $\zz$ accuracy with a delay that is more than the subspace dimension $r$ by only log factors: the delay is order $r \log n \log(1/\zz)$. The memory complexity is $n$ times this number while $nr$ is the amount of memory required to store the output subspace estimate.
Our guarantee for NORST needs assumptions similar to those needed by standard robust PCA solutions. Different from standard robust PCA, we need slow subspace change, we replace right singular vectors' incoherence by a statistical version of it, but we need a weaker bound on outlier fractions per row.

Slow subspace change is a natural assumption for background images of static camera videos (with no sudden scene changes).
Our statistical assumptions on $\at$ are mild and can be relaxed further. As already explained, the identically distributed requirement can be relaxed. In the video application, the zero mean assumption  can be approximately satisfied if we estimate the mean background image by computing the empirical average of the first $t_\train$ frames, $\Lhat_{[1:t_\train]}$, obtained using AltProj.  Mutual independence of $\at$'s models the fact that the changes in each background image w.r.t. a ``mean'' background are independent, when conditioned on their subspace. This is valid, for example, if the background changes are due to illumination variations or due to moving curtains (see Fig. \ref{fig:mr_full}). Mutual independence can be relaxed to instead assuming an autoregressive model on the $\at$'s: this will require using the matrix Azuma inequality \cite{tropp} to replace matrix Bernstein. We believe the zero mean requirement can also be eliminated.

\bibliographystyle{IEEEbib}
{\footnotesize
\bibliography{tipnewpfmt_kfcsfullpap}

\begin{thebibliography}{10}

\bibitem{rrpcp_icml}
P.~Narayanamurthy and N.~Vaswani,
\newblock ``Nearly optimal robust subspace tracking,''
\newblock in {\em Intnl. Conf. Machine Learning (ICML)}, 2018, pp. 3701--3709.

\bibitem{pca_dd_isit}
N.~Vaswani and P.~Narayanamurthy,
\newblock ``Pca in sparse data-dependent noise,''
\newblock in {\em ISIT}, 2018, pp. 641--645.

\bibitem{pca_dd}
N.~Vaswani and P.~Narayanamurthy,
\newblock ``Finite sample guarantees for pca in non-isotropic and
  data-dependent noise,''
\newblock in {\em Allerton 2017, long version at arXiv:1709.06255}, 2017.

\bibitem{rpca}
E.~J. Cand{\`e}s, X.~Li, Y.~Ma, and J.~Wright,
\newblock ``Robust principal component analysis?,''
\newblock {\em J. ACM}, vol. 58, no. 3, 2011.

\bibitem{rpca2}
V.~Chandrasekaran, S.~Sanghavi, P.~A. Parrilo, and A.~S. Willsky,
\newblock ``Rank-sparsity incoherence for matrix decomposition,''
\newblock {\em SIAM Journal on Optimization}, vol. 21, 2011.

\bibitem{rpca_zhang}
D.~Hsu, S.~M. Kakade, and T.~Zhang,
\newblock ``Robust matrix decomposition with sparse corruptions,''
\newblock {\em IEEE Trans. Info. Th.}, Nov. 2011.

\bibitem{robpca_nonconvex}
P.~Netrapalli, U~N Niranjan, S.~Sanghavi, A.~Anandkumar, and P.~Jain,
\newblock ``Non-convex robust pca,''
\newblock in {\em NIPS}, 2014.

\bibitem{rpca_gd}
X.~Yi, D.~Park, Y.~Chen, and C.~Caramanis,
\newblock ``Fast algorithms for robust pca via gradient descent,''
\newblock in {\em NIPS}, 2016.

\bibitem{rmc_gd}
Y.~Cherapanamjeri, K.~Gupta, and P.~Jain,
\newblock ``Nearly-optimal robust matrix completion,''
\newblock {\em ICML}, 2016.

\bibitem{rrpcp_allerton}
C.~Qiu and N.~Vaswani,
\newblock ``Real-time robust principal components' pursuit,''
\newblock in {\em Allerton Conf. on Communication, Control, and Computing},
  2010.

\bibitem{rrpcp_perf}
C.~Qiu, N.~Vaswani, B.~Lois, and L.~Hogben,
\newblock ``Recursive robust pca or recursive sparse recovery in large but
  structured noise,''
\newblock {\em IEEE Trans. Info. Th.}, pp. 5007--5039, August 2014.

\bibitem{grass_undersampled}
J.~He, L.~Balzano, and A.~Szlam,
\newblock ``Incremental gradient on the grassmannian for online foreground and
  background separation in subsampled video,''
\newblock in {\em IEEE Conf. on Comp. Vis. Pat. Rec. (CVPR)}, 2012.

\bibitem{xu_nips2013_1}
J.~Feng, H.~Xu, and S.~Yan,
\newblock ``Online robust pca via stochastic optimization,''
\newblock in {\em NIPS}, 2013.

\bibitem{robust_admm}
Nguyen~Viet Dung, Nguyen~Linh Trung, Karim Abed-Meraim, et~al.,
\newblock ``Robust subspace tracking with missing data and outliers via admm,''
\newblock in {\em 2019 27th European Signal Processing Conference (EUSIPCO)}.
  IEEE, 2019, pp. 1--5.

\bibitem{rst_lr}
S.~{Javed}, A.~{Mahmood}, J.~{Dias}, and N.~{Werghi},
\newblock ``Robust structural low-rank tracking,''
\newblock {\em IEEE Transactions on Image Processing}, vol. 29, pp. 4390--4405,
  2020.

\bibitem{rst_sparse}
Tianzhu Zhang, Changsheng Xu, and Ming-Hsuan Yang,
\newblock ``Robust structural sparse tracking,''
\newblock {\em IEEE transactions on pattern analysis and machine intelligence},
  vol. 41, no. 2, pp. 473--486, 2018.

\bibitem{rrpcp_aistats}
J.~Zhan, B.~Lois, H.~Guo, and N.~Vaswani,
\newblock ``{Online (and Offline) Robust PCA: Novel Algorithms and Performance
  Guarantees},''
\newblock in {\em Intnl. Conf. Artif. Intell. Stat. (AISTATS)}, 2016.

\bibitem{rrpcp_dynrpca}
P.~Narayanamurthy and N.~Vaswani,
\newblock ``Provable dynamic robust pca or robust subspace tracking,''
\newblock {\em IEEE Transactions on Information Theory}, vol. 65, no. 3, pp.
  1547--1577, 2019.

\bibitem{rrpcp_tsp}
H.~Guo, C.~Qiu, and N.~Vaswani,
\newblock ``An online algorithm for separating sparse and low-dimensional
  signal sequences from their sum,''
\newblock {\em IEEE Trans. Sig. Proc.}, vol. 62, no. 16, pp. 4284--4297, 2014.

\bibitem{tropp}
J.~A. Tropp,
\newblock ``Just relax: Convex programming methods for identifying sparse
  signals,''
\newblock {\em IEEE Trans. Info. Th.}, pp. 1030--1051, March 2006.

\bibitem{davis_kahan}
C.~Davis and W.~M. Kahan,
\newblock ``The rotation of eigenvectors by a perturbation. iii,''
\newblock {\em SIAM J. Numer. Anal.}, vol. 7, pp. 1--46, Mar. 1970.

\bibitem{jsait_arxiv}
Namrata Vaswani and Praneeth Narayanamurthy,
\newblock ``Fast robust subspace tracking via pca in sparse data-dependent
  noise,''
\newblock {\em arXiv preprint arXiv:2006.08030}, 2020.

\bibitem{nadler}
B.~Nadler,
\newblock ``Finite sample approximation results for principal component
  analysis: A matrix perturbation approach,''
\newblock {\em Ann. Statist.}, 2008.

\bibitem{matcomp_candes}
E.~J. Candes and B.~Recht,
\newblock ``Exact matrix completion via convex optimization,''
\newblock {\em Found. of Comput. Math}, , no. 9, pp. 717--772, 2008.

\bibitem{selin_reprocs}
A.~Ozdemir, E.~M. Bernat, and S.~Aviyente,
\newblock ``Recursive tensor subspace tracking for dynamic brain network
  analysis,''
\newblock {\em IEEE Transactions on Signal and Information Processing over
  Networks}, 2017.

\bibitem{chordal_dist}
K.~Ye and L.~H. Lim,
\newblock ``Schubert varieties and distances between subspaces of different
  dimensions,''
\newblock {\em SIAM Journal on Matrix Analysis and Applications}, vol. 37, no.
  3, pp. 1176--1197, 2016.

\bibitem{rrpcp_review}
N.~Vaswani, T.~Bouwmans, S.~Javed, and P.~Narayanamurthy,
\newblock ``Robust subspace learning: Robust pca, robust subspace tracking and
  robust subspace recovery,''
\newblock {\em IEEE Signal Proc. Magazine}, July 2018.

\bibitem{candes_rip}
E.~Candes,
\newblock ``The restricted isometry property and its implications for
  compressed sensing,''
\newblock {\em C. R. Math. Acad. Sci. Paris Serie I}, 2008.

\bibitem{normal_approx}
Vladimir Koltchinskii, Karim Lounici, et~al.,
\newblock ``Normal approximation and concentration of spectral projectors of
  sample covariance,''
\newblock {\em The Annals of Statistics}, vol. 45, no. 1, pp. 121--157, 2017.

\bibitem{vince_vu_annals}
V.~Q. Vu and J.~Lei,
\newblock ``Minimax sparse principal subspace estimation in high dimensions,''
\newblock {\em Annals of Statistics}, 2013.

\bibitem{sparsepca_cai}
T~Tony Cai, Zongming Ma, and Yihong Wu,
\newblock ``Sparse pca: Optimal rates and adaptive estimation,''
\newblock {\em The Annals of Statistics}, vol. 41, no. 6, pp. 3074--3110, 2013.

\bibitem{kpca_first}
Bernhard Sch{\"o}lkopf, Alexander Smola, and Klaus-Robert M{\"u}ller,
\newblock ``Nonlinear component analysis as a kernel eigenvalue problem,''
\newblock {\em Neural computation}, vol. 10, no. 5, pp. 1299--1319, 1998.

\bibitem{kernelpca_nips}
Laurent Zwald and Gilles Blanchard,
\newblock ``On the convergence of eigenspaces in kernel principal component
  analysis,''
\newblock in {\em Advances in neural information processing systems}, 2006, pp.
  1649--1656.

\bibitem{svd_topic_est}
Zheng~Tracy Ke and Minzhe Wang,
\newblock ``A new svd approach to optimal topic estimation,''
\newblock {\em arXiv preprint arXiv:1704.07016}, 2017.

\bibitem{rel_perturb}
Ren-Cang Li,
\newblock ``Relative perturbation theory: Ii. eigenspace and singular subspace
  variations,''
\newblock {\em SIAM J. Matrix Anal. Appl.}, vol. 20, no. 2, pp. 471--492, 1998.

\bibitem{rel_perturb2}
Ilse~CF Ipsen,
\newblock ``An overview of relative sin $\theta$ theorems for invariant
  subspaces of complex matrices,''
\newblock {\em Journal of computational and applied mathematics}, vol. 123, no.
  1-2, pp. 131--153, 2000.

\bibitem{zhan_pcp_jp}
J.~Zhan and N.~Vaswani,
\newblock ``Robust pca with partial subspace knowledge,''
\newblock {\em IEEE Trans. Sig. Proc.}, July 2015.

\bibitem{rrpcp_tsp19}
P.~Narayanamurthy, V.~Daneshpajooh, and N.~Vaswani,
\newblock ``Provable subspace tracking from missing data and matrix
  completion,''
\newblock {\em IEEE Trans. Sig. Proc.}, pp. 4245--4260, 2019.

\bibitem{hornjohnson}
R.~Horn and C.~Johnson,
\newblock {\em Matrix Analysis},
\newblock Cambridge Univ. Press, 1985.

\bibitem{tail_bound}
J.~A. Tropp,
\newblock ``User-friendly tail bounds for sums of random matrices,''
\newblock {\em Found. Comput. Math.}, vol. 12, no. 4, 2012.

\bibitem{rrpcp_isit15}
B.~Lois and N.~Vaswani,
\newblock ``Online matrix completion and online robust pca,''
\newblock in {\em IEEE Intl. Symp. Info. Th. (ISIT)}, 2015.

\bibitem{l1_best}
Lin Xiao and Tong Zhang,
\newblock ``A proximal-gradient homotopy method for the l1-regularized
  least-squares problem,''
\newblock in {\em ICML}, 2012.

\bibitem{musco2015randomized}
Cameron Musco and Christopher Musco,
\newblock ``Randomized block krylov methods for stronger and faster approximate
  singular value decomposition,''
\newblock in {\em Advances in Neural Information Processing Systems}, 2015, pp.
  1396--1404.

\end{thebibliography}
}

\appendices

\renewcommand\thetheorem{\arabic{section}.\arabic{theorem}}

\counterwithin{theorem}{section}

\section{Proofs for Sec. \ref{pca_section}} \label{proofs_pca_section}


\subsection{Proof of Theorem \ref{mainthm_pca}}
%

\begin{proof}[Proof of Theorem \ref{mainthm_pca}]
This uses the Davis-Kahan sin theta theorem \cite{davis_kahan}:
\begin{lem}[Davis-Kahan $\sin \theta$ theorem]\label{sintheta}
Let $\D_0$ be a Hermitian matrix whose span of top $r$ eigenvectors equals $\Span(\P)$. Let $\D$ be the Hermitian matrix with top $r$ eigenvectors  $\Phat$. Then,
\begin{align}
\SE(\Phat,\P) &\le \frac{\|(\D-\D_0)\P\|}{\lambda_r(\D_0) - \lambda_{r+1}(\D)} \nonumber \\
& \le  \frac{\|\D-\D_0\|}{\lambda_r(\D_0) - \lambda_{r+1}(\D_0) - \lambda_{\max}(\D-\D_0)}
\label{sintheta_bnd_2}
\end{align}
as long as the denominator is positive. The second inequality follows from the first using Weyl's inequality.%
\end{lem}

For our proof, set $\bm{D}_0 = \frac{1}{\tmaxpca}\sum_t \lt\lt{}'$. Notice that this is a Hermitian matrix with $\P$ as the top $r$ eigenvectors.
Let $\bm{D} = \frac{1}{\tmaxpca}\sum_t \yt \yt{}'$. Recall that $\Phat$ is its matrix of top $r$ eigenvectors.
Observe
\begin{align*}
&\D - \D_0 = \frac{1}{\tmaxpca} \sum_t (\yt \yt{}' - \lt \lt{}') \\
&= \frac{1}{\tmaxpca} \sum_t  ( \wt \wt{}' + \bv \bv{}'  + \lt \wt{}'  + \bv \wt{}'  \\
&+ \lt \bv{}' + \wt \lt{}' + \wt \bv{}' + \bv \lt{}'  ) \\
&:= \noise_{\w} + \noise_{\bv} + \cross_{\l,\w}  +  \cross_{\l, \bv} \\
&+ \cross_{\bv,\w} +  \cross_{\l, \w}{}' + \cross_{\l, \bv}{}'  + \cross_{\bv,\w}{}' \\
&: = \noise + \cross + \cross{}'
\end{align*}
Also notice that $\lambda_{r+1}(\D_0) = 0$, $\lambda_r(\D_0) = \lambda_{\min}\left(\frac{1}{\tmaxpca} \sum_t \at \at{}'\right)$.
Now, applying Theorem \ref{sintheta},
\begin{align*}
\SE(\Phat,\P) \leq \frac{2\|\cross\| + \|\noise\|}{\lambda_{\min}\left(\frac{1}{\tmaxpca}\sum_t\at\at{}'\right) - \mathrm{numerator}}
\end{align*}

Now, we can bound $\|\cross\| \leq \|\E[\cross]\| + \|\cross - \E[\cross]\|$ and similarly for the $\noise$ term. We use the Cauchy-Schwartz inequality for bounding the expected values of both.

Recall that $\M_t = \M_{2,t} \M_{1,t}$ with $\bz:= \|\frac{1}{\tmaxpca} \sum_t \M_{2,t} \M_{2,t}{}'\|$ and $ q:=  \max_t \|\M_{1,t} \P\|$ with $q < 2$. Thus,
\begin{align}
&\|\E[\noise]\| \leq  \norm{\frac{1}{\tmaxpca} \sum_t \M_t \P \Lam \P{}'\M_{1,t}{}' \M_{2,t}{}'} + \norm{\Sigma_{\bv}} \nn \\
& \le  \sqrt{\norm{\frac{1}{\tmaxpca} \sum_t \M_t \P \Lam \P{}'\M_{1,t}{}'(\cdot){}'}  \norm{\frac{1}{\tmaxpca} \sum_t \M_{2,t} \M_{2,t}{}'} } + \lambda_v^+ \nn \\
& \le  \sqrt{\max_t \|\M_t\P \Lam \P{}'\M_{1,t}{}' \|^2  \ \bz} + \lambda_v^+ \le  \sqrt{\bz} q^2 \lambda^+ + \lambda_v^+
\label{bnd_avg_sig_noise_cor_bnd_2}
\end{align}
Similarly,
\begin{align}
&\|\E[\cross_{\l,\wt}]\| = \norm{\frac{1}{\tmaxpca} \sum_t \M_{2,t} \M_{1,t} \P \Lam \P{}'} \nn \\
& \le \sqrt{ \norm{\frac{1}{\tmaxpca} \sum_t \P \Lam \P{}'\M_{1,t}{}'\M_{1,t}\P \Lam \P{}'}   \norm{\frac{1}{\tmaxpca} \sum_t \M_{2,t} \M_{2,t}{}'} } \nn \\
& \le \sqrt{ \max_t \|\M_{1,t}\P \Lam \P{}'\|^2 \ \bz } \le  \sqrt{\bz} q \lambda^+
\label{bnd_avg_sig_noise_cor_bnd_1}
\end{align}

Since $\bv$ is uncorrelated noise, $\E[\cross_{\l,\bv}] = 0$ and $\E[\cross_{\wt,\bv}] = 0$.
We now lower bound $\lambda_{\min}\left(\frac{1}{\tmaxpca}\sum_t \at \at{}'\right)$ as
\begin{align*}
\lambda_{\min}\left(\frac{1}{\tmaxpca}\sum_t \at \at{}' \right) &= \lambda_{\min}\left( \Lam - \left(\frac{1}{\tmaxpca}\sum_t \at \at{}' - \Lam\right) \right) \\
 &\geq \lambda_{\min}(\Lam) -  \lambda_{\max}\left(\frac{1}{\tmaxpca}\sum_t \at \at{}'- \Lam \right) \\
&\geq \lambda^- - \norm{\frac{1}{\tmaxpca}\sum_t \at \at{}' - \Lam}
\end{align*}
and thus we have
\begin{align*}
&\SE(\Phat, \P) \\
&\leq \frac{4 \sqrt{\bz}q \lambda^+ + \lambda_v^+ + 2\|\cross - \E[\cross]\| + \|\noise - \E[\noise]\|}{\lambda^- - \norm{\frac{1}{\tmaxpca}\sum_t\at \at{}' - \Lam} - \mathrm{numerator}}
\end{align*}
\subsubsection{Concentration bounds}
Now we only need to bound $\|\noise - \E[\noise]\|$ and $\|\cross - \E[\cross]\|$. These are often referred to as  ``statistical error'', while the error due to nonzero  $\| \E[\cross]\|$ or $\| \E[\noise]\|$ is called the ``bias''.  We use concentration bounds from Lemma \ref{hp_bnds}. 
\begin{align*}
& \|\noise - \E[\noise]\| + 2 \|\cross - \E[\cross]\| \\
 \leq & \norm{\frac{1}{\tmaxpca}\sum_t(\wt \wt{}' - \E[\wt\wt{}'])]} + \norm{\frac{1}{\tmaxpca}\sum_t(\bv \bv{}' - \E[\bv \bv{}'])]} \\
 &+ 2\norm{\frac{1}{\tmaxpca}\sum_t(\lt \wt{}' - \E[\lt\wt{}'])]} \\
& + 2\norm{\frac{1}{\tmaxpca}\sum_t \lt \bv{}'} + 2\norm{\frac{1}{\tmaxpca}\sum_t \wt \bv{}'} \\
 \leq &
C\sqrt{\eta} q^2  f  \sqrt{\frac{r \log n}{\tmaxpca}}   \lambda^- +
C \sqrt{\eta}  \frac{\lambda_v^+}{\lambda^-}  \sqrt{\frac{r \log n}{\tmaxpca}}   \lambda^-  \\
&+
C \sqrt{\eta} q f  \sqrt{\frac{r \log n}{\tmaxpca}}     \lambda^- \\
& +  C   \sqrt{\eta}  \sqrt{\frac{\lambda_v^+}{\lambda^-} f} \sqrt{\frac{r \log n}{\tmaxpca}} \lambda^-  +  C   \sqrt{\eta} q \sqrt{\frac{\lambda_v^+}{\lambda^-} f}  \sqrt{\frac{r \log n}{\tmaxpca}} \lambda^- \\
  \leq & C   \sqrt{\eta}   q f \sqrt{\frac{r \log n}{\tmaxpca}} \lambda^-   +  C   \sqrt{\eta}  \sqrt{\frac{\lambda_v^+}{\lambda^-} f}  \sqrt{\frac{r \log n}{\tmaxpca}} \lambda^-
:= \epsbnd\lambda^-
\end{align*}
where the last line follows from using $q^2 < 2q $ and $\lambda_v^+ \leq \lambda^+$.

In case we only need to bound $\|\noise - \E[\noise]\|$, we can get a tighter bound that contains only the first two terms and not all five. Clearly, we have
\begin{align*}
& \|\noise - \E[\noise]\| \\
& \le C\sqrt{\eta} q^2  f  \sqrt{\frac{r \log n}{\tmaxpca}}   \lambda^- +
C \sqrt{\eta}  \frac{\lambda_v^+}{\lambda^-}  \sqrt{\frac{r \log n}{\tmaxpca}}   \lambda^-  +\\
& : = H_{noise}(\alpha)
\end{align*}

The bound on $\|\frac{1}{\tmaxpca} \sum_t \at \at{}' - \Lam\|_2$ follows directly from the first item of Lemma \ref{hp_bnds}.
\end{proof}

\subsection{A useful corollary that follows from above proof}
From the above proof, we can write out a bound for   $\| \bphi\mathrm{cross} \bphi' \| $  for a projection matrix $\bphi$ by noticing that
each term of $\cross $ is of the form $\sum_t \lt (.)' = \P \sum_t \at (.)'$. Thus $\| \P' \mathrm{cross} \|  = \|\cross\|$. Thus,
$\| \bphi\mathrm{cross} \bphi' \| \le  \|\bphi \P\| \|\mathrm{cross}\| \le \|\bphi \P\| ( \|\E[\mathrm{cross}]\|  +  \|\cross - \E[\cross]\| )$.
Similarly, we can also get a bound on $\lambda_{\max}( \mathrm{noise} ) = \|\mathrm{noise}\|$.

Assume $b = 0.01/f^2$, $ q> \zz > \sqrt{g}$.
Consider $\mathrm{cross}$. If $\alpha \ge C \max\left( \frac{q^2 f^2}{\epsilon_1^2} r \log n, \frac{g f}{\epsilon_1^2} \max(r_v,r) \log n\right)$, then $\epsbnd \le \epsilon_1 \lambda^-$.
If we set $\epsilon_1 = 0.002 \max(  \sqrt{b} q, \sqrt{b} \zz)$ and $b = 0.01/f^2$ (bound on $\outfracrow(\alpha)$), then, since $\zz > \sqrt{g}$ , $\alpha = C f^2 \max(r_v,r) \log n$ suffices. Since $q \ge \zz$, then, $\epsilon_1 = 0.002 \sqrt{b} q$. Thus,
\[
 \| \bphi\mathrm{cross} \bphi' \|   \le  \|\bphi \P\| (2 \sqrt{b} q \lambda^+  + \epsbnd \lambda^-)  \le 2.02 \sqrt{b} \|\bphi \P\| q \lambda^+
\]

Consider $\noise$.  We will use $H_{noise}(\alpha)$ for this. If $\alpha > C \max\left( \frac{q^4 f^2}{\epsilon_2^2} r \log n, \frac{g^2}{\epsilon_2^2}  \max(r_v,r) \log n\right)$, then  $H_{noise}(\alpha) \le \epsilon_2 \lambda^-$. If we set $\epsilon_2 = 0.002 \sqrt{b} \max(   q^2, \zz^2)$, then since $\zz^4 > g^2$, thus, $\alpha = C f^2 \max(r_v,r) \log n$ suffices. Since $q> \zz$, $\epsilon_2 = 0.002 \sqrt{b}  q^2$. We have the following corollary.


\begin{corollary} \label{pca_corol}
If $\alpha = C f^2 \max(r_v,r) \log n$, and if $ q \ge \zz > \sqrt{g}$, then, w.p. $1 - 10 n^{-10}$,
\begin{align*}
 \| \bphi\mathrm{cross} \bphi' \|  & \le  \|\bphi \P\| (2 \sqrt{b} q \lambda^+  + \epsbnd \lambda^-)  \\
 &\le 2.02 \sqrt{b} \|\bphi \P\| q \lambda^+, 
 \\
 \lambda_{\max}( \bphi \mathrm{noise} \bphi)  \le \|\noise\|
 & \le  \sqrt{b} q^2 \lambda^+  + \lambda_v^+  + \epsbnd \lambda^- \\ &\le 1.01 \sqrt{b} q^2 \lambda^+  + \lambda_v^+ \\
& \le 1.01 \sqrt{b} q^2 \lambda^+  + \zz^2 \lambda^-
 \end{align*}
\end{corollary}

%

\subsection{Main idea of the proof of Corollary \ref{pca_ssch}} \label{proof_pca_ssch_idea}
The key difference in this proof is our choice of $\D_0$. Since we want to bound $\SE(\Phat,\P)$, we need to pick it in such a way that its matrix of top $r$ singular vectors equals $\Span(\P)$. We pick
\[
\D_0 = \frac{1}{\alpha}  \P \left(  (\alpha-\alpha_0) \Lam +  \alpha_0 \P' \P_0  \Lam \P_0' \P \right) \P'
\]
Clearly, $\lambda_{r+1}(\D_0) = 0$. With this choice of $\D_0$, 
\begin{align*}
\D - \D_0 &= \cross + \cross{}' +  \noise  \\
&+ \left(\frac{1}{\alpha} \sum_t \lt \lt'  - \E[\frac{1}{\alpha} \sum_t \lt \lt'] \right) +  \left( \E[\frac{1}{\alpha} \sum_t \lt \lt'] - \D_0 \right)
\end{align*}
where $\cross, \noise$ are as defined earlier with the change that $\lt$ is now defined differently. Thus, the only thing that changes when bounding these is our definition of $q$. The last term in the expression above equals $c_0 \P_\perp \P_\perp' \P_0  \Lam \P_0' \P_\perp \P_\perp{}'  + c_0 \P_\perp \P_\perp' \P_0  \Lam \P_0' \P \P' + (.)'$ with $c_0: = \frac{\alpha_0}{\alpha}$. This is what  generates the extra $4 \Delta f$ term in our $\SE$ bound.
A complete proof is provided in the Supplement given in the ArXiv version of this work \cite[Appendix IV-A]{jsait_arxiv}).

\subsection{Concentration Bounds}
We state the lemma below so that it can also be used in proving the most general PCA result given in the Supplement given in the ArXiv version of this work \cite[Appendix IV-A]{jsait_arxiv}).
Let
$\Lam_t = \E[\at \at{}']$, $\Lamb = \frac{1}{\tmaxpca}\sum_t \Lam_t$, $\lambda_{\max}^+ := \max_t \|\Lam_t\|$, $\lambda^-_{\avg} = \lambda_{\min}(\Lamb)$, $f = \lambda_{max}^+/\lambda^-_{\avg}$, $\lambda_{v,\max}^+ := \max_t \|\E[\vp_t\vp_t{}']\|$ and $\nois = \lambda_{v,\max}^+/\lambda^-_{\avg}$.

To use the lemma under the simpler i.i.d. assumption used in the main paper, remove the $_{\max}, _{\avg}$ subscripts from all terms, e.g., replace $\lambda_{\max}^+$ by $\lambda^+$, $\lambda^-_{\avg}$ by $\lambda^-$ and so on.

\begin{lem}\label{hp_bnds}
With probability at least $1 - 10n^{-10}$, 
\begin{align*}
&  \left\|\frac{1}{\tmaxpca}\sum_t \at \at{}' -  \Lamb \right\| \le    C \sqrt{\eta} f \sqrt{\frac{r \log n}{\tmaxpca}}  \lambda^-_{\avg}, \\
& \left\|\frac{1}{\tmaxpca}\sum_t \lt \wt{}' - \frac{1}{\tmaxpca}\E\left[\sum_t \lt \wt{}'\right] \right\|_2 \le C\sqrt{\eta} q f  \sqrt{\frac{r \log n}{\tmaxpca}}     \lambda^-_{\avg}, \\ 
& \left\|\frac{1}{\tmaxpca}\sum_t \wt \wt{}' - \frac{1}{\tmaxpca}\E\left[\sum_t \wt \wt{}'\right] \right\|_2 \le  C\sqrt{\eta} q^2  f  \sqrt{\frac{r \log n}{\tmaxpca}}   \lambda^-_{\avg}, \\ 
& \left\|\frac{1}{{\tmaxpca}} \sum_t \lt {\vp_t}' \right\|_2 \le  C\sqrt{\eta} \sqrt{\nois f} \sqrt{\frac{ \max(r_v, r) \log n}{\tmaxpca}} \lambda^-_{\avg}, \\
& \left\|\frac{1}{{\tmaxpca}} \sum_t \wt {\vp_t}' \right\|_2 \le  C\sqrt{\eta} q \sqrt{\nois f} \sqrt{\frac{ \max(r_v, r) \log n}{\tmaxpca}} \lambda^-_{\avg}, \\
& \left\| \frac{1}{{\tmaxpca}} \sum_t  \vp_t {\vp_t}' - \frac{1}{{\tmaxpca}} \E\left[ \sum_t  \vp_t {\vp_t}' \right] \right\|_2 \le  C\sqrt{\eta}  \nois \sqrt{\frac{r_v \log n}{\tmaxpca}} \lambda^-_{\avg}. \\
\end{align*}
\end{lem}


\begin{proof}[Proof of Lemma \ref{hp_bnds}]  \
%
%
{\em $\at \at{}'$ term. } This and all other items use Matrix Bernstein for rectangular matrices, Theorem 1.6 of \cite{tail_bound}. This says the following. For a finite sequence of $d_1 \times d_2$ zero mean independent matrices $\Z_k$ with
\begin{align*}
& \|\Z_k\|_2 \le R, \text{ and}  \\
&\max(\|\sum_k \E[ \Z_k{}'\Z_k ]\|_2, \|\sum_k \E[ \Z_k\Z_k{}' ]\|_2) \le \sigma^2, 
\end{align*}
we have $\Pr(\|\sum_k \Z_k\|_2 \ge s ) \le (d_1+d_2) \exp\left(- \frac{s^2/2}{\sigma^2 + Rs / 3}\right) \leq (d_1+d_2) \exp\left(- c\min\left(\frac{s^2}{2\sigma^2}, \frac{3s}{2R}\right)\right)$. Let $\tilde{\Z}_t := \at \at{}'$ and we apply the above result to $\Z_t = \tilde{\Z}_t - \E[\tilde{\Z}_t]$. with $s = \epsilon \tmaxpca$. Now it is easy to see that $\|\Z_t\| \leq 2\|\at\at{}'\| \leq 2 \|\at\|_2^2 \leq 2 \eta r \lambda_{\max}^+ := R$ and similarly, $\|\E[\Z_t^2]\| = \|\E[\|\at\|_2^2 \at\at{}']\|\leq \tmaxpca \cdot  \max_{\at} \|\at\|_2^2 \cdot \max_t \E[\at\at{}'] \leq \tmaxpca \eta r (\lambda_{\max}^+)^2 := \sigma^2 $ and thus, w.p. at most $2 r \exp\left( - c \min\left( \frac{\epsilon^2 \tmaxpca}{r (\lambda_{\max}^+)^2}, \frac{\epsilon^2\tmax}{r \lambda_{\max}^+ \epsilon} \right) \right)$. Now we set $\epsilon = \epsilon_5 \lambda_{\min}^-$ with $\epsilon_5 = C \sqrt\eta f \sqrt{\frac{r\log n}{\tmaxpca}}$ to get our result.

{\em $\lt \wt'$ term. }
%
Let $\Z_t := \lt \wt{}'$. We apply this result to $\tilde{\Z}_t:= \Z_t - \E[\Z_t]$ with $s = \epsilon \tmaxpca$. To get the values of $R$ and $\sigma^2$ in a simple fashion, we use the facts that (i)
if $\|\Z_t\|_2 \le R_1$, then $\|\tilde{\Z}_t\| \le 2R_1$; and (ii) $\sum_t \E[\tilde{\Z}_t \tilde{\Z}_t{}'] \preccurlyeq \sum_t \E[\Z_t \Z_t{}']$. 
Thus, we can set $R$ to two times the bound on $\|\Z_t\|_2$ and  we can set $\sigma^2$ as the maximum of the bounds on $\|\sum_t \E[\Z_t \Z_t{}']\|_2$ and $\|\sum_t \E[\Z_t{}' \Z_t]\|_2$.

It is easy to see that  $R = 2 \sqrt{\eta  r \lambda^+_{\max}} \sqrt{\eta r  q^2 \lambda^+_{\max}} = 2 \eta r q \lambda^+_{\max}$. To get $\sigma^2$, observe that
\begin{align*}
\left\| \sum_t \E[\wt \lt{}' \lt \wt{}'] \right\|_2 & \le \tmaxpca (\max_{\lt} \|\lt\|^2) \cdot \max_t  \|\E[\wt \wt{}']\| \\
& \le \tmaxpca \eta  r \lambda^+_{\max} \cdot  q^2 \lambda^+  = \tmaxpca \eta r q^2 (\lambda^+_{\max})^2.
\end{align*}
Repeating the above steps, we get the same bound on $\| \sum_t \E[\Z_t \Z_t{}'] \|_2$. Thus,  $\sigma^2 = \tmaxpca \eta r q^2 (\lambda^+_{\max})^2$.

Thus, we conclude that,
$
\left\|\sum_t \lt \wt{}' - \E[\sum_t \lt \wt{}'] \right\|_2 \ge \epsilon \tmaxpca
$
w.p. at most
$2n \exp\left(- c \min \left(\frac{\epsilon^2 \tmaxpca}{\eta r q^2 (\lambda^+_{\max})^2}, \frac{\epsilon \tmaxpca}{\eta r q \lambda^+_{\max}}\right)\right)$
Set $\epsilon = \epsilon_0 \lambda^-$ with $\epsilon_0 = c q f \sqrt{\frac{r \log n}{\tmaxpca}}$ so that our bound holds w.p. at most $2n^{-10}$. This follows because $\tmaxpca \geq C f^2 r \log n$.

%
%

{\em $\wt \wt{}'$ , $\lt \vp_t'$, $\wt \vp_t'$ and $\vp_t \vp_t'$ terms. } Apply matrix Bernstein as done above.
\end{proof}

\section{Proof of Theorem \ref{thm1} and Corollary \ref{cor:thm1}} \label{proof_of_thm1}

\begin{proof}[Proof of Theorem \ref{thm1}]
The overall structure of this proof is similar to that in \cite{rrpcp_isit15,rrpcp_aistats}.
%
Define
\[
\that_{j-1,fin}: =  \that_{j-1} + K \alpha, \  t_{j,*}=\that_{j-1,fin} + \left \lceil \frac{t_j - \that_{j-1,fin}}{\alpha} \right \rceil \alpha
\]
Thus, $\that_{j-1,fin}$ is the time at which the $(j-1)$-th subspace update is complete; w.h.p., this occurs before $t_j$. With this assumption, $t_{j,*}$ is such that $t_j$ lies in the interval $[t_{j,*}-\alpha+1,t_{j,*}]$.

\noindent Recall from the algorithm that we increment $j$ to $j+1$ at $t= \that_j+K\alpha:= \that_{j,fin}$.
Define the events
\ben
\item $\mathrm{Det0}:= \{\that_j = t_{j,*} \} = \{\lambda_{\max}(\frac{1}{\alpha} \sum_{t= t_{j,*}-\alpha+1}^{t_{j,*}} \bphi \lhat_t \lhat_t'\bphi) > \lthres\}$ and
\\ $\mathrm{Det1}:= \{\that_j = t_{j,*} + \alpha\} = \{\lambda_{\max}(\frac{1}{\alpha} \sum_{t= t_{j,*} +1}^{t_{j,*}+\alpha} \bphi \lhat_t \lhat_t'\bphi) > \lthres\} $,
\item $\mathrm{\subup}:=\cap_{k=1}^K  \mathrm{\subup}_k$ where $\mathrm{\subup}_k:= \{\SE(\Phat_{j,k}, \P_{j}) \le q_k\}$,

\item $\mathrm{NoFalseDets}:= \{\text{for all $\J^\alpha \subseteq [\that_{j,fin}, t_{j+1})$, } \\ \lambda_{\max}(\frac{1}{\alpha} \sum_{t \in \J^\alpha} \bphi \lhat_t \lhat_t'\bphi) \le \lthres\}$
\item $\Gamma_{0,\ed}:= \{\SE(\Phat_0, \P_0) \le 0.25 \}$,
\item  $\Gamma_{j,\ed}:= \Gamma_{j-1,\ed} \cap
\big( (\mathrm{Det0} \cap \mathrm{\subup} \cap \mathrm{NoFalseDets}) \cup
(\overline{\mathrm{Det0}} \cap \mathrm{Det1} \cap \mathrm{\subup} \cap \mathrm{NoFalseDets}) \big)$.
\een

Let $p_0$ denote the probability that, conditioned on $\Gamma_{j-1,\ed}$, the change got detected at $t=t_{j,*}$, i.e., let
\[
p_0:= \Pr(\mathrm{Det0}|\Gamma_{j-1,\ed}).
\]
Thus, $\Pr(\overline{\mathrm{Det0}}|\Gamma_{j-1,\ed}) = 1- p_0$. It is not easy to bound $p_0$. However, as we will see, this will not be needed. Assume that $\Gamma_{j-1,\ed} \cap \overline{\mathrm{Det0}}$ holds. Consider the interval $\J^\alpha: = [t_{j,*}, t_{j,*}+\alpha)$. This interval starts at or after $t_j$, so, for all $t$ in this interval, the subspace has changed. For this interval, $\bphi = \I - \Phat_{j-1} \Phat_{j-1}{}'$. 
Applying the first item of Lemma \ref{lem:sschangedet}, w.p. at least $1-10n^{-10}$,
\[
\lambda_{\max} \left(\frac{1}{\alpha} \sum_{t \in \J^\alpha} \bphi \lhat_t \lhat_t'\bphi \right) \geq \lthres
\]
and thus $\that_j = t_{j,*} + \alpha$.
In other words,
\[
\Pr(\mathrm{Det1} | \Gamma_{j-1,\ed} \cap \overline{\mathrm{Det0}}) \ge 1 - 10n^{-10}.
\]

Conditioned on $\Gamma_{j-1,\ed} \cap \overline{\mathrm{Det0}} \cap \mathrm{Det1}$, the first SVD step is done at $t= \that_j + \alpha = t_{j,*} + 2\alpha$ and the subsequent steps are done every $\alpha$ samples. We can prove Lemma \ref{lem:reprocspcalemone} with $\Gamma_{j,0}$ replaced by $\Gamma_{j,\ed} \cap \overline{\mathrm{Det0}} \cap \mathrm{Det1}$ and Lemma \ref{lem:reprocspcalemk} with $\Gamma_{j,k-1}$ replaced by $\Gamma_{j,\ed} \cap \overline{\mathrm{Det0}} \cap \mathrm{Det1} \cap \subup_1 \cap \cdots \cap \subup_{k-1}$ and with the $k$-th SVD interval being $\J_k:=[\that_j+(k-1)\alpha, \that_j + k \alpha)$. Applying Lemmas \ref{lem:reprocspcalemone}, and \ref{lem:reprocspcalemk} for each $k$, we get
\[
\Pr(\subup |\Gamma_{j-1,\ed} \cap \overline{\mathrm{Det0}} \cap \mathrm{Det1}) \ge ( 1 - 10n^{-10})^{K+1}.
\]
We can also do a similar thing for the case when the  change is detected at $t_{j,*}$, i.e. when $\mathrm{Det0}$ holds. In this case, we replace $\Gamma_{j, 0}$ by $\Gamma_{j,\ed} \cap \mathrm{Det0}$ and $\Gamma_{j, k}$ by $\Gamma_{j,\ed} \cap \mathrm{Det0} \cap \subup_1 \cap \cdots \cap \subup_{k-1}$ and conclude that
\[
\Pr(\subup|\Gamma_{j-1,\ed} \cap \mathrm{Det0}) \ge ( 1 - 10n^{-10})^{K}.
\]

Finally consider the $\mathrm{NoFalseDets}$ event. First, assume that $\Gamma_{j-1,\ed} \cap \mathrm{Det0} \cap \subup$ holds.  Consider any interval $\J^\alpha \subseteq [\that_{j,fin}, t_{j+1})$. In this interval, $\Phat_{(t)} = \Phat_j$, $\bphi = \I -  \Phat_j \Phat_j{}'$ and $\SE(\Phat_j,\P_j) \le \zz$. Using the second part of Lemma \ref{lem:sschangedet} we conclude that w.p. at least $1- 10n^{-10}$,
\[
\lambda_{\max} \left(\frac{1}{\alpha} \sum_{t \in \J^\alpha} \bphi \lhat_t \lhat_t'\bphi \right)  < \lthres
\]
Since $\mathrm{Det0}$ holds, $\that_j = t_{j,*}$.
Thus, we have a total of $\lfloor \frac{t_{j+1} - t_{j,*} - K \alpha - \alphadel}{\alpha} \rfloor$ intervals $\J^\alpha$ that are subsets of $[\that_{j,fin}, t_{j+1})$. Moreover, $\lfloor \frac{t_{j+1} - t_{j,*} - K \alpha - \alphadel}{\alpha} \rfloor \le \lfloor \frac{t_{j+1} - t_j - K \alpha - \alphadel}{\alpha} \rfloor \le \lfloor \frac{t_{j+1} - t_j}{\alpha} \rfloor - (K+1)$ since $\alpha \le \alphadel$.
Thus,
\begin{align*}
\Pr(\mathrm{NoFalseDets} | \Gamma_{j-1,\ed} \cap \mathrm{Det0} \cap \subup) \\ \ge (1 - 10n^{-10})^{\lfloor \frac{t_{j+1} - t_j}{\alpha} \rfloor - (K)}
\end{align*}
On the other hand, if we condition on $\Gamma_{j-1,\ed} \cap \overline{\mathrm{Det0}} \cap \mathrm{Det1} \cap \subup$, then $\that_j = t_{j,*} + \alpha$. Thus,
\begin{align*}
\Pr(\mathrm{NoFalseDets} | \Gamma_{j-1,\ed} \cap \overline{\mathrm{Det0}} \cap \mathrm{Det1} \cap \subup) \\
\ge (1 - 10n^{-10})^{\lfloor \frac{t_{j+1} - t_j}{\alpha} \rfloor - (K+1)}
\end{align*}
We can now combine the above facts to bound $\Pr(\Gamma_{j,\ed}|\Gamma_{j-1,\ed})$. Recall that $p_0:= \Pr(\mathrm{Det0}|\Gamma_{j-1,\ed})$.
Clearly, the events $(\mathrm{Det0} \cap \subup \cap \mathrm{NoFalseDets})$ and $(\overline{\mathrm{Det0}} \cap \mathrm{Det1} \cap \subup \cap \mathrm{NoFalseDets})$ are disjoint. Thus,
\begin{align*}
& \Pr(\Gamma_{j,\ed}|\Gamma_{j-1,\ed}) \\
& = p_0 \Pr(\subup \cap \mathrm{NoFalseDets} |\Gamma_{j-1,\ed} \cap \mathrm{Det0})  \\
& + (1-p_0) \Pr(\mathrm{Det1}|\Gamma_{j-1,\ed} \cap \overline{\mathrm{Det0}}) \cdot \\
& \Pr(\subup \cap \mathrm{NoFalseDets} |\Gamma_{j-1,\ed}\cap \overline{\mathrm{Det0}} \cap \mathrm{Det1}) \\
& \ge p_0 ( 1 - 10n^{-10})^{K} (1 - 10n^{-10})^{\lfloor \frac{t_{j+1} - t_j}{\alpha} \rfloor - (K)} \\
& + (1-p_0) ( 1 - 10n^{-10}) \cdot \\& ( 1 - 10n^{-10})^{K}  (1 - 10n^{-10})^{\lfloor \frac{t_{j+1} - t_j }{\alpha} \rfloor - (K+1)}  \\
& =  ( 1 - 10n^{-10})^{\lfloor \frac{t_{j+1} - t_j}{\alpha} \rfloor}
\ge ( 1 - 10n^{-10})^{t_{j+1}-t_j}.
\end{align*}
Thus, since the events $\Gamma_{j,\ed}$ are nested,
$
\Pr(\Gamma_{J,\ed}|\Gamma_{0,\ed}) = \prod_j \Pr(\Gamma_{j,\ed}|\Gamma_{j-1,\ed}) \ge \prod_j ( 1 - 10n^{-10})^{t_{j+1}-t_j} = ( 1 - 10n^{-10})^d
\ge  1 - 10d n^{-10}.
$
\end{proof}


\begin{proof}[Proof of  Corollary \ref{cor:thm1}]
%
It should be noted that $basis(\M)$ is not a unique matrix, it refers to any matrix $\P$ that has orthonormal columns and whose span equals the span of $\M$. Thus $basis([\Phat_{j-1},\Phat_j]) \equiv basis( [ \Phat_{j-1}, \Phat_{j-1,\perp} \Phat_j]) \equiv basis( [ \Phat_{j}, \Phat_{j,\perp} \Phat_{j-1}])$. Let us denote any of these matrices by $\Phat_{j-1,j}$.

For $t \in [\that_{j-1} + K\alpha, t_{j})$,  $\Pt = \P_{j-1}$ while for $t \in [t_j, \that_j + K\alpha -1 )$, $\Pt = \P_j$. For all $t$ in these two intervals $\Phat_{(t)} = \Phat_{j-1,j}$.
The proof of this corollary is an easy consequence of this fact and the fact that, for two basis matrices $\P_1, \P_2$ that are mutually orthonormal, i.e., for which $\P_1{}' \P_2 = 0$,
\[
(\I - \P_1 \P_1' - \P_2 \P_2') = (\I - \P_1\P_1') (\I - \P_2\P_2').
\]
Thus, $\SE(\Phat_{j-1,j}, \P_{j-1}) \le \SE(\Phat_{j-1}, \P_{j-1}) \le \zz$ and $\SE(\Phat_{j-1,j}, \P_j) \le \SE(\Phat_j, \P_j) \le \zz$.
\end{proof}

\section{Proofs for Section \ref{sec:norst}: Time complexity derivation and Proof of Theorem \ref{thm1_newnorst}}\label{sec:proof}

\subsection{Time complexity derivation} \label{timecomp_derive}
Consider initialization.
To ensure that $\SE(\Phat_0, \P_0) \in O(1/\sqrt{r})$, we need to use $C \log r$ iterations of AltProj. Since there is no lower bound in the AltProj guarantee on the required number of matrix columns (except the trivial lower bound of rank) \cite{robpca_nonconvex}, we can use $t_\train = C r$ frames for initialization. Thus the initialization complexity is $O(n t_\train r^2 \log(\sqrt{r}) = O(n r^3 \log r)$ \cite{robpca_nonconvex}. The projected-CS step complexity is equal to the cost of a matrix vector multiplication with the measurement matrix times negative logarithm of the desired accuracy in solving the $l_1$ minimization problem. Since the measurement matrix for the CS step is $\I - \Phat_{(t-1)} \Phat_{(t-1)}{}'$, the cost per CS step (per frame) is $O(nr\log(1/\epsilon))$ \cite{l1_best} and so the total cost is $O((d-t_\train) nr\log(1/\epsilon))$.
 The subspace update involves at most $((d - t_\train)/ \alpha)$ rank $r$-SVD's on $n\times \alpha$ matrices all of which have constant eigen-gap (this is indirectly proved in the proofs of the second item of Lemmas \ref{lem:reprocspcalemone} and \ref{lem:reprocspcalemk}). Thus the total time for subspace update steps is at most $((d-t_\train)/\alpha)*O(n \alpha r \log(1/\epsilon)) = O((d-t_\train) n r \log(1/\epsilon))$ \cite{musco2015randomized}. Thus the running time of the complete algorithm is $O(ndr\log(1/\epsilon) + nr^3 \log r)$. As long as $r^2 \log r \le d \log(1/\epsilon)$, the time complexity of the entire algorithm is $O(n d r\log(1/\epsilon))$.


\subsection{Proof of Theorem \ref{thm1_newnorst} for NORST-NoDet}
In this algorithm we do not detect change. We just keep updating the subspace by $r$-SVD applied every $\alpha$ time instants on the last $\alpha$ $\lhat_t$'s, $\Lhat_{t;\alpha}$. For $\alpha$-intervals $\J$ for which $\Pt = \P_j$ for all $t \in \J$, there is no change to the analysis. 
We start at $t = t_0 = \that_0 = 1$ with initial subspace estimate $\Phat_0$ available. Let $\Delta_0 = \SE(\Phat_0, \P_0)$.
The first subspace update is done at $t = \alpha$, the second at $t=2\alpha$, and so on. By Lemma \ref{lem:reprocspcalemone} with $\Phat_{j,0} =\Phat_0$, we can show that after one update, the error reduces to $ 1.2 \max(\Delta_0/4, \zz )$. After this, by applying Lemma \ref{lem:reprocspcalemk} $K-1$ times, we can show that, after at most $K$ steps with $K = \log (\Delta_0/\zz)$, the error reduces to $1.2 \zz$.  Beyond this time, the error does not decrease further.
We know that $\Pt = \P_0$ for $t \in [t_0, (K+2)\alpha]$, but can change after that. 

Consider the $\alpha$-interval  $\J$ that contains the change time $t_1$.  The projected CS analysis for this interval remains exactly the same as above. But to analyze the subspace update for this interval we need to  use Corollary \ref{pca_ssch}. More generally consider the $j$-th change, and the interval $\J = [\lfloor t_j/\alpha \rfloor+1, \lfloor t_j/\alpha\rfloor + \alpha]$, which is the $\alpha$-frame interval that contains $t_j$.

For $t \in \J$, we have $\lhat_t = \yt - \xhat_t = \lt +\et + \vt$ where
\begin{align*}
\et &= \bm{I}_{\Tt}\left(\bpsi_{\Tt}{}'\bpsi_{\Tt}\right)^{-1} \I_{\Tt}{}' \bpsi(\lt + \vt):= (\e_{\l})_t + (\e_{\v})_t
\end{align*}
$\bpsi = \I - \Phat_{j-1}\Phat_{j-1}{}'$,  $\lt = \P_{j-1} \at $ for $t \in [\lfloor t_j/\alpha\rfloor, t_j)$ and $\lt = \P_{j} \at $ for $t \in [t_j, \lfloor t_j/\alpha \rfloor +\alpha)$. 

Let $\Phat_{j,0}$ denote the subspace estimate $\Phat_{(t)}$ computed for this interval.
We apply Corollary \ref{pca_ssch} with $\yt \equiv \lhat_t$, $\wt \equiv (\e_{\l})_t$, $\vp_t \equiv (\e_{\v})_t + \vt$, $\lt \equiv \lt$, $\M_{1,t} = -\left( \bpsi_{\T_t}{}' \bpsi_{\T_t}\right)^{-1} \bpsi_{\T_t}{}'$, $\Phat = \Phat_{j,0}$, $\P = \P_j$, $\P_0 = \P_{j-1}$.
Since
$\norm{\M_{1, t} \P_0} = \| \left( \bpsi_{\T_t}{}' \bpsi_{\T_t}\right)^{-1} \bpsi_{\T_t}{}' \P_j\| \leq 1.2 \zz $, $\norm{\M_{1, t} \P} = \| \left( \bpsi_{\T_t}{}' \bpsi_{\T_t}\right)^{-1} \bpsi_{\T_t}{}' \P_j\| \leq 1.2 (\zz + \SE(\P_{j-1}, \P_j))$, thus $q_{00} = 1.2 (\zz + \SE(\P_{j-1}, \P_j))$.
Also, $\bz \equiv b_0= 0.01/f^2$ which is the upper bound on $\outfracrow(\alpha)$, $\|\E[(\e_{\v})_t (\e_{\v})_t{}']\| \leq (1.2)^2 \lambda_v^+$.
Thus, with probability at least $1 - 10n^{-10}$,
\begin{align*}
\SE(\Phat_{j,0},\P_j) \le 2.5 (3 (\Delta f + 4 \cdot 0.1  \cdot 1.2 (\zz + \Delta) + \frac{\lambda_{v}^+ }{ \lambda^-} ) \le 10 \Delta
\end{align*}
Here we used $\frac{\lambda_{v}^+ }{ \lambda^-} = \zz^2 < \Delta$. 

Redefine $\that_j = \lfloor t_j/\alpha \rfloor + \alpha$ and $\Phat_{j,0}$ to denote the estimate from the change interval.
To analyze the next $\alpha$-interval for new-NORST, we apply Lemma \ref{lem:reprocspcalemone} with above re-definitions. Thus, $q_0 = 1.2 \cdot 10 \Delta$. We can conclude that $\SE(\Phat_{j,1}, \P_j) \le \max(0.3 q_0,\zz) = q_1$. For the next $K-1$ intervals, we apply  Lemma \ref{lem:reprocspcalemk} $K-1$ times with $q_{k} =1.2 \max(0.25 q_{k-1},\zz)$.

\section*{Author Biographies}

\begin{IEEEbiography}[]{Praneeth Narayanamurthy} (Email: pkurpadn@iastate.edu) (S' 18) is a Ph.D. student in the Department of Electrical and Computer Engineering at Iowa State University (Ames, IA). He previously obtained his B.Tech degree in Electrical and Electronics Engineering from National Institute of Technology Karnataka (Surathkal, India) in 2014. His research interests include the algorithmic and theoretical aspects of High-Dimensional Statistical Signal Processing, and Machine Learning.
\end{IEEEbiography}

\begin{IEEEbiography}[]{Namrata Vaswani}(Email: namrata@iastate.edu)
received a B.Tech from the Indian Institute of Technology (IIT), Delhi, in 1999 and a Ph.D. from UMD in 2004, both in electrical engineering. Since Fall 2005, she has been with the Iowa State University where she is currently the Anderlik Professor of Electrical and Computer Engineering. Her research interests lie in a data science, with a particular focus on Statistical Machine Learning, Signal Processing, and Computer Vision. She has served two terms as an Associate Editor for the IEEE Transactions on Signal Processing; as a lead guest-editor for a Proceedings of the IEEE Special Issue (Rethinking PCA for modern datasets); and is currently serving as an Area Editor for the  IEEE Signal Processing Magazine. Vaswani is a recipient of the Iowa State Early Career Engineering Faculty Research Award (2014), the Iowa State University Mid-Career Achievement in Research Award (2019), University of Maryland's ECE Distinguished Alumni Award (2019), as well as the 2014 IEEE Signal Processing Society Best Paper Award. This major award recognized the contributions of her 2010 IEEE Transactions on Signal Processing paper co-authored with her student Wei Lu on {\em Modified-CS: Modifying compressive sensing for problems with partially known support}. She is a Fellow of the IEEE  (class of 2019).
\end{IEEEbiography}






\end{document}